\documentclass[twoside,11pt]{article}

\usepackage{url}
\usepackage{subfigure} 
\usepackage{wrapfig}
\usepackage[inline]{enumitem}
\usepackage{amsmath}
\usepackage{amsfonts}
\usepackage{mathtools}
\usepackage{caption}
\usepackage{jmlr2e_arxiv}

\usepackage{amsthm}

\usepackage{algorithm,algorithmic}
\usepackage{todonotes}
\usepackage{makecell}
\newcolumntype{x}[1]{>{\centering\arraybackslash}p{#1}}
\usepackage{tikz}

\usepackage{multirow}
\usepackage{slashbox}
\usetikzlibrary{decorations.pathreplacing}
\usepackage[inline]{enumitem}

\usepackage{nameref}
\newcounter{mylabelcounter}

\makeatletter
\newcommand{\labelText}[2]{%
\refstepcounter{mylabelcounter}%
\immediate\write\@auxout{%
  \string\newlabel{#2}{{1}{\thepage}{{\unexpanded{#1}}}{mylabelcounter.\number\value{mylabelcounter}}{}}%
}%
}
\makeatother

\usepackage{setspace}

\newtheorem{thm}{Theorem}
\newtheorem{defn}{Definition}
\newtheorem{lem}{Lemma}
\newtheorem{asmpt}{Assumption}
\newtheorem{prop}{Proposition}

\ShortHeadings{Successor Features Combine  Model-based and Model-Free Reinforcement Learning}{Lehnert and Littman}
\firstpageno{1}

\definecolor{c0}{RGB}{31,119,180}
\definecolor{c1}{RGB}{255,127,14}
\definecolor{c2}{RGB}{44,160,44}
\definecolor{c3}{RGB}{214,39,40}

\begin{document}

\title{Successor Features Combine Elements of Model-Free and Model-based Reinforcement Learning}

\author{\name Lucas Lehnert \email lucas\_lehnert@brown.edu \\
       \addr Computer Science Department\\
       Carney Institute for Brain Science\\
       Brown University\\
       Providence, RI 02912, USA
       \AND
       \name Michael L. Littman \email michael\_littman@brown.edu \\
       \addr Computer Science Department\\
       Brown University\\
       Providence, RI 02912, USA
} 

\maketitle

\begin{abstract}
A key question in reinforcement learning is how an intelligent agent can generalize knowledge across different inputs.
By generalizing across different inputs, information learned for one input can be immediately reused for improving predictions for another input.
Reusing information allows an agent to compute an optimal decision-making strategy using less data.
State representation is a key element of the generalization process, compressing a high-dimensional input space into a low-dimensional latent state space.
This article analyzes properties of different latent state spaces, leading to new connections between model-based and model-free reinforcement learning.
Successor features, which predict frequencies of future observations, form a link between model-based and model-free learning: 
Learning to predict future expected reward outcomes, a key characteristic of model-based agents, is equivalent to learning successor features.
Learning successor features is a form of temporal difference learning and is equivalent to learning to predict a single policy's utility, which is a characteristic of model-free agents.
Drawing on the connection between model-based reinforcement learning and successor features, we demonstrate that representations that are predictive of future reward outcomes generalize across variations in both transitions and rewards.
This result extends previous work on successor features, which is constrained to fixed transitions and assumes re-learning of the transferred state representation.
\end{abstract}


\section{Introduction}

A central question in Reinforcement Learning (RL)~\citep{sutton2018rlbook} is how to process high dimensional inputs and compute decision-making strategies that maximize rewards.
For example, a self-driving car has to process all its sensor data to decide when to accelerate or brake to drive the car safely. 
If the sensor data is high dimensional, obtaining an optimal decision-making strategy may become computationally very difficult because an optimal decision has to be computed for every possible sensor input.
This ``curse of dimensionality''~\citep{bellman1961adaptivecontrol} can be overcome by compressing high dimensional sensor data into a lower dimensional latent state space.
In the self-driving car example, if the car is following another vehicle that is stopping, detecting brake lights is sufficient to make the decision to slow the self-driving car down. 
Other information such as the colour of the car in front can be ignored.
In RL, these decisions are grounded through a reward function, which would give high reward if the self-driving car reaches its destination and low reward if an accident is caused.
An intelligent agent can simplify a task by mapping high-dimensional inputs into a lower dimensional latent state space, leading to faster learning because information can be reused across different inputs. 
This article addresses the question how different principles of compressing inputs are related to one another and demonstrates how an agent can learn to simplify one task to accelerate learning on a different previously unseen task.

Previous work presents algorithms that reuse Successor Features (SFs)~\citep{baretto2017sf} to initialize learning across tasks with different reward specifications, leading to improvements in learning speed in challenging control tasks~\citep{baretto2018deepsf,zhang2017deepsucc,kulkarni2016deep}.
This article follows a different methodology: By analyzing which properties different latent state spaces are predictive of, different models of generalizations are compared leading to new connections between latent state spaces that support either \emph{model-free RL} or \emph{model-based RL}.
Model-free RL memorizes and makes predictions for one particular decision-making strategy, while model-based RL aims at predicting future reward outcomes for any arbitrary decision-making strategy.
Our analysis demonstrates how previous formulations of SFs learn latent state spaces that are predictive of the optimal decision-making strategy and are thus akin to model-free RL.
This article introduces a new model, called the \emph{Linear Successor Feature Model (LSFM)}, and presents results demonstrating that LSFMs learn latent state spaces that support model-based RL.

Latent state spaces model equivalences between different inputs and are suitable for transfer across different tasks only if these equivalences are preserved. 
Because LSFMs construct latent state spaces that extract equivalences of a task's transition dynamics and rewards, they afford transfer to tasks that preserve these equivalences but have otherwise different transitions, rewards, or optimal decision-making strategies.
Ultimately, by compressing a high-dimensional state space, the representation learned by an LSFM can enable a model-based agent to learn a model of an MDP and an optimal policy significantly faster than agents that do not use a state abstraction or use a state abstraction of a different type. 
In contrast to SFs, which afford transfer across tasks with different rewards but assume fixed transition dynamics~\citep{baretto2017sf,stachenfeld2017sr}, LSFMs remove the assumption of a fixed transition function.
While SFs are re-learned and adjusted to find the optimal policy of a previously unseen task~\citep{lehnert2017sf}, the presented experiments outline how LSFMs can preserve the latent state space across different tasks and construct an optimal policy by using less data than tabular RL algorithms that do not generalize across inputs.

To unpack the different connection points and contributions, we start by first presenting the two state representation types in Section~\ref{sec:state-rep} and provide a formal introduction of successor features in Section~\ref{sec:sf}.
Subsequently, Section~\ref{sec:rew-pred} introduces LSFMs and presents the main theoretical contributions showing how LSFMs can be used for model-based RL.
Then, Section~\ref{sec:val-pred} presents the link to model-free learning and presents a sequence of examples and simulation illustrating to what extent the representations learned with LSFMs generalize across tasks with different transitions, rewards, and optimal policies.

\section{Predictive State Representations} \label{sec:state-rep}

An RL task is formalized as a Markov Decision Processes (MDP) $M = \langle \mathcal{S}, \mathcal{A}, p, r, \gamma \rangle$, where the state space $\mathcal{S}$ describes all possible sensor inputs, and the finite action space $\mathcal{A}$ describes the space of all possible decisions.
The state space $\mathcal{S}$ is assumed to be an arbitrary set and can either be finite or uncountably infinite.
All results presented in this article are stated for arbitrary state spaces $\mathcal{S}$ unless specified otherwise.
How the state changes over time is determined by the transition function $p$, which specifies a probability or density function of reaching a next state $s'$ if an action $a$ is selected at state $s$.
Transitions are Markovian and the probability or density of transitioning to a state $s'$ is specified by $p(s'| s,a)$.
The reward function $r : \mathcal{S} \times \mathcal{A} \times \mathcal{S} \to [-R, R]$ specifies which reward is given for each transition and rewards are bounded in magnitude by a constant $R \in \mathbb{R}$.
Besides assuming a bounded reward function, both reward and transition functions are assumed to be arbitrary.

A \emph{policy} $\pi$ describes an agent's decision-making strategy and specifies a probability $\pi(s,a)$ of selecting an action $a \in \mathcal{A}$ at state $s \in \mathcal{S}$.
A policy's performance is described by the value function
\begin{equation}
V^\pi(s) = \mathbb{E}_{p, \pi} \left[ \sum_{t=1}^\infty \gamma^t r(s_t,a_t,s_{t+1}) \middle| s_1 = s \right], \label{eq:v-def}
\end{equation}
which predicts the expected discounted return generated by selecting actions according to $\pi$ when trajectories are started at the state $s$.
The expectation\footnote{The subscript of the expectation operator $\mathbb{E}$ denotes the probability distributions or densities over which the expectation is computed.} in Equation~\eqref{eq:v-def} is computed over all infinite length trajectories that select actions according to $\pi$ and start at state $s$.
Similarly, the Q-value function is defined as
\begin{equation}
Q^\pi(s,a) = \mathbb{E}_{p, \pi} \left[ \sum_{t=1}^\infty \gamma^t r(s_t,a_t,s_{t+1}) \middle| s_1 = s, a_1 = a \right]. \label{eq:q-def}
\end{equation}
The expectation of the Q-value function in Equation~\eqref{eq:q-def} is computed over all infinite length trajectories that start at state $s$ with action $a$ and then follow the policy $\pi$.

A latent state space $\mathcal{S}_\phi$ is constructed using a \emph{state representation} function $\phi : \mathcal{S} \to \mathcal{S}_\phi$.
A state representation can be understood as a compression of the state space, because two different states $s$ and $\tilde{s}$ can be assigned to the same latent state $\phi(s) = \phi(\tilde{s})$.
In this case, the state representation $\phi$ aliases $s$ and $\tilde{s}$.

Figure~\ref{fig:example-column-world-grid} presents a grid world example where nine states are compressed into three different latent states and $\mathcal{S}_\phi = \{ \pmb{\phi}_1, \pmb{\phi}_2, \pmb{\phi}_3 \}$.
In this example, the state representation partitions the state space along three different columns and constructs a smaller latent grid world of size $3 \times 1$.
If an intelligent agent uses this state representation for learning, the agent would only maintain information about which column it is in but not which row and effectively operate on this smaller latent $3 \times 1$ grid world.

\begin{figure}
\centering

\subfigure[Column World Example]{
\label{fig:example-column-world-grid}
\includegraphics[scale=1.0]{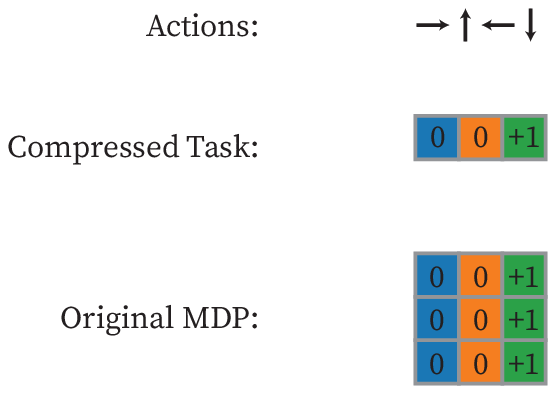}
}
\subfigure[State Values $V^\pi$]{
\label{fig:example-column-world-values}
\hspace{1cm}\includegraphics[scale=1.0]{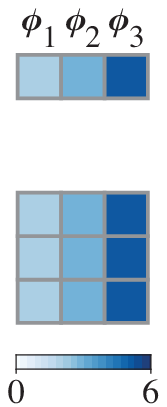}\hspace{1cm}
}
\subfigure[Example Trajectory]{
\label{fig:example-column-world-traj}
\hspace{1cm}\includegraphics[scale=1.0]{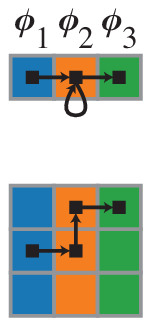}\hspace{1cm}
}

\caption{State Representations Construct Lower Dimensional Latent State Spaces. 
~\ref{fig:example-column-world-grid}: The column world example is a $3 \times 3$ grid world where an agent can move up (action $\uparrow$), down (action $\downarrow$), left (action $\leftarrow$), or right (action $\rightarrow$) to adjacent grid cells and entering the right (green) column is rewarded.
The number in each grid cell indicate the reward obtained by entering the respective cell.
The state representation merges each column (colour) into a different latent state.
~\ref{fig:example-column-world-values}: The lower panel presents a matrix plot of the state values $V^\pi$ for a policy that selects actions uniformly at random. 
Grid cells of the same column have equal distance to the rewarding column and thus equal state values.
Because this state representation only generalizes across states of the same column, the constructed latent state space can be used to predict the value function $V^\pi$ as well.
The top panel, which presents a matrix plot of state values for the three latent states, illustrates how the latent state space can be used for value predictions as well.
~\ref{fig:example-column-world-traj}: Using the latent state space, a latent $3 \times 1$ MDP can be constructed. 
Both latent MDP and original MDP produce equal reward sequences for trajectories that follow the same action sequence and that start at corresponding states and latent states. 
The figure illustrates such an example trajectory that starts at (1,1) or latent state $\pmb{\phi}_i$ and then follows the action sequence $\rightarrow, \downarrow, \rightarrow$.
Selecting the action $\downarrow$ is modeled as a self loop in the latent MDP.
Both trajectories generate the same reward sequence of $0, 0, 1$.
Intuitively, the compressed $3 \times 1$ grid world simulates reward sequences in exactly the same way as the $3 \times 3$ column world task.
}
\label{fig:example-column-world}
\end{figure}

\subsection{Value-Predictive State Representations}

A value-predictive state representation constructs a latent state space that retains enough information to support accurate value or Q-value predictions.
Figure~\ref{fig:example-column-world-values} shows that the value of states of the same column are equal.
Suppose each latent state is represented as a one-hot bit vector\footnote{A one-hot bit vector $\pmb{e}_i$ is a column vector of zeros but with the $i$th entry set to one. Vectors are denoted with bold lower-case letters. Matrices are denoted with bold capitalized letters.} in three dimensions and $\mathcal{S}_\phi = \{ \pmb{e}_1, \pmb{e}_2 , \pmb{e}_3 \}$.
Because each grid cell is mapped to a one-hot bit vector, $\phi(s) = \pmb{e}_i$ for some index $i$.
Furthermore there exists a real valued vector $\pmb{v} \in \mathbb{R}^3$ such that 
\begin{equation}
V^\pi(s) =( \phi(s) )^\top \pmb{v} = \pmb{e}_i^\top \pmb{v}.
\end{equation} 
Because the state representation $\phi$ outputs one-hot bit vectors, each entry of the vector $\pmb{v}$ contains the state value associated with one of the three columns.
Similarly, the state representation $\phi$ can be used for Q-value predictions and for each action $a$, $Q^\pi(s,a) = \pmb{e}_i^\top \pmb{q}_a$, where $\pmb{q}_a \in \mathbb{R}^3$.
This article refers to such state representations that serve as basis functions~\citep{sutton1996generalization,konidaris2008fourier} for accurate value or Q-value predictions as \emph{value-predictive}.
Because value-predictive state representations are only required to be predictive of a particular policy $\pi$, they implicitly depend on the policy $\pi$.

\subsection{Reward-Predictive State Representations}

A reward-predictive state representation constructs a latent state space that retains enough information about the original state space to support accurate predictions of future expected reward outcomes.
Figure~\ref{fig:example-column-world-traj} shows a trajectory that starts at grid cell (1,1) and follows an action sequence to produce a reward sequence of $0,0,1$.
Because the state representation constructs a latent MDP that is a $3 \times 1$ grid world, a trajectory starting at the latent grid cell $\pmb{\phi}_1$ and following the same action sequence results in a latent state sequence of $\pmb{\phi}_1,\pmb{\phi}_2, \pmb{\phi}_2, \pmb{\phi}_3$. 
If a positive reward is associated with entering the state $\pmb{\phi}_3$, the latent grid world would also produce a reward sequence of $0,0,1$.
For any start state and any arbitrary action sequence, both the latent grid world and the original grid world produce equal reward sequences.
This property can be formalized using a family of functions $\{ f_t \}_{t \in \mathbb{N}}$ that predicts the expected reward outcome after executing an action sequence $a_1,...,a_t$ starting at state $s$:
\begin{equation}
f_t: (s, a_1,...,a_t) \mapsto \mathbb{E}_{p} \left[r_t \middle| s, a_1, ..., a_t \right]. \label{eq:def-exp-rew}
\end{equation}
The expectation in Equation~\eqref{eq:def-exp-rew} is computed over all trajectories that start in state $s$ and follow the action sequence $a_1,...,a_t$.
A state representation is \emph{reward-predictive} if the function $f_t$ can be re-parametrized in terms of the constructed latent state space and if there exists a family of functions $\{ g_t \}_{t \in \mathbb{N}}$ such that
\begin{equation}
\forall t \ge 1, s, a_1,...,a_t,~ f_t(s,a_1,...,a_t) = g_t(\phi(s),a_1,...,a_t). \label{eq:exp-rew-fn}
\end{equation}
Because reward-predictive state representations are designed to produce accurate predictions of future expected reward outcomes, they need to encode information about both the transition and reward functions. 
Unlike value-predictive state representations, reward-predictive state representations are independent of a particular policy.

\subsection{Learning State Representations}

Figure~\ref{fig:example-column-world} presents an example where the same state representation is both value- and reward-predictive. 
In this article, we will present the distinctions and connection points between value-predictive and reward-predictive state representations.
Further, we will discuss learning algorithms that search the space of all possible state representations to identify approximations of value- or reward-predictive state representations.
The following results and examples demonstrate that value- and reward-predictive state representations can generalize
across states very differently.
Specifically, re-using a reward-predictive state representation $\phi$ across tasks can accelerate learning because a model-based agent only needs to construct the function $g_t$ (which is defined on a small latent state space) for the new task instead of re-learning the function $f_t$ from scratch.
To ensure a fair comparison between different approximations, this article assumes that the dimensionality of the latent state space or the number of latent states is a fixed hyper-parameter.

\section{Successor Features}\label{sec:sf}

Successor features~\citep{baretto2017sf} are a generalization of the Successor Representation (SR)~\citep{dayan1993successor}.
The SR can be defined as follows:
For finite state and action spaces, the transition probabilities while selecting actions according to a policy $\pi$ can be written as a stochastic transition matrix $\pmb{P}^\pi$.
If the start state with index $s$ is represented as a one-hot bit vector $\pmb{e}_s$, the probability distribution of reaching a state after one time step of executing policy $\pi$ can be written as a row vector $\pmb{e}_s^\top \pmb{P}^\pi$.
After $t$ time steps, the probability distribution over states can be written as a vector $\pmb{e}_s^\top ( \pmb{P}^\pi )^t$. 
Suppose the path across the state space has a random length that follows the Geometric distribution with parameter $\gamma$:
At each time step, a biased coin is flipped and the path continues with probability $\gamma$.
In this model, the probability vector of reaching different states in $t$ time steps is $(1 - \gamma) \gamma^{t-1} \pmb{e}_s^\top ( \pmb{P}^\pi )^t$.
Omitting the factor $(1-\gamma)$, the SR recursively computes the marginal probabilities over all time steps:
\begin{equation}
\pmb{\Psi}^\pi = \sum_{t=1}^\infty \gamma^{t-1} ( \pmb{P}^\pi )^{t-1} = \pmb{I} + \gamma \pmb{P}^\pi \pmb{\Psi}^\pi. \label{eq:SR}
\end{equation}
Each entry $(i,j)$ of the matrix $(1 - \gamma) \pmb{\Psi}^\pi$ contains the marginal probability across all possible durations of transitioning from state $i$ to state $j$.
Intuitively, the entry $(i,j)$ of the matrix $\pmb{\Psi}^\pi$ can be understood as the frequency of encountering state $j$ when starting a path at state $i$ and following the policy $\pi$.
An \emph{action conditioned fSR} describes the marginal probability across all possible durations of transitioning from state $i$ to state $j$, but first a particular action $a$ is selected, and then a policy $\pi$ is followed:
\begin{equation}
\pmb{\Psi}^\pi_a \overset{\text{def.}}{=} \sum_{t=1}^\infty \gamma^{t-1} \pmb{P}_a ( \pmb{P}^\pi )^{t-2} = \pmb{I} + \gamma \pmb{P}_a \pmb{\Psi}^\pi, \label{eq:SR-a-cond}
\end{equation}
where $\pmb{P}_a$ is the stochastic transition matrix describing all transition probabilities when action $a$ is selected.
Because $\pmb{P}_a$ is a stochastic matrix, it can be shown that $\pmb{\Psi}^\pi$ is invertible and that there exists a one-to-one correspondence between each transition matrix and action-conditional SR matrix.\footnote{By Equation~\eqref{eq:SR}, $\pmb{\Psi}^\pi = \pmb{I} + \gamma \pmb{P}^\pi \pmb{\Psi}^\pi \iff \pmb{I} = (\pmb{I} - \gamma \pmb{P}^\pi) \pmb{\Psi}^\pi \iff ( \pmb{\Psi}^\pi )^{-1} = \pmb{I} - \gamma \pmb{P}^\pi$. Equation~\eqref{eq:SR-a-cond} outlines how to construct $\pmb{\Psi}^\pi_a$ from $\pmb{P}_a$ for all actions. The reverse direction follows from $\pmb{\Psi}^\pi_a = \pmb{I} + \gamma \pmb{P}_a \pmb{\Psi}^\pi \iff (\pmb{\Psi}^\pi_a - \pmb{I}) (\pmb{\Psi}^\pi)^{-1}/ \gamma = \pmb{P}_a$.}

SFs combine this idea with arbitrary state representations. Given a state representation $\phi$, the SF is a column vector defined for each state and action pair~\citep{kulkarni2016deep,zhang2017deepsucc} and 
\begin{equation}
\pmb{\psi}^\pi(s,a) = \mathbb{E}_{p,\pi} \left[ \sum_{t=1}^\infty \gamma^{t-1} \pmb{\phi}_{s_t} \middle| s_1=s, a_1=a \right], \label{eq:sf-def}
\end{equation}
where the expectation in Equation~\eqref{eq:sf-def} is computed over all infinite length trajectories that start at state $s$ with action $a$ and then follow the policy $\pi$.
The output at state $s$ of the vector-valued state-representation function $\phi$ is denoted by $\pmb{\phi}_s$.
We will refer to this vector $\pmb{\phi}_s$ as either the latent vector or latent state.
If following the policy $\pi$ leads to encountering a particular latent vector $\pmb{\phi}_{s'}$ many times, then this latent vector will occur in the summation in Equation~\eqref{eq:sf-def} many times.\footnote{The remainder of the article will list function arguments in the subscript of a symbol and $\pmb{\phi}_s$ always denotes the output of $\phi$ at state $s$.}
Depending on the state representation $\phi$, the vector $\pmb{\psi}^\pi(s,a)$ will be more similar to the latent vector $\pmb{\phi}_{s'}$ and $\pmb{\psi}^\pi(s,a)$ will be more dis-similar to a latent vector $\pmb{\phi}_{\tilde{s}}$ if $\tilde{s}$ is a state that cannot be reached from the state $s$.
A SF vector $\pmb{\psi}^\pi$ can be understood as a statistic measuring how frequently different latent states vectors are encountered.

The following two sections draw connections between SFs, reward-predictive, and value-predictive state representations and outline under what assumptions learning SFs is equivalent to learning reward- or value-predictive state representations.

\section{Connections to Reward-Predictive Representations}\label{sec:rew-pred}

A reward-predictive state representation constructs a latent state space that is rich enough to produce predictions of future expected reward outcomes for any arbitrary actions sequence.
For accurate predictions, the empirical transition probabilities between latent states have to mimic transitions in the original task.
Figure~\ref{fig:three-state-example} presents a reward-prediction example where only one action is available to the agent.
In this task, the goal is to predict that a positive reward is obtained in three time steps if the agent starts at state $s_1$.
This example compares two different state representations, the representation $\phi$, which does not compress the state space, and $\tilde{\phi}$, which merges the first two states into one latent state.
These two state representations lead to different empirical latent transition probabilities.
While the first representation preserves the deterministic transitions of the task, the second representation does not.
If states $s_1$ and $s_2$ are mapped to the same latent state $\tilde{\pmb{\phi}}_1$, then a transition from state $s_1$ to $s_2$ appears as a self-loop from latent state $\tilde{\pmb{\phi}}_1$ to itself and a transition from $s_2$ to $s_3$ appears as a transition from $\tilde{\pmb{\phi}}_1$ to $\tilde{\pmb{\phi}}_2$.
Because the state representation $\phi$ constructs a latent state space with empirical latent transition probabilities that match the transition probabilities of the original task, this state representation is reward predictive.

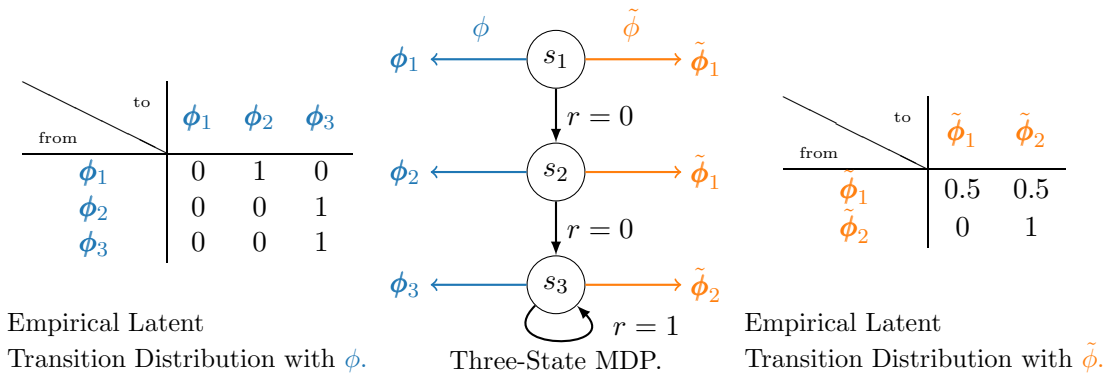
\begin{figure}
\centering

\begin{tikzpicture}
\definecolor{c0}{RGB}{31,119,180}
\definecolor{c1}{RGB}{255,127,14}

\node[] () at (-4.9,1.5){
\begin{tabular}{c | c c c}
\backslashbox{\tiny from}{\tiny to} & \color{c0}{$\pmb{\phi}_1$} & \color{c0}{$\pmb{\phi}_2$} & \color{c0}{$\pmb{\phi}_3$} \\
\hline
\color{c0}{$\pmb{\phi}_1$} & 0 & 1 & 0 \\
\color{c0}{$\pmb{\phi}_2$} & 0 & 0 & 1 \\
\color{c0}{$\pmb{\phi}_3$} & 0 & 0 & 1 \\
\end{tabular}};

\node[circle,draw=black,minimum size=.3cm](s1) at (0, 3.0) {$s_1$};
\node[circle,draw=black,minimum size=.3cm](s2) at (0, 1.5) {$s_2$};
\node[circle,draw=black,minimum size=.3cm](s3) at (0, 0.0) {$s_3$};

\draw[thick,-latex] (s1) -- (s2) node[pos=.5, right] {$r=0$};
\draw[thick,-latex] (s2) -- (s3) node[pos=.5, right] {$r=0$};
\draw[thick,-latex] (s3) edge[thick,out=-135,in=-45,looseness=4] (s3);
\node[anchor=south] (rewlabel) at (1.2, -.8) {$r=1$};

\node[anchor=south,c0] (phi) at (-1, 3.1) {$\phi$};
\node[c0] (phi1) at (-2, 3.0) {$\pmb{\phi}_1$}; \draw[->, thick, c0, fill=c0] (s1) to (phi1);
\node[c0] (phi2) at (-2, 1.5) {$\pmb{\phi}_2$}; \draw[->, thick, c0, fill=c0] (s2) to (phi2);
\node[c0] (phi3) at (-2, 0.0) {$\pmb{\phi}_3$}; \draw[->, thick, c0, fill=c0] (s3) to (phi3); 

\node[anchor=south,c1] (phitilde) at (1, 3.1) {$\tilde{\phi}$};
\node[c1] (phitilde1) at (2, 3.0) {$\tilde{\pmb{\phi}}_1$}; \draw[->, thick, c1, fill=c1] (s1) to (phitilde1);
\node[c1] (phitilde2) at (2, 1.5) {$\tilde{\pmb{\phi}}_1$}; \draw[->, thick, c1, fill=c1] (s2) to (phitilde2);
\node[c1] (phitilde3) at (2, 0.0) {$\tilde{\pmb{\phi}}_2$}; \draw[->, thick, c1, fill=c1] (s3) to (phitilde3);

\node[] () at (4.9,1.5){
\begin{tabular}{c | c c}
\backslashbox{\tiny from}{\tiny to} & \color{c1}{$\tilde{\pmb{\phi}}_1$} & \color{c1}{$\tilde{\pmb{\phi}}_2$} \\
\hline
\color{c1}{$\tilde{\pmb{\phi}}_1$} & 0.5 & 0.5 \\
\color{c1}{$\tilde{\pmb{\phi}}_2$} & 0 & 1 \\
\end{tabular}};

\draw[] (-4.9,-.1) node[label={[align=left]south:\small Empirical Latent\\ \small Transition Distribution with \color{c0}{$\phi$}.}] {};
\draw[] ( 0,-.6) node[label={[align=left]south:\small Three-State MDP.}] {};
\draw[] ( 4.9,-.1) node[label={[align=left]south:\small Empirical Latent\\ \small Transition Distribution with \color{c1}{$\tilde{\phi}$}.}] {};

\end{tikzpicture}

\caption{Three-State MDP Example.
The centre schematic shows a single action three-state MDP with deterministic transitions (black arrows).
Only the self-looping transition at state $s_3$ is rewarded.
The two state representations $\phi$ and $\tilde{\phi}$ map the three states to different feature vectors, resulting in different empirical feature-transition probabilities.
These probabilities are computed from observed trajectories that start at state $s_1$.}
\label{fig:three-state-example}
\end{figure}

\begin{figure}
\centering
\label{fig:weighting-fn}
\begin{tikzpicture}

\definecolor{c0}{RGB}{31,119,180}
\definecolor{c1}{RGB}{255,127,14}

\draw [thick,c1,fill=c1] (0.25,0) to[out=0,in=180,looseness=0.6] (1,1.2) to[out=0,in=180,looseness=0.6] (1.75,0) to[out=180,in=0,looseness=0] cycle;
\draw [thick,c0,fill=c0] (2.6,0)      to[out=0,in=180,looseness=0.6] (4.5,1) to[out=0,in=180,looseness=0.6] (5.8,0)   to[out=180,in=0,looseness=0] cycle;

\draw [-latex,thick,black] (1.2,0) to[out=-40,in=180+40,looseness=0.6] node[anchor=north,below=0.1cm] {$\text{Pr} \{ s \overset{a}{\rightarrow} \pmb{\phi}_3 \}$} (4.0,0);
\draw [-latex,thick,black] (1.0,2) to[out=40,in=180-40,looseness=0.6] node[anchor=south,above=0.1cm] {$\text{Pr} \{ \pmb{\phi}_1 \overset{a}{\rightarrow} \pmb{\phi}_3 | \omega \}$} (4.0,2);

\draw[fill=white,opacity=0,fill opacity=0.5] (3.25,0) -- (3.25,1.1) -- (4.75,1.1) -- (4.75,0) -- cycle;

\draw[-latex,thick] (-.1,0) -- (6.5,0) node[anchor=west] {$\mathcal{S}$};
\draw[-latex,thick] (0,-.1) -- (0,1.5)  node[label={[align=right]west: \small Probability\\ \small Density}] {};

\draw[-latex,thick] (-.1,2) -- (6.5,2)  node[label={[align=left]east: \small Feature\\ \small Space}] {};

\draw[fill=black] (1.2,0) circle (0.08cm) node[anchor=north,below=0.15cm] {$s$};

\draw[fill=black] (1.0,2) circle (0.08cm) node[anchor=north,below=0.15cm] {$\pmb{\phi}_1$};
\draw[fill=black] (2.5,2) circle (0.08cm) node[anchor=north,below=0.15cm] {$\pmb{\phi}_2$};
\draw[fill=black] (4.0,2) circle (0.08cm) node[anchor=north,below=0.15cm] {$\pmb{\phi}_3$};
\draw[fill=black] (5.5,2) circle (0.08cm) node[anchor=north,below=0.15cm] {$\pmb{\phi}_4$};

\draw[gray] (0.25,0) -- (0.25, 2);
\draw[gray] (1.75,0) -- (1.75, 2);
\draw[gray] (3.25,0) -- (3.25, 2);
\draw[gray] (4.75,0) -- (4.75, 2);
\draw[gray] (6.25,0) -- (6.25, 2);

\draw[] (5.6,0.9) node {$p(s,a,\cdot)$};
\draw[] (1.6,1.1) node {$\omega(\cdot)$};

\node[anchor=west,text width=4.9cm] at (8.2, 2.4) {\small $\text{Pr} \{ s \overset{a}{\rightarrow} \pmb{\phi}_3 \}$: The probability of transitioning from $s$ to any state mapped to $\pmb{\phi}_3$ by selecting action $a$.};
\node[anchor=west,text width=4.9cm] at (8.2, 0) {\small $\text{Pr} \{ \pmb{\phi}_1 \overset{a}{\rightarrow} \pmb{\phi}_3 | \omega \}$: Empirical probability of transitioning from latent state $\pmb{\phi}_1$ to latent state $\pmb{\phi}_3$ by selecting action $a$.};

\end{tikzpicture}
\caption{Empirical Latent Transition Probabilities Depend on State-Visitation Frequencies.
This example illustrates how empirical latent transition probabilities depend on state-visitation frequencies. 
In this example, the state space $\mathcal{S}$ is a bounded interval in $\mathbb{R}$ that is clustered into one of four latent states: $\pmb{\phi}_1, \pmb{\phi}_2, \pmb{\phi}_3$, or $\pmb{\phi}_4$. 
State-visitation frequencies are modeled for each partition independently using the density function $\omega$.
The schematic plots the density function $p$ over states of selecting action $a$ at state $s$ (blue area) and the density function $\omega$ over the state partition $\pmb{\phi}_1$ (orange area). 
The probability $\text{Pr} \{ s \protect\overset{a}{\rightarrow} \pmb{\phi}_3 \}$ of transitioning into the partition $\pmb{\phi}_3$ is the blue shaded area.
The probability $\text{Pr} \{ \pmb{\phi}_1 \protect\overset{a}{\rightarrow} \pmb{\phi}_3 | \omega \}$ of a transition from $\pmb{\phi}_1$ to $\pmb{\phi}_3$ occurring is the marginal of $\text{Pr} \{ s \protect\overset{a}{\rightarrow} \pmb{\phi}_3 \}$ over all states $s$ mapping to $\pmb{\phi}_1$, weighted by $\omega$. }
\label{fig:weighting-fn}
\end{figure}
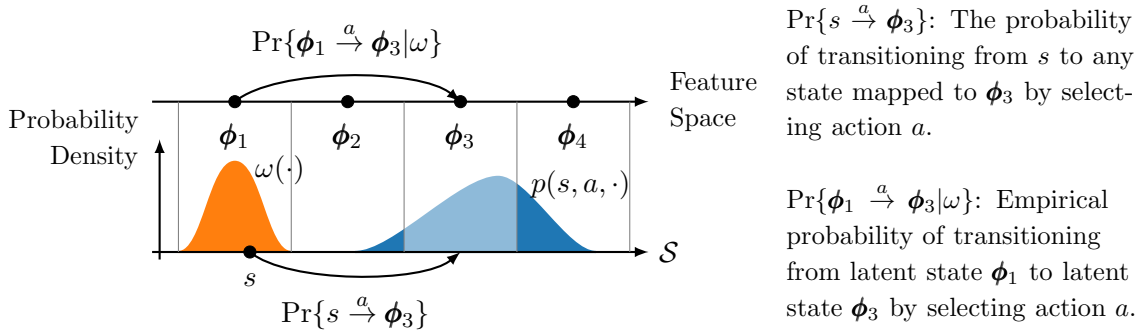 

A reward-predictive state representation can be used in conjunction with a Linear Action Model~\citep{sutton2008lineardyna,yao2012lam} to compute expected future reward outcomes.

\begin{defn}[Linear Action Model (LAM)]
Given an MDP and a state representation $\phi : \mathcal{S} \to \mathbb{R}^n$, a LAM consists of a set of matrices and vectors $\{ \pmb{M}_a, \pmb{w}_a \}_{a \in \mathcal{A}}$, where $\pmb{M}_a$ is of size $n \times n$ and the column vector $\pmb{w}_a$ is of dimension $n$.
\end{defn}

Given a fixed state representation, the transition matrices of a LAM $\{ \pmb{M}_a \}_{a \in \mathcal{A}}$ model the empirical latent transition probabilities and the vectors $\{ \pmb{w}_a \}_{a \in \mathcal{A}}$ model a linear map from latent states to expected one-step reward outcomes.
The expected reward outcome after following the action sequence $a_1,...,a_t$ starting at state $s$ can then be approximated with 
\begin{equation}
\mathbb{E}_p \left[ r_t \middle| s, a_1,...,a_t \right] \approx \pmb{\phi}_s^\top \pmb{M}_{a_1} \cdots \pmb{M}_{a_{t-1}} \pmb{w}_{a_t}. \label{eq:rollout-pred}
\end{equation}
The following sections will address how a state representation $\phi$ and a LAM can be found to predict expected future reward outcomes as accurately as possible.
Because this article's goal is to establish different connections between learning successor features and model-based RL and to demonstrate that the learned reward-predictive state representations are suitable for transfer across variations in transitions and rewards, an extension of these model to non-linear latent transition and reward functions is left to future work.

To tie SFs to reward-predictive state representations, we first introduce a set of square real-valued matrices $\{ \pmb{F}_a \}_{a \in \mathcal{A}}$ such that, for every state $s$ and action $a$, 
\begin{align}
\pmb{\phi}_s^\top \pmb{F}_a &\approx \pmb{\psi}^\pi(s,a) \label{eq:lsfm-2}\\
&= \mathbb{E}_{p,\pi} \left[ \sum_{t=1}^\infty \gamma^{t-1} \pmb{\phi}_{s_t} \middle| s_1=s, a_1=a \right] \label{eq:lsfm-1} 
\end{align} 
where the policy $\pi$ is defined on the latent state space.
A Linear Successor Feature Model (LSFM) is then defined using the matrices $\{ \pmb{F}_a \}_{a \in \mathcal{A}}$:

\begin{defn}[Linear Successor Feature Model (LSFM)]
Given an MDP, a policy $\pi$, and a state representation $\phi : \mathcal{S} \to \mathbb{R}^n$, an LSFM consists of a set of matrices and vectors $\{ \pmb{F}_a, \pmb{w}_a \}_{a \in \mathcal{A}}$, where $\pmb{F}_a$ is of size $n \times n$ and the column vector $\pmb{w}_a$ is of dimension $n$.
The matrices $\{ \pmb{F}_a \}_{a \in \mathcal{A}}$ are used to model a linear map from latent state features to SFs as described in Equation~\eqref{eq:lsfm-2}.
\end{defn}

LSFMs require the state representation $\phi$ to (linearly) approximate the SF $\pmb{\psi}^\pi(s,a)$ using the matrices $\{ \pmb{F}_a \}_{a \in \mathcal{A}}$ (Equation~\eqref{eq:lsfm-2}).
Previously presented SF frameworks~\citep{baretto2017sf,baretto2018deepsf} do not use the state representation $\phi$ to approximate SF vectors $\pmb{\psi}^\pi$ as described in Equation~\eqref{eq:lsfm-2} and construct the state-representation function $\phi$ differently, for example, by setting the output of $\phi$ to be one-step rewards~\citep{barreto2017sf}.
In contrast, LSFMs are used to learn the state-representation function $\phi$ that satisfies Equation~\eqref{eq:lsfm-2}.
Because LSFMs distinctly incorporate this approximative property, SFs can be connected to model-based RL.

An intelligent agent that uses a state representation $\phi$ operates directly on the constructed latent state space and is constrained to only search the space of policies that are defined on its latent state space.
These policies are called \emph{abstract policies}.

\begin{defn}[Abstract Policies]\label{def:abs-pi}
An abstract policy $\pi_\phi$ is a function mapping latent state and action pairs to probability values:
\begin{equation*}
\forall s \in \mathcal{S}, a \in \mathcal{A}, \pi_\phi(\pmb{\phi}_s,a) \in [0,1] ~\text{and}~ \sum_a \pi_\phi(\pmb{\phi}_s,a) = 1.
\end{equation*}
For a fixed state representation $\phi$, the set of all abstract policies is denoted with $\Pi_\phi$.
\end{defn}

The following sections first tie learning LAMs to reward-predictive state representations.
We then show that learning LSFMs is equivalent to learning LAMs tying LSFMs to reward-predictive state representations.

\subsection{Encoding Bisimulation Relations}\label{sec:bisim}

To ensure accurate predictions of future reward outcomes, the previous discussion suggests that the empirical latent transition probabilities have to match the transition probabilities in the original task.
Figure~\ref{fig:weighting-fn} presents a schematic explaining these dependencies further.
In this example, the state space is a bounded interval in $\mathbb{R}$ that is mapped to four different latent states, $\pmb{\phi}_1$, $\pmb{\phi}_2$, $\pmb{\phi}_3$, or $\pmb{\phi}_4$.
The probability of transitioning from the state $s$ to any state that is mapped to $\pmb{\phi}_3$ is denoted with $\text{Pr} \{ s \overset{a}{\to} \pmb{\phi}_3 \}$.
This probability $\text{Pr} \{ s \overset{a}{\to} \pmb{\phi}_3 \}$ is the marginal over all states $s'$ that are mapped to the latent state $\pmb{\phi}_3$.
Assume that $\omega$ is a density function over all states that are mapped to the latent state $\pmb{\phi}_1$.
This density function could model the visitation frequencies of different states as an intelligent agent interacts with the MDP.
The empirical probability of transitioning from latent state $\pmb{\phi}_1$ to $\pmb{\phi}_3$ is then the marginal over all states mapping to $\pmb{\phi}_1$ and
\begin{equation}
\text{Pr} \left\{ \pmb{\phi}_1 \overset{a}{\to} \pmb{\phi}_3 \middle| \omega \right\} = \int_{s : \phi(s) = \pmb{\phi}_1} \omega(s) \text{Pr} \{ s \overset{a}{\to} \pmb{\phi}_3 \} ds = \mathbb{E}_{\omega} \left[ \text{Pr} \{ s \overset{a}{\to} \pmb{\phi}_3 \} \middle| \phi(s) = \pmb{\phi}_1 \right]. \label{eq:latent-prop}
\end{equation}
The expectation in Equation~\eqref{eq:latent-prop} is computed with respect to $\omega$ over all states $s$ that map to the latent state $\pmb{\phi}_1$.
As Equation~\eqref{eq:latent-prop} outlines, the empirical transition probability $\text{Pr} \{ \pmb{\phi}_1 \overset{a}{\to} \pmb{\phi}_3 | \omega \}$ depends on the visitation frequencies $\omega$.
The probability $\text{Pr} \{ s \overset{a}{\to} \pmb{\phi}_3 \}$ of transitioning from a state $s$ into a partition only depends on the transition function $p$ itself.

Consider two different states $s$ and $\tilde{s}$ that map to the same latent state and $\phi(s) = \phi(\tilde{s})$.
If the state representation is constructed such that
\begin{equation}
\forall a , \forall \pmb{\phi}_i, ~ \text{Pr} \{ s \overset{a}{\to} \pmb{\phi}_i \} = \text{Pr} \{ \tilde{s} \overset{a}{\to} \pmb{\phi}_i \} ~\text{and}~ \mathbb{E}_p \left[ r(s,a,s') \middle| s,a \right] = \mathbb{E}_p \left[ r(\tilde{s},a,s') \middle| \tilde{s},a \right], \label{eq:bisim-inf}
\end{equation}
then the empirical latent state transition probabilities would become independent of $\omega$ because the integrand in Equation~\eqref{eq:latent-prop} is constant and
\begin{equation}
\text{Pr} \left\{ \pmb{\phi}_1 \overset{a}{\to} \pmb{\phi}_3 \middle| \omega \right\} = \int_{s : \phi(s) = \pmb{\phi}_1} \omega(s) \underbrace{\text{Pr} \{ s \overset{a}{\to} \pmb{\phi}_3 \}}_\text{constant} ds = \text{Pr} \{ s \overset{a}{\to} \pmb{\phi}_3 \} . \label{eq:omeage-indep-example}
\end{equation}
Equation~\eqref{eq:omeage-indep-example} follows directly from the transition condition in line~\eqref{eq:bisim-inf}, because the probability $\text{Pr} \{ s \overset{a}{\rightarrow} \pmb{\phi}_3 \}$ is constant for all states $s$ that are mapped to the latent state vector $\pmb{\phi}_1$.
If the two identities in line~\eqref{eq:bisim-inf} hold, then the resulting latent state space constructs latent transition probabilities that correspond to the transition probabilities in the original task.
Equation~\eqref{eq:bisim-inf} describes an informal definition of bisimulation~\citep{givan2003bisimulation}. 
Definition~\ref{def:bisimulation-full} listed in Appendix~\ref{ap:bisim-pfs} presents a formal measure theoretic definition of bisimulation on arbitrary (measurable) state spaces.
This definition is used to prove the theorems stated in this section.
To prove that LAMs encode state representations that generalize only across bisimilar states, two assumptions are made.

\begin{asmpt}\label{asmpt:n-part}
The state space $\mathcal{S}$ of an MDP can be partitioned into at most $n$ different partitions of bisimilar states, where $n$ is a natural number.
\end{asmpt}

\begin{asmpt}\label{asmpt:one-hot-phi}
A state representation $\phi : \mathcal{S} \to \{ \pmb{e}_1,...,\pmb{e}_n \}$ is assumed to have a range that consists of all $n$ one-hot bit vectors. 
For each $i$, there exists a state $s$ such that $\phi(s) = \pmb{e}_i$.
\end{asmpt}

Assumption~\ref{asmpt:n-part} is not particularly restrictive in a learning context: If an agent has observed $n$ distinct states during training, then a state representation assigning each state to one of $n$ different one-hot bit vectors can always be constructed. 
While doing so may not be useful to generalize across different states, this argument suggests that Assumption~\ref{asmpt:n-part} is not restrictive in practice.
Assumption~\ref{asmpt:one-hot-phi} is relaxed in the following sections.

If action $a$ is selected at state $s$, the expected next feature vector is
\begin{equation}
\mathbb{E}_p \left[ \phi(s') \middle| s, a \right] = \sum_{j=1}^n \text{Pr} \{ s \overset{a}{\to} \pmb{e}_j \} \pmb{e}_j = \left[ \cdots , \text{Pr} \{ s \overset{a}{\to} \pmb{e}_j \}, \cdots  \right]^\top. \label{eq:prob-vec}
\end{equation}
The expected value in Equation~\eqref{eq:prob-vec} is computed over all possible next states $s'$ that can be reached from state $s$ by selecting action $a$.
In Equation~\eqref{eq:prob-vec}, the next state $s'$ is a random variable whose probability distribution or density function is described by the MDP's transition function $p$.
By Assumption~\ref{asmpt:one-hot-phi}, each state is mapped to some one-hot bit vector $\pmb{e}_j$.
Because there are only $n$ different one-hot bit vectors of dimension $n$, the summation in Equation~\eqref{eq:prob-vec} is finite.
Each entry of the resulting vector in Equation~\eqref{eq:prob-vec} stores the probability $\text{Pr} \{ s \overset{a}{\to} \pmb{e}_j \}$ of observing the feature vector $\pmb{e}_j$ after selecting action $a$ at state $s$.

Because the expected next feature vector $\mathbb{E}_p \left[ \phi(s') \middle| s, a \right]$ is a probability vector, the transition matrices $\{ \pmb{M}_a \}_{a \in \mathcal{A}}$ of a LAM are stochastic:
If $\phi(s) = \pmb{e}_i$ and $\pmb{e}_i^\top \pmb{M}_a = \mathbb{E}_p \left[ \pmb{e}_j \middle| s, a  \right]$, then the $i$th row of the matrix $\pmb{M}_a$ is equal to the probability vector shown in Equation~\eqref{eq:prob-vec}.
If $\pmb{e}_i^\top \pmb{w}_a = \mathbb{E}_p \left[ r(s,a,s') \middle| s,a \right]$, then the weight vectors of a LAM $\{ \pmb{w}_a \}_{a \in \mathcal{A}}$ encode a reward table.
These observations lead to the first theorem.\footnote{Appendix~\ref{ap:bisim-pfs} presents formal proofs for all presented theorems.}

\begin{thm}\label{thm:bisim-lam}
For an MDP $\langle \mathcal{S}, \mathcal{A}, p, r, \gamma \rangle$, let $\phi : \mathcal{S} \to \{ \pmb{e}_1,...,\pmb{e}_n \}$ be a state representation and $\{ \pmb{M}_a, \pmb{w}_a \}_{a \in \mathcal{A}}$ a LAM.
Assume that $\mathcal{S}$ can be partitioned into at most $n$ partitions of bisimilar states.
If the state representation $\phi$ satisfies
\begin{equation}
\forall s \in \mathcal{S}, \forall a \in \mathcal{A}, ~ \pmb{\phi}_s^\top \pmb{w}_a = \mathbb{E}_p \left[ r(s,a,s') \middle| s,a \right] ~\text{and }~ \pmb{\phi}_s^\top \pmb{M}_a = \mathbb{E}_p \left[ \pmb{\phi}_{s'} \middle| s, a \right], \label{eq:bisim-lam}
\end{equation}
then $\phi$ generalizes across bisimilar states and any two states $s$ and $\tilde{s}$ are bisimilar if $\pmb{\phi}_s = \pmb{\phi}_{\tilde{s}}$.
\end{thm}

The proof of Theorem~\ref{thm:bisim-lam} uses the fact that the expected value of one-hot bit vectors encode exact probability values.
A similar observation can be made about the SFs for a one-hot bit-vector state representation.
In this case, the $(1 - \gamma)$ rescaled SF contains the marginal of reaching a state partition across time steps:
\begin{equation}
(1 - \gamma) \mathbb{E}_{p, \pi} \left[ \sum_{t=1}^\infty \gamma^{t-1} \pmb{e}_t \middle| s, a_1 \right] = \left[ ..., \sum_{t=1}^\infty (1-\gamma) \gamma^{t-1} \mathbb{E}_{p,\pi} \left[ \text{Pr} \left\{ s \overset{a_1 \cdots a_t}{\longrightarrow} \pmb{e}_i \right\} \middle| s, a_1 \right] , ...  \right]^\top, \label{eq:sf-prob-vec}
\end{equation}
where the expectation in Equation~\eqref{eq:sf-prob-vec} is computed over infinite length trajectories starting at state $s$ with action $a$.
This observation leads to the following theorem stating that LSFMs can be used to identify a one-hot state representation that generalizes across bisimilar states.

\begin{thm}\label{thm:bisim-lsfm}
For an MDP $\langle \mathcal{S}, \mathcal{A}, p, r, \gamma \rangle$, let $\phi : \mathcal{S} \to \{ \pmb{e}_1,...,\pmb{e}_n \}$ be a state representation and $\{ \pmb{F}_a, \pmb{w}_a \}_{a \in \mathcal{A}}$ an LSFM.
If, for one policy $\pi \in \Pi_\phi$, the representation $\phi$ satisfies
\begin{equation}
\forall s \in \mathcal{S}, \forall a \in \mathcal{A}, ~ \pmb{\phi}_s^\top \pmb{w}_a = \mathbb{E}_p \left[ r(s,a,s') \middle| s,a \right] ~\text{and}~ \pmb{\phi}_s^\top \pmb{F}_a = \pmb{\phi}_s^\top + \gamma \mathbb{E}_{p, \pi} \left[ \pmb{\phi}_{s'} \pmb{F}_{a'} \middle| s, a \right], \label{eq:bisim-lsfm}
\end{equation}
then $\phi$ generalizes across bisimilar states and any two states $s$ and $\tilde{s}$ are bisimilar if $\pmb{\phi}_s = \pmb{\phi}_{\tilde{s}}$.
If Equation~\eqref{eq:bisim-lsfm} holds for one policy $\pi \in \Pi_\phi$, then Equation~\eqref{eq:bisim-lsfm} also holds every other policy $\tilde{\pi} \in \Pi_\phi$ as well.
\end{thm}

Equation~\eqref{eq:bisim-lsfm} describes a fixed-point equation similar to the Bellman fixed-point equation:
\begin{align}
\pmb{e}_s^\top \pmb{F}_a &= \mathbb{E}_{p, \pi} \left[ \sum_{t=1}^\infty \gamma^{t-1} \pmb{e}_{s_t} \middle| s, a \right] \\
&= \pmb{e}_s^\top + \gamma \mathbb{E}_{p} \left[ \mathbb{E}_{p,\pi} \left[ \sum_{t=1}^\infty \gamma^{t-1} \pmb{e}_{s_t} \middle| s',a' \right] \middle| s, a \right] \\
&= \pmb{e}_s^\top + \gamma \mathbb{E}_{p} \left[ \pmb{e}_{s'}^\top \pmb{F}_{a'} \middle| s, a \right]. \label{eq:sf-fix-point}
\end{align}
Finding a policy $\pi \in \Pi_\phi$ to test if Equation~\eqref{eq:bisim-lsfm} holds for a state representation $\phi$ is trivial, because it is sufficient to test the state representation for any single policy.
Theorems~\ref{thm:bisim-lam} and~\ref{thm:bisim-lsfm} show that LAMs and LSFMs can be used to identify one-hot reward-predictive state representations.
To arrive at an algorithm that can learn reward-predictive state representations, the following sections convert the conditions outlined in Theorems~\ref{thm:bisim-lam} and~\ref{thm:bisim-lsfm} into learning objectives.
The next section presents an analysis showing how violating these conditions by some margin results in increased reward-sequence prediction errors.
We refer to state representations that can only approximately predict expected reward sequences as an \emph{approximate reward-predictive state representation}.

\subsection{Approximate Reward-Predictive State Representations}

In this section, we analyze to what extent a state representation is reward predictive if it only approximately satisfies the conditions outlined in Theorems~\ref{thm:bisim-lam} and~\ref{thm:bisim-lsfm}.
In addition, we will also generalize beyond one-hot representations and relax Assumption~\ref{asmpt:one-hot-phi} by considering state representations that map the state space into $\mathbb{R}^n$.
The latent feature's dimension $n$ is considered a fixed hyper-parameter.

Because LAMs only model one-step transitions but are used to predict entire reward sequences, the scale and expansion properties of the constructed latent state space influences how prediction errors scale and compound~\citep{talvitie2018rewardsformisspecifiedmodel,asadi2018lipschitzmb}.
Define the following variables:\footnote{All norms are assumed to be $L_2$. The Euclidean norm is used for vectors. The norm of a matrix $\pmb{M}$ is computed with $|| \pmb{M} || = \sqrt{\sum_{i,j} \pmb{M}(i,j)^2}$, where the summation ranges over all matrix entries $\pmb{M}(i,j)$.}
\begin{equation}
W = \max_{a \in \mathcal{A}} || \pmb{w}_a ||,~M = \max_{a \in \mathcal{A}} || \pmb{M}_a ||,~N = \sup_{s \in \mathcal{S}} || \pmb{\phi}_s ||.
\end{equation}
To identify approximate reward-predictive state representations, a state representation $\phi$ is analyzed by its one-step reward-prediction error and one-step expected transition error.
These quantities are computed using a LAM $\{ \pmb{M}_a, \pmb{w}_a \}_{a \in \mathcal{A}}$ and are defined as
\begin{align}
\varepsilon_r &= \sup_{s,a} \left| r(s,a) - \pmb{\phi}_s^\top \pmb{w}_a \label{eq:eps-def-r} \right| \text{  and}\\
\varepsilon_p &= \sup_{s,a} \left| \left| \mathbb{E}_p \left[ \pmb{\phi}_{s'}^\top \middle| s,a \right] - \pmb{\phi}_s^\top \pmb{M}_a \right| \right|. \label{eq:eps-def-t} 
\end{align}
Equivalently, a state representation is also analyzed using an LSFM that predicts the SF for a policy that selects actions uniformly at random.
For an LSFM $\{ \pmb{F}_a, \pmb{w}_a \}_{a \in \mathcal{A}}$, define 
\begin{equation}
\overline{\pmb{F}} = \frac{1}{|\mathcal{A}|} \sum_{a \in \mathcal{A}} \pmb{F}_a. \label{eq:f-bar-def}
\end{equation}
For such an LSFM, the linear SF prediction error is defined as
\begin{align}
\varepsilon_\psi &= \sup_{s,a} \left| \left| \pmb{\phi}_s^\top + \gamma \mathbb{E}_p \left[ \pmb{\phi}_{s'}^\top \overline{\pmb{F}} \middle| s,a \right] - \pmb{\phi}_s^\top \pmb{F}_a \right| \right|. \label{eq:eps-def-sf} 
\end{align}
Because the matrix $\overline{\pmb{F}}$ averages across all actions, the LSFM computes the SFs for a policy that selects actions uniformly at random. 
Here, we focus on uniform random action selection to simplify the analysis.
While the matrix $\overline{\pmb{F}}$ could be constructed differently, the proofs of the theoretical results presented in this section assume that $\overline{\pmb{F}}$ can only depend on the matrices $\{ \pmb{F}_a \}_{a \in \mathcal{A}}$ and is not a function of the state $s$.

Similar to the previous discussion, LSFMs are closely related to LAMs and the one-step transition error $\varepsilon_p$ can be upper bounded by the linear SF error $\varepsilon_\psi$.
We define the quantities $\varepsilon_r$, $\varepsilon_t$, and $\varepsilon_\psi$ to upper bound the magnitude with which the identities provided in Theorems~\ref{thm:bisim-lam} and~\ref{thm:bisim-lsfm} are violated.
The following analysis generalizes the previously presented results by showing that, if a state representation approximately satisfies the requirements of Theorems~\ref{thm:bisim-lam} and~\ref{thm:bisim-lsfm}, then this state representation is approximately reward predictive.

\begin{lem}\label{lem:t-model}
For an MDP, a state representation $\phi$, an LSFM and a LAM, 
\begin{equation}
\varepsilon_p \le \varepsilon_\psi \frac{1 + \gamma M}{\gamma} +  C_{\gamma,M,N} \Delta, \label{eq:lem-t-model}
\end{equation}
where $C_{\gamma,M,N} = \left. (1 + \gamma) (1 + \gamma M) N \middle/ (\gamma (1 - \gamma M)) \right.$ and $\Delta = \max_a || \pmb{I} + \gamma \pmb{M}_a \overline{\pmb{F}}  - \pmb{F}_a ||$.
\end{lem}

Lemma~\ref{lem:t-model} presents a bound stating that if an LSFM has low linear SF prediction errors, then a corresponding LAM can be constructed with low one-step transition error $\varepsilon_p$, assuming that $\Delta$ is close to zero.
If $\Delta = 0$, then the matrices $\{ \pmb{F}_a \}_{a \in \mathcal{A}}$ can be thought of as action-conditional SR matrices for the transition matrices $\{ \pmb{M}_a \}_{a \in \mathcal{A}}$. 
In Section~\ref{sec:sf}, Equation~\eqref{eq:SR-a-cond}, the action-conditional SR matrix is defined as $\pmb{\Psi}_a^\pi = \pmb{I} + \gamma \pmb{P}_a \pmb{\Psi}^\pi$, where $\pmb{P}_a$ is a stochastic transition matrix for a finite MDP.
Furthermore, Section~\ref{sec:sf} also shows that there exists a bijection between the transition matrices $\{ \pmb{P}_a \}_{a \in \mathcal{A}}$ and the action-conditional SR matrices $\{ \pmb{\Psi}_a^\pi \}_{a \in \mathcal{A}}$.
Similarly, if $\Delta = 0$, then
\begin{equation}
\pmb{F}_a = \pmb{I} + \gamma \pmb{M}_a \overline{\pmb{F}}. \label{eq:f-mat-matching}
\end{equation}
In fact, the proof of Lemma~\ref{lem:t-model} first proceeds by assuming $\Delta = 0$ and showing a one-to-one correspondence between the LSFM matrices $\{ \pmb{F}_a \}_{a \in \mathcal{A}}$ and the LAM's transition matrices $\{ \pmb{M}_a \}_{a \in \mathcal{A}}$.
For arbitrary state representations and LSFMs, Equation~\eqref{eq:f-mat-matching} may not hold and $\Delta > 0$.

\begin{figure}
\centering
\subfigure[Reward-predictive representations can be encoded using different state features.]{
\label{fig:example-reward-pred-mdp}
\tikzset{
mystyle/.style={
  circle,
  inner sep=0pt,
  text width=8mm,
  align=center,
  draw=black,
  fill=white
  }
}
\begin{tikzpicture}

\node[anchor=west,align=right] at (-2.6, 3) {\color{c1}\small $\phi(A) = \pmb{\phi}_{AB}$};
\node[anchor=west,align=right] at (-2.6, 0) {\color{c1}\small $\phi(B) = \pmb{\phi}_{AB}$};

\node[anchor=west,align=right] at (-2.6, 2.2) {\color{c1}\small $\pmb{e}_{A} = \left[ 1, 0, 0, 0, 0 \right]^\top$};
\node[anchor=west,align=right] at (-2.6, -0.8) {\color{c1}\small $\pmb{e}_{B} = \left[ 0, 1, 0, 0, 0 \right]^\top$};

\node[mystyle,draw=black] (A) at (0, 3) {\color{black}$A$};
\node[mystyle,draw=black] (B) at (0, 0) {\color{black}$B$};
\node[mystyle,draw=black] (C) at (4, 3) {\color{black}$C$};
\node[mystyle,draw=black] (D) at (4,.7) {\color{black}$D$};
\node[mystyle,draw=black] (E) at (4,-.7) {\color{black}$E$};

\draw[thick,-latex] (A) -- node[pos=.5,above] {\small $r=0$} (C);
\draw[thick,-latex] (C) edge[thick,out=-30,in=30,looseness=5] node[pos=.5, right] {\small $r=0.5$} (C);

\node[circle,draw=black,fill=black,inner sep=1.5pt](c) at (1.5,0) {};
\draw[thick,-] (B) -- node[pos=.5,above] {\small $r=0$} (c);
\draw[thick,-latex] (c) edge[thick,out=0,in=210] node[pos=.9, left] {\small $p=\frac{1}{2}$~~~~~} (D);
\draw[thick,-latex] (c) edge[thick,out=0,in=150] node[pos=.9, left] {\small $p=\frac{1}{2}$~~~~~} (E);

\draw[thick,-latex] (D) edge[thick,out=-30,in=30,looseness=5] node[pos=.5, right] {\small $r=1$} (D);
\draw[thick,-latex] (E) edge[thick,out=-30,in=30,looseness=5] node[pos=.5, right] {\small $r=0$} (E);

\node[anchor=west,align=left] at (6.3, 4.4) {\color{c0}\small One-hot \\ \color{c0}\small Vectors};
\node[anchor=west,align=left] at (9.5, 4.4) {\color{c1}\small Arbitrary \\ \color{c1}\small Vectors};

\node[anchor=west,align=left] at (6.3, 3) {\color{c0}\small $\pmb{e}_C = \left[ 0, 0, 1, 0, 0 \right]^\top$};
\node[anchor=west,align=left] at (6.3, .7) {\color{c0}\small $\pmb{e}_D = \left[ 0, 0, 0, 1, 0 \right]^\top$};
\node[anchor=west,align=left] at (6.3,-.7) {\color{c0}\small $\pmb{e}_E = \left[ 0, 0, 0, 0, 1 \right]^\top$};

\node[anchor=west,align=left] at (9.5, 3) {\color{c1}\small $\pmb{\phi}_C = \left[ 0, 0.5, 0.5 \right]^\top$};
\node[anchor=west,align=left] at (9.5, .7) {\color{c1}\small $\pmb{\phi}_D = \left[ 0, 1, 0 \right]^\top$};
\node[anchor=west,align=left] at (9.5,-.7) {\color{c1}\small $\pmb{\phi}_E = \left[ 0, 0, 1 \right]^\top$};
\end{tikzpicture}}

\subfigure[Prediction targets of a LAM and LSFM for both state representations.]{
\label{fig:example-reward-pred-tab}
\begin{tabular}{l l l l}
\hline
\small Model & \small Prediction Target & \small \color{c0} with one-hot & \small \color{c1} with arbitrary \\
\hline
\hline
\multirow{2}{*}{\small LAM} & \small $\mathbb{E}_p [ \pmb{\phi} | A]$ & \small \color{c0}$=[ 0, 0, 1, 0, 0 ]^\top$ & \small \color{c1}$= [ 0, 0.5, 0.5 ]^\top$ \\ 
 & \small $\mathbb{E}_p [ \pmb{\phi} | B ] $ & \small \color{c0}$= [ 0, 0, 0, 0.5, 0.5 ]^\top$ & \small \color{c1}$= [ 0, 0.5, 0.5 ]^\top$ \\ 
 \hline
 \multirow{2}{*}{\small LSFM} & \small $\mathbb{E}_p [ \sum_{t=1}^\infty \gamma^{t-1} \pmb{\phi}_t | A ]$ & \small \color{c0}$=[ 1, 0, 9, 0, 0 ]^\top$ & \small \color{c1}$=\pmb{\phi}_{AB} + [ 0, 3.5, 3.5 ]^\top$ \\
 & \small $\mathbb{E}_p [ \sum_{t=1}^\infty \gamma^{t-1} \pmb{\phi}_t | B ]$ & \small \color{c0}$=[ 0, 1, 0, 3.5, 3.5 ]^\top$ & \small \color{c1}$=\pmb{\phi}_{AB} + [ 0, 3.5, 3.5 ]^\top$ \\  
\hline \\
\end{tabular}}
\caption{Real-valued reward-predictive state representations may not encode bisimulations, but support predictions of future expected reward outcomes.
\ref{fig:example-reward-pred-mdp}: In this five-state, example no states are bisimilar.
Each edge is labelled with the reward given to the agent for a particular transition. 
The transition departing state $B$ is probabilistic and leads to state $D$ or $E$ with equal probability.
All other transitions are deterministic.
Two different state representations are considered. 
One representation maps states to one-hot bit vectors and the other representation maps states to real-valued vectors.
\ref{fig:example-reward-pred-tab}: Prediction targets for both LAM and LSFM depend on what state representation is used.
For a one-hot state representation, the LAM and LSFM have different prediction targets for states $A$ and $B$, because a one-hot bit-vector state representation can be used to detect that transition probabilities are different between $A$ and $B$.
In contrast, real valued state representations may lead to equal prediction targets for both LAM and LSFM, because the state features $\pmb{\phi}_C$, $\pmb{\phi}_D$, and $\pmb{\phi}_E$ can hide different transition probabilities. 
The state representation $\phi$ is reward predictive and $\varepsilon_r = \varepsilon_p = \varepsilon_\psi = 0$.}
\label{fig:example-reward-pred}
\end{figure}

The following theorem presents a bound stating that low one-step reward and one-step transition errors lead to state representations that support accurate predictions of future expected reward outcomes.
By Lemma~\ref{lem:t-model}, the following results also apply to LSFMs because low linear SF prediction errors lead to low one-step expected transition errors.

\begin{thm}\label{thm:rollout-bound-lam}
For an MDP, state representation $\phi \to \mathbb{R}^n$, and for all $T \ge 1, s, a_1,...,a_T$,
\begin{equation}
\left| \pmb{\phi}_s^\top \pmb{M}_{a_1} \cdots \pmb{M}_{a_{T-1}} \pmb{w}_{a_T} - \mathbb{E}_p \left[ r_T \middle| s,a_1,...,a_T \right] \right| \le \varepsilon_p \sum_{t=1}^{T-1} M^t W + \varepsilon_r .\label{eq:rollout-bnd-lam}
\end{equation}
\end{thm}

Theorem~\ref{thm:rollout-bound-lam} shows that prediction errors of expected rollouts are bounded linearly in $\varepsilon_r$ and $\varepsilon_p$ and prediction errors tend to zero as $\varepsilon_r$ and $\varepsilon_p$ tend to zero.
Because LAMs are one-step models, prediction errors increase linearly with $T$ if $M \le 1$ as the model is used to generalize over multiple time steps.
Prediction errors may increase exponentially if the transition matrices are expansions and $M > 1$, similar to previously presented bounds~\citep{asadi2018lipschitzmb}.

The following theorem bounds the prediction error of finding a linear approximation of the Q-function $Q^\pi(s,a) \approx \pmb{\phi}_s^\top \pmb{q}_a$ using a state representation $\phi$ and a real valued vector $\pmb{q}_a$.

\begin{thm}\label{thm:approx-val-fn}
For an MDP, state representation $\phi : \mathcal{S} \to \mathbb{R}^n$, any arbitrary abstract policy $\pi \in \Pi_\phi$, and LAM $\{ \pmb{M}_a, \pmb{w}_a \}_{a \in \mathcal{A}}$, there exists vectors $\pmb{v}^\pi$ and $\{ \pmb{q}_a = \pmb{w}_a + \gamma \pmb{M}_a \pmb{v}^\pi \}_{a \in \mathcal{A}}$ such that, for all states $s$ and actions $a$,
\begin{equation}
\left| V^\pi(s) - \pmb{\phi}_s^\top \pmb{v}^\pi \right| \le \frac{\varepsilon_r + \gamma \varepsilon_p \left| \left| \pmb{v}^\pi \right| \right|}{1 - \gamma}~ \text{and}~ \left| Q^\pi(s,a) - \pmb{\phi}_s^\top \pmb{q}_a \right| \le \frac{\varepsilon_r + \gamma \varepsilon_p \left| \left| \pmb{v}^\pi \right| \right|}{1 - \gamma}.
\end{equation}
\end{thm}

By Theorem~\ref{thm:approx-val-fn}, an (approximate) reward-predictive state representation (approximately) generalizes across all abstract policies, because the same state representation can be used to predict the value of every possible abstract policy $\pi \in \Pi_\phi$.
Prediction errors tend to zero as $\varepsilon_r$ and $\varepsilon_p$ tend to zero.
The value prediction error bounds stated in Theorem~\ref{thm:approx-val-fn} are similar to bounds presented by~\cite{bertsekas2011dp} on approximate (linear) policy iteration, because the presented results also approximate value functions using a function that is linear in some basis function $\phi$.
Conforming to these previously presented results on linear value function approximation, prediction errors scale linearly in one-step prediction errors $\varepsilon_\psi$ and $\varepsilon_p$.
Theorems~\ref{thm:approx-val-fn} and~\ref{thm:rollout-bound-lam} show that, by learning a state representation that predicts SFs for policies that select actions uniformly at random, an approximate reward-predictive state representation is obtained.
This state representation generalizes across the entire space of abstract policies, because accurate predictions of each policy's value function are possible if prediction errors are low enough.
Appendix~\ref{app:approx-pfs} presents formal proofs of Theorems~\ref{thm:rollout-bound-lam} and~\ref{thm:approx-val-fn}.

Figure~\ref{fig:example-reward-pred} presents an example highlighting that reward-predictive state representations do not necessarily encode bisimulation relations.
In this example, states $A$ and $B$ are not bisimilar, because the probabilities with which they transition to $C$, $D$, or $E$ are different.
The state representation $\phi$ generalizes across these two states and $\varepsilon_r = \varepsilon_p = \varepsilon_\psi = 0$.
The expected reward sequence for transitioning out of $A$ or $B$ is always $0, 0.5, 0.5, ...$, so both states have equal expected future reward outcomes and the state representation $\phi$ is reward predictive.
However, the state representation is not predictive of the probability with which a particular reward sequence is observed.
For example, the latent state space constructed by $\phi$ would have to make a distinction between state $A$ and $B$ to support predictions stating that a reward sequence of $0,1,1,...$ can be obtained from state $B$ with probability 0.5.
The example in Figure~\ref{fig:example-reward-pred} highlights the difference between the analysis presented in this section and the previous section: 
By relaxing Assumption~\ref{asmpt:one-hot-phi} and considering state-representation functions that map states to real-valued vectors instead of one-hot bit vectors, one may still obtain a reward-predictive state representation, but this representation may not necessarily encode a bisimulation relation.

Note that an (approximate) reward-predictive state representation $\phi: \mathcal{S} \to \mathbb{R}^n$ could in principle map each state to a distinct latent state vector.
As demonstrated by the following simulations, the idea behind generalizing across states is that the dimension $n$ of the constructed latent space is sufficiently small to constrain the LSFM learning algorithm to assign approximately the same latent state vector to different states.
The following section illustrates how (approximate) reward-predictive state representations model generalization across different states.

\subsection{Learning Reward-Predictive Representations}\label{sec:learn-rew-pred}

\begin{figure}
\centering
\includegraphics[scale=1.0]{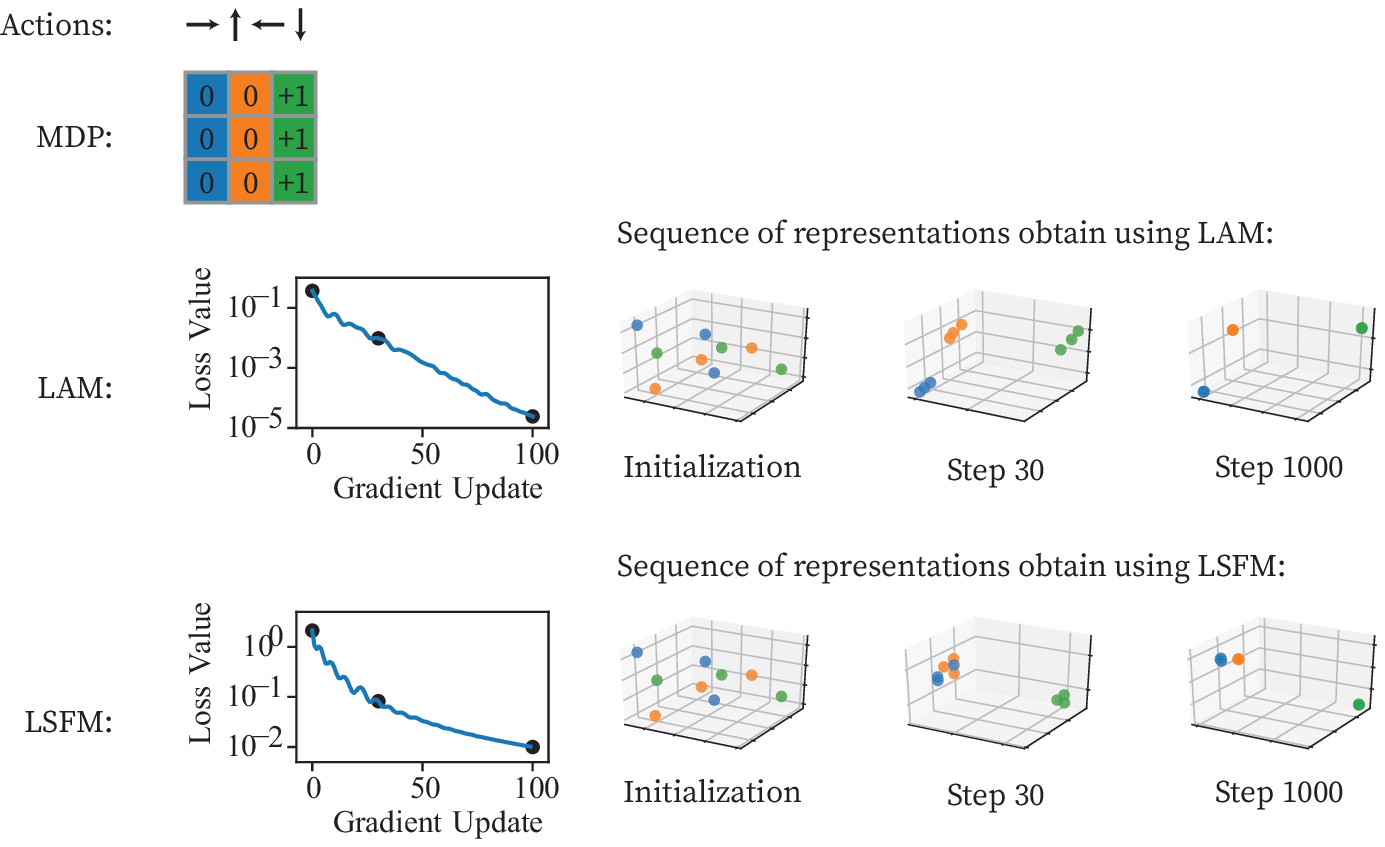}
\caption{In the column world task, learning reward-predictive state representations leads to clustering grid cells by each column.
The top row illustrates a map of the column world task and a colouring of each column.
The middle row presents an experiment that optimizes across different state representations to find a LAM that can be used for accurate one-step reward and one-step expected transition predictions.
Each latent state is plotted as a dot in 3D-space and dots are coloured by the column they correspond to.
At the end of the optimization process, three clusters of the same colour are formed showing that approximately the same latent state is assigned to states of the same column.
The third row repeats the same experiment using an LSFM, which assesses whether the constructed latent state space can be used for accurate one-step reward predictions and SF predictions.
Appendix~\ref{app:matrix-opt} describes this experiment in detail.}
\label{fig:column-world}
\end{figure}

Using the previously presented theoretical results, this section designs a loss objective to learn approximate reward-predictive state representations.
By optimizing a loss function, a randomly chosen state-representation function $\phi: \mathcal{S} \to \mathbb{R}^n$ is iteratively improved until the function $\phi$ can be used to accurately predict future reward sequences.
The cluster plots in Figure~\ref{fig:column-world} illustrate this process: Starting with a random assignment of feature vectors to different grid states, a state representation is iteratively improved until all states of the same column are clustered approximately into the same latent state.
These latent state vectors were only clustered because the loss function assesses whether a state representation is reward predictive.
The fact that states of the same column are assigned approximately to the same latent state is an artifact of this optimization process.
The hyper-parameter $n$ can be understood as the size of the constructed latent space and a bound on the degree of compression of the state space.
For example, if the state space consists of nine different states, setting $n=9$ could result in not compressing the state space and mapping nine states onto nine distinct one-hot bit vectors.
The following experiments explore how choosing a low enough feature dimension leads to compression and generalization across states.

The previous sections present bounds on prediction errors that are parametrized in the worst-case one-step reward-prediction error $\varepsilon_r$ and worst-case linear SF prediction error $\varepsilon_\psi$.
Given only a finite data set of transitions $\mathcal{D} = \left\{ (s_i,a_i,r_i,s_i') \right\}_{i=1}^D$, it may not be possible to compute estimates for $\varepsilon_r$ and $\varepsilon_\psi$ without making further assumptions on the MDP at hand, such as a finite state space or a bounded Rademacher complexity of the transition and reward functions to obtain robust performance on a test data set~\citep{mohri2018foundationsml}.
Because the goal of this project is to study the connections between SFs and different models of generalization across different states, an analysis of how to find provably correct predictions of $\varepsilon_r$ and $\varepsilon_\psi$ given a finite data set $\mathcal{D}$ is beyond the scope of this article.
Instead, the conducted experiments collect a data set $\mathcal{D}$ by sampling trajectories using a policy that selects actions uniformly at random. 
The data set $\mathcal{D}$ is generated to be large enough to cover all state and action pairs, ensuring that all possible transitions and rewards are implicitly represented in this data set.
If the data set does not cover all states and actions, then the resulting reward-predictive state representation may only produce accurate predictions for some reward sequences because the data set does not communicate all aspects of the MDP at hand. 
A study of this interaction between a possibly limited training data set and the resulting model's ability to make accurate predictions is left for future work.

We design a loss objective to learn LSFMs $\mathcal{L}_\text{LSFM}$ that is the sum of three different terms:
The first term $\mathcal{L}_r$ computes the one-step reward prediction error and is designed to minimize the reward error $\varepsilon_r$.
The second term $\mathcal{L}_\psi$ computes the SF prediction error and is designed to minimize the SF error $\varepsilon_\psi$.
The last term $\mathcal{L}_N$ is a regularizer constraining the gradient optimizer to find a state representation that outputs unit norm vectors. 
Empirically, we found that this regularizer encourages the optimizer to find a model with $M \approx 1$ and $W \approx 1$.
(Reward-sequence prediction errors are lower for these values of $M$ and $W$, as stated in Theorem~\ref{eq:rollout-bnd-lam}.)
Given a finite data set of transitions $\mathcal{D} = \left\{ (s_i,a_i,r_i,s_i') \right\}_{i=1}^D$, the formal loss objective is
\begin{equation}
\mathcal{L}_\text{LSFM} = \underbrace{\sum_{i=1}^D \Big( \pmb{\phi}_{s_i}^\top \pmb{w}_{a_i} - r_i \Big)^2}_{=\mathcal{L}_r} + \alpha_\psi \underbrace{ \sum_{i=1}^D \Big| \Big| \pmb{\phi}_{s_i}^\top \pmb{F}_a - \vec{\pmb{y}}_{s_i,a_i,r_i,s_i'} \Big| \Big|_2^2 }_{= \mathcal{L}_\psi} + \alpha_N \underbrace{ \sum_{i=1}^D \Big( \Big| \Big| \pmb{\phi}_{s_i} \Big| \Big|_2^2 - 1 \Big)^2}_{= \mathcal{L}_N}, \label{eq:lsfm-loss-defn}
\end{equation}
for a finite data set of transitions $\mathcal{D} = \left\{ (s_i,a_i,r_i,s_i') \right\}_{i=1}^D$.
In Equation~\eqref{eq:lsfm-loss-defn}, the prediction target 
\begin{equation*}
\vec{\pmb{y}}_{s,a,r,s'} =  \pmb{\phi}_{s}^\top + \gamma \pmb{\phi}_{s'}^\top \overline{\pmb{F}}
\end{equation*}
and $\alpha_\psi, \alpha_\text{N} > 0$ are hyper-parameters.
These hyper-parameters weigh the contribution of each error term to the overall loss objective.
If $\alpha_\psi$ is set to too small a value, then the resulting state representation may only produce accurate one-step reward predictions, but not accurate predictions of longer reward sequences.
This article presents simulations on finite state spaces and represents a state representation as a function $s \mapsto \pmb{e}_s^\top \pmb{\Phi}$ where $\pmb{\Phi}$ is a weight matrix of size $| \mathcal{S} | \times n$.
An approximation of a reward-predictive state representation is obtained by performing gradient descent on the loss objective $\mathcal{L}_\text{LSFM}$ with respect to the free parameters $\{ \pmb{F}_a, \pmb{w}_a \}_{a \in \mathcal{A}}$ and $\pmb{\Phi}$.
For each gradient update, the target $\vec{\pmb{y}}_{s,a,r,s'}$ is considered a constant.
The previously presented bounds show that prediction errors also increase with $\Delta = \max_a || \pmb{I} + \gamma \pmb{M}_a \overline{\pmb{F}}  - \pmb{F}_a ||$.
Minimizing $\mathcal{L}_\psi$ for a fixed state representation $\phi$ minimizes $\Delta$, because
\begin{equation}
\Delta \le c_\phi \mathcal{L}_{\psi}, \label{eq:delta-bnd}
\end{equation}
where $c_\phi$ is a non-negative constant.
Appendix~\ref{app:loss-fn} presents a formal proof for Equation~\eqref{eq:delta-bnd}.

\begin{figure}
\centering

\subfigure[Puddle-world task map]{\label{fig:puddle-world-map}\includegraphics[scale=.89]{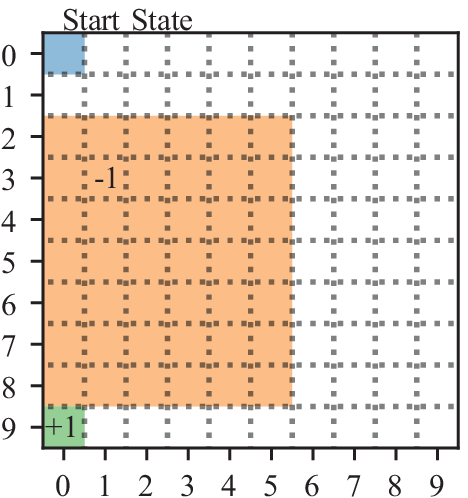}}~
\subfigure[LAM generalization map]{\label{fig:puddle-world-gen-lam}\includegraphics[scale=.89]{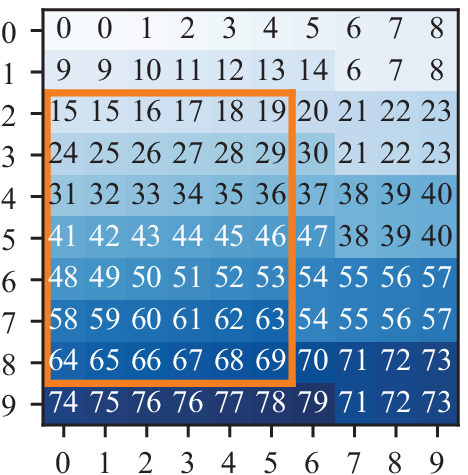}}~
\subfigure[LSFM generalization map]{\label{fig:puddle-world-gen-lsfm}\includegraphics[scale=.89]{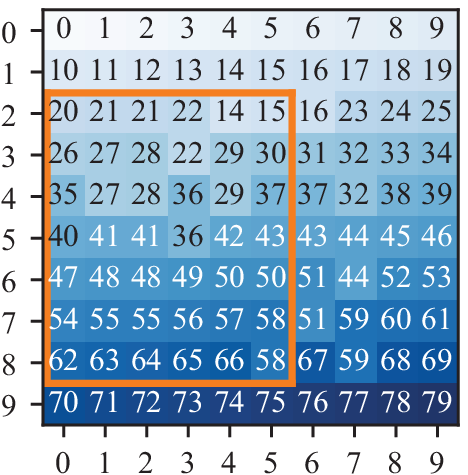}}

\subfigure[Example Rollout, Random]{\label{fig:puddle-world-rollout-init}\hspace{0.4cm}\includegraphics[scale=0.8]{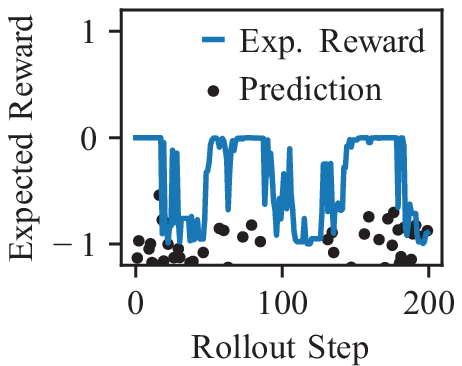}\hspace{0.4cm}}~
\subfigure[Example Rollout, LAM]{\label{fig:puddle-world-rollout-lam}\hspace{0.4cm}\includegraphics[scale=0.8]{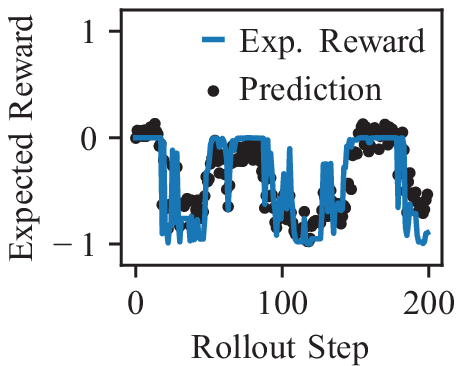}\hspace{0.4cm}}~
\subfigure[Example Rollout, LSFM]{\label{fig:puddle-world-rollout-lsfm}\hspace{0.4cm}\includegraphics[scale=0.8]{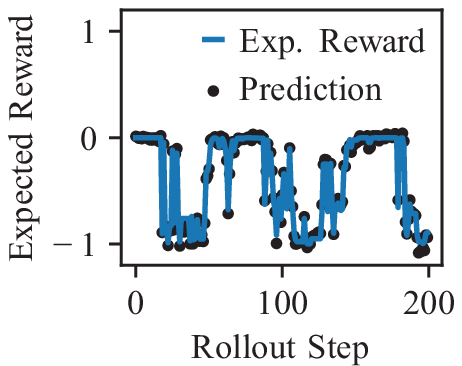}\hspace{0.4cm}}

\subfigure[Reward-Prediction Error]{\label{fig:puddle-world-rollout-error}\label{fig:puddle-world-reward}\hspace{0.4cm}\includegraphics[scale=0.8]{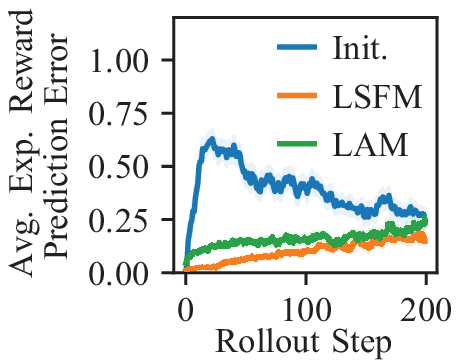}\hspace{0.4cm}}~
\subfigure[LAM Value-Prediction Error]{\label{fig:puddle-world-value-error-lam}\hspace{0.4cm}\includegraphics[scale=0.8]{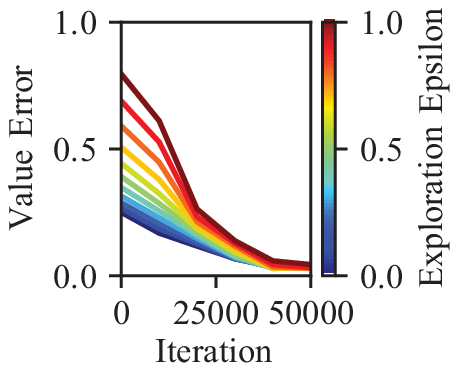}\hspace{0.4cm}}~
\subfigure[LSFM Value-Prediction Error]{\label{fig:puddle-world-value-error-lsfm}\hspace{0.4cm}\includegraphics[scale=0.8]{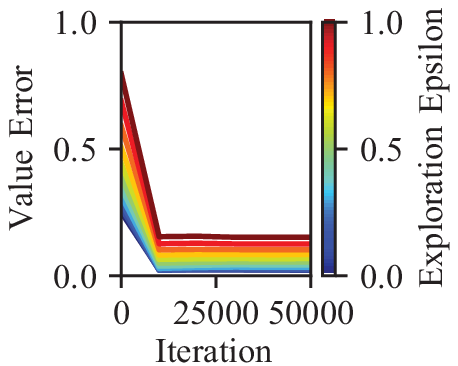}\hspace{0.4cm}}

\caption{Puddle-World Experiment.
\ref{fig:puddle-world-map}: Map of the puddle-world task in which the agent can move up, down, left, or right to transition to adjacent grid cells. 
The agent always starts at the blue start cell and once the green reward cell is reached a reward of $+1$ is given and the interaction sequence is terminated. 
For each transition that enters the orange puddle, a reward of $-1$ is received.
\ref{fig:puddle-world-gen-lam},~\ref{fig:puddle-world-gen-lsfm}: Partitioning obtained by merging latent states into clusters by Euclidean distance. 
\ref{fig:puddle-world-rollout-init},~\ref{fig:puddle-world-rollout-lam},~\ref{fig:puddle-world-rollout-lsfm}: Expected reward and predictions for a randomly chosen 200-step action sequence using a randomly chosen representation, a representation learned with a LAM, and a representation learned with an LSFM. 
\ref{fig:puddle-world-rollout-error}: Average expected reward-prediction errors with standard error for each representation.
\ref{fig:puddle-world-value-error-lam},~\ref{fig:puddle-world-value-error-lsfm}: Optimizing the loss objective results in a sequence of state representations suitable for finding linear approximations of the value functions for a range of different $\varepsilon$-greedy policies.
Appendix~\ref{app:opt} presents more details.}
\label{fig:puddle-world}
\end{figure}

To assess whether minimizing the loss $\mathcal{L}_\text{LSFM}$ leads to approximating reward-predictive state representations, a transition data set was collected from the puddle-world task~\citep{boyan1995generalization}.
Conforming to the previous analysis, the LSFM is compared to a LAM that is trained using a similar loss function, described in Appendix~\ref{app:opt}.

Figure~\ref{fig:puddle-world} presents the puddle-world experiments and the results.
In puddle-world (Figure~\ref{fig:puddle-world-map}), the agent has to navigate from a start state to a goal to obtain a reward of one while avoiding a puddle.
Entering the puddle is penalized with a reward of minus one.
To predict future reward outcomes accurately, a state representation has to preserve the grid position as accurately as possible to predict the location of the different reward cells.

By constraining the latent state space to 80 dimensions, the optimization process is forced to find a compression of all 100 grid cells. 
To analyze across which states the learned reward-predictive state representation generalizes, all feature vectors were clustered using agglomerative clustering.
Two different states that are associated with feature vectors $\pmb{\phi}_s$ and $\pmb{\phi}_{\tilde{s}}$ are merged into the same cluster if
their Euclidean distance $|| \pmb{\phi}_s - \pmb{\phi}_{\tilde{s}} ||_2$ is low.
For example, a randomly chosen representation would randomly assign states to different latent states and the partition map could assign the grid cell at $(0,0)$ and $(9,9)$ to the same latent state.
Figures~\ref{fig:puddle-world-gen-lam} and~\ref{fig:puddle-world-gen-lsfm} plot the obtained clustering as a partition map.
Grid cells are labelled with the same partition index if they belong to the same cluster and colours correspond to the partition index.
To predict reward outcomes accurately, the position in the grid needs to be roughly retained in the constructed latent state space.
The partition maps in Figures~\ref{fig:puddle-world-gen-lam} and~\ref{fig:puddle-world-gen-lsfm} suggest that the learned state representation extracts this property from a transition data set, by generalizing only across neighboring grid cells and tiling the grid space.
Only a transition data set $\mathcal{D}$ was given as input to the optimization algorithm and the algorithm was not informed about the grid-world topology of the task in any other way.

Figures~\ref{fig:puddle-world-rollout-init},~\ref{fig:puddle-world-rollout-lam},~\ref{fig:puddle-world-rollout-lsfm} plot an expected reward rollout and the predictions presented by a random initialization (\ref{fig:puddle-world-rollout-init}), the learned representation using a LAM (\ref{fig:puddle-world-rollout-lam}), and the learned representation using a LSFM (\ref{fig:puddle-world-rollout-lsfm}).
The blue curve plots the expected reward outcome $\mathbb{E}_p [ r_t | s,a_1,...,a_t ]$ as a function of $t$ for a randomly selected action sequence.
Because transitions are probabilistic, the (blue) expected reward curve is smoothed and does not assume exact values of $-1$ or $+1$.
While a randomly initialized state representation produces poor predictions of future reward outcomes (Figures~\ref{fig:puddle-world-rollout-init}), the learned representations produce relatively accurate predictions and follow the expected reward curve (Figures~\ref{fig:puddle-world-rollout-lam} and~\ref{fig:puddle-world-rollout-lsfm}).
Because the optimization process was forced to compress 100 grid cells into a 80-dimensional latent state space, the latent state space cannot preserve the exact grid cell position and thus approximation errors occur.

Figure~\ref{fig:puddle-world-rollout-error} averages the expected reward-prediction errors across 100 randomly selected action sequences.
While a randomly chosen initialization produces high prediction errors, the learned state representations produce relatively low prediction errors of future reward outcomes.
If expected reward-prediction errors are random after 200 time steps, then the $\gamma$-discounted return can be off by at most $0.9^{200} \cdot 1 / (1 - 0.9) \approx 1.4 \cdot 10^{-9}$ after 200 time steps for $\gamma  = 0.9$ and a reward range of $[-1, 1]$.
Hence, planning over a horizon of more than 200 time steps will impact a policy's value estimate insignificantly.
Reward-prediction errors decrease for a randomly chosen state representation (blue curve in Figure~\ref{fig:puddle-world-rollout-error}) because the stochasticity of the task's transitions smooths future expected reward outcomes as the number of steps increases.

While the plots in Figure~\ref{fig:puddle-world} suggest that both LSFMs and LAMs can be used to learn approximate reward-predictive state representations, the LSFM produces lower prediction errors for expected reward outcomes than the LAM and the LAM produces lower value-prediction errors.
Because both models optimize different non-linear and non-convex loss functions, the optimization process leads to different local optima, leading to different performance on the puddle-world task.
While prediction errors are present, Figure~\ref{fig:puddle-world} suggests that both LSFM and LAM learn an approximate reward-predictive state representation and empirically the differences between each model are not significant.

\subsection{Connection to Model-based Reinforcement Learning}

The key characteristic of a model-based RL agent is to build an internal model of a task that supports predictions of future reward outcomes for any arbitrary decision sequence.
Because reward-predictive state representations construct a latent state space suitable for predicting reward sequences for any arbitrary decision sequence, learning reward-predictive state representations can be understood as form of model-based RL.
LSFMs tie SFs to reward-predictive state representations, which support predictions of future reward outcomes for any arbitrary decision sequence.
The presented analysis describes how learning SFs is akin to learning a transition and reward model in model-based RL.

\section{Connections to Value-Predictive Representations}\label{sec:val-pred}

This section first outlines how SFs are related to value-predictive state representations and how learning SFs is akin to model-free learning.
Subsequently, we illustrate how reward-predictive state representations can be re-used across tasks with different transitions and rewards to find an optimal policy while re-using value-predictive state representations may prohibit an agent from learning an optimal policy.
\cite{baretto2017sf} present SFs as a factorization of the Q-value function for an arbitrary fixed policy $\pi$ and demonstrate that re-using SFs across tasks with different reward functions improves the convergence rate of online learning algorithms.
This factorization assumes a state and action representation function $\xi: \mathcal{S} \times \mathcal{A} \to \mathbb{R}^m$ that serves as a basis function for one-step reward predictions and
\begin{equation}
\forall s \in \mathcal{S}, ~ \forall a \in \mathcal{A}, ~ \pmb{\xi}_{s,a}^\top \pmb{w} = \mathbb{E}_p \left[ r(s,a,s') \middle| s,a \right]. \label{eq:r-def}
\end{equation}
Using a state and action representation function, the Q-value function can be factored:
\begin{align}
Q^\pi(s,a) &= \mathbb{E}_{p,\pi} \left[ \sum_{t=1}^\infty \gamma^{t-1} r(s_t, a_t, s_{t+1}) \middle| s_1=s, a_1=a \right]  \\
&= \mathbb{E}_{p,\pi} \left[ \sum_{t=1}^\infty \gamma^{t-1} \pmb{\xi}_{s_t,a_t}^\top \pmb{w} \middle| s_1=s, a_1=a \right] &&(\text{by~\eqref{eq:r-def}}) \\
&=  \mathbb{E}_{p,\pi} \left[ \sum_{t=1}^\infty \gamma^{t-1} \pmb{\xi}_{s_t,a_t}^\top \middle| s_1=s, a_1=a \right]  \pmb{w} \\
&= \left( \pmb{\psi}_\text{SA}^\pi(s_1,a_1) \right)^\top \pmb{w}. &&(\text{where $s_1=s$, $a_1=a$}) \label{eq:sf-factorization}
\end{align}
Equation~\eqref{eq:sf-factorization} assumes the following definition for a SF $\psi_\text{SA}^\pi$:
\begin{equation}
\pmb{\psi}_\text{SA}^\pi(s,a) \overset{\text{def.}}{=} \mathbb{E}_{p,\pi} \left[ \sum_{t=1}^\infty \gamma^{t-1} \pmb{\xi}_{s_t,a_t} \middle| s_1=s, a_1=a \right]. \label{eq:sf-sa-def}
\end{equation}
The state and action SF $\pmb{\psi}_\text{SA}^\pi$ is a basis function that allows accurate predictions of the Q-value function $Q^\pi$.
Consequently, the representation $\pmb{\psi}_\text{SA}^\pi$ is a value-predictive state representation because it is designed to construct a latent feature space that supports accurate predictions of the Q-value function $Q^\pi$.

\begin{figure}
\centering
\tikzset{
mystyle/.style={
  circle,
  inner sep=0pt,
  text width=8mm,
  align=center,
  draw=black,
  fill=white
  }
}
\begin{tikzpicture}

\node[mystyle,draw=black] (A) at (1, 2) {\color{black}$A$};
\node[mystyle,draw=black] (B) at (1, 0) {\color{black}$B$};
\node[mystyle,draw=c1,fill=c1] (C) at (4.5, 2) {\color{black}$C$};
\node[mystyle,draw=c2,fill=c2] (D) at (4.5, 0) {\color{black}$D$};

\draw[thick,-latex] (C) edge[thick,out=-30,in=30,looseness=5] node[pos=.5, right] {\small $a,b,r=0.5$} (C);
\draw[thick,-latex] (D) edge[thick,out= 20,in=50,looseness=10] node[pos=.5, right] {\small $a,r=1$} (D);
\draw[thick,-latex] (D) edge[thick,out=-50,in=-20,looseness=10] node[pos=.5, right] {\small $b,r=0$} (D);

\draw[thick,-latex] (A) -- node[pos=.2,above] {\small $b, r=0$} (C);
\draw[thick,-latex] (A) -- node[pos=.1,right] {~~\small $a, r=0$} (D);
\draw[thick,-latex] (B) -- node[pos=.2,below] {\small $b, r=0$} (D);
\draw[thick,-latex] (B) -- node[pos=.1,right] {~~\small $a, r=0$} (C);

\node[anchor=west,align=left] at (-7.5, 3.2) {\small Value predictive:};
\node[anchor=west,align=left] at (-7.5, 2.3) {\small $Q^{\overline{\pi}}(A,a) = \color{c0} \frac{\gamma}{1 - \gamma} 0.5$};
\node[anchor=west,align=left] at (-7.5, 1.7) {\small $Q^{\overline{\pi}}(A,b) = \color{c0} \frac{\gamma}{1 - \gamma} 0.5$};
\node[anchor=west,align=left] at (-7.5, .3) {\small $Q^{\overline{\pi}}(B,a) = \color{c0} \frac{\gamma}{1 - \gamma} 0.5$};
\node[anchor=west,align=left] at (-7.5, -.3) {\small $Q^{\overline{\pi}}(B,b) = \color{c0} \frac{\gamma}{1 - \gamma} 0.5$};

\node[anchor=west,align=left] at (-4, 3.2) {\small Reward predictive:};
\node[anchor=west,align=left] at (-4, 2.3) {\small $\pmb{\psi}^{\overline{\pi}}(A,a) = \pmb{\phi}_A + \color{c2}  \frac{1}{1 - \gamma} \pmb{\phi}_D$};
\node[anchor=west,align=left] at (-4, 1.7) {\small $\pmb{\psi}^{\overline{\pi}}(A,b) = \pmb{\phi}_A + \color{c1}  \frac{1}{1 - \gamma} \pmb{\phi}_C$};
\node[anchor=west,align=left] at (-4, 0.3) {\small $\pmb{\psi}^{\overline{\pi}}(B,a) = \pmb{\phi}_B + \color{c1}  \frac{1}{1 - \gamma} \pmb{\phi}_C$};
\node[anchor=west,align=left] at (-4, -.3) {\small $\pmb{\psi}^{\overline{\pi}}(B,b) = \pmb{\phi}_B + \color{c2}  \frac{1}{1 - \gamma} \pmb{\phi}_D$};

\end{tikzpicture}

\caption{Value-predictive state representations may prohibit an agent from learning an optimal policy.
In this MDP, the agent can choose between action $a$ and action $b$.
All transitions are deterministic and each edge is labelled with the reward given to the agent.
If a uniform-random action-selection policy is used to construct a value-predictive state representation, then both states $A$ and $B$ will have equal Q-values.
A reward-predictive state representation would always distinguish between $A$ and $B$, because at state $A$ the action sequence $b,a,a...$ leads to a reward sequence of $0, 0.5, 0.5, ...$ while at state $B$ the action sequence $b,a,a,...$ leads to a reward sequence of $0, 1, 1, ...$.
An LSFM detects that states $A$ and $B$ should not be merged into the same latent state, because the states have different SFs.
The optimal policy is to select action $a$ at state $A$, and action $b$ at state $B$ and then collect a reward of one at state $D$ by repeating action $a$. 
If an agent uses a reward-predictive state representation, then the optimal policy could be recovered.
If an agent uses a value-predictive state representation, the agent would be constrained to not distinguish between states $A$ and $B$ and cannot recover an optimal policy.}
\label{fig:value-pred-counter-example}
\end{figure}
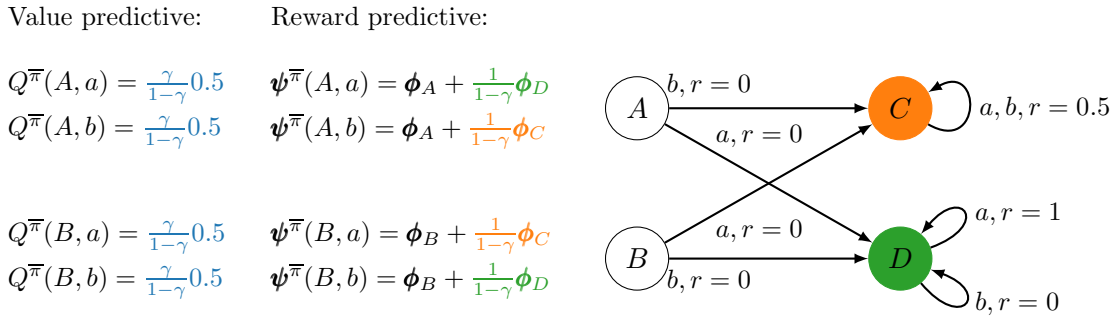

Figure~\ref{fig:value-pred-counter-example} illustrates that value-predictive state representations generalize across different states differently than reward-predictive state representations.
The counter example presented in Figure~\ref{fig:value-pred-counter-example} demonstrates that it is not always possible to construct an optimal policy using a value-predictive state representation.
If a sub-optimal policy is used, value-predictive state representations may alias states that have different optimal actions prohibiting an intelligent agent from finding an optimal policy.
In this case the value-predictive state representation has to be adjusted for the agent to be able to find an optimal policy, a phenomenon that has been previously described by~\cite{russek2017predictive}.
In contrast, reward-predictive state representations generalize across the entire policy space (Theorem~\ref{thm:approx-val-fn}) and allow an agent to find either an optimal policy or close to optimal policy in the presence of approximation errors.

In the following section, we show that learning SFs is akin to learning a value function in model-free RL and the presented argument ties SFs to model-free RL.
Subsequently, we demonstrate that reward-predictive state representations can be re-used across tasks with different transitions and rewards to find an optimal policy on two transfer examples in Sections 5.2 and 5.3.
While value-predictive state representations may prohibit an agent from learning an optimal policy (Figure~\ref{fig:value-pred-counter-example}), we illustrate in which cases reward-predictive state representations overcome this limitation.

\subsection{Connection to Linear Temporal Difference Learning}\label{sec:td-learning}

Algorithms that learn SFs can be derived similarly to linear TD-learning~\citep[Chapter 8.4]{sutton98}.
In linear TD-learning algorithms such as linear Q-learning or SARSA, all Q-values are represented with
\begin{equation}
Q^\pi(s,a; \pmb{\theta}) = \pmb{\xi}_{s,a}^\top \pmb{\theta},
\end{equation}
where $\pmb{\theta}$ is a real-valued weight vector that does not depend on a state $s$ or action $a$.
Linear TD-learning learns the parameter vector $\pmb{\theta}$ by minimizing the mean squared value error
\begin{equation}
\text{VE}( \pmb{\theta} ) = \sum_{s,a,r,s'} \mu(s,a,r,s') \left( Q_{\pmb{\theta}}^\pi(s,a; \pmb{\theta}) - y_{s,a,r,s'} \right)^2. \label{eq:msve}
\end{equation}
Equation~\eqref{eq:msve} averages prediction errors with respect to some distribution $\mu$ with which transitions $(s,a,r,s')$ are sampled.
The prediction target 
\begin{equation}
y_{s,a,r,s'} = r + \gamma \sum_{a'} b(s',a') Q^\pi(s', a'; \pmb{\theta})
\end{equation}
varies by which function $b$ is used.
For example, to find the optimal policy linear Q-learning uses an indicator function $b(s,a) = \pmb{1}[a = \arg \max_a Q^\pi(s, a;\pmb{\theta})]$ so that $y_{s,a,r,s'} = r + \gamma \max_{a'} Q^\pi(s', a';\pmb{\theta})$.
For Expected SARSA~\citep[Chapter 6.6]{sutton2018rlbook}, which evaluates a fixed policy $\pi$, the target can be constructed using $b(s,a) = \pi(s,a)$, where $\pi(s,a)$ specifies the probability with which $a$ is selected at state $s$.
When computing a gradient of $\text{VE}( \pmb{\theta} )$ the \emph{prediction target} $y_{s,a,r,s'}$ is considered a constant.
For an observed transition $(s,a,r,s')$, the parameter vector is updated using the rule
\begin{equation}
\pmb{\theta}_{t+1} = \pmb{\theta}_t + \alpha \left(  Q^\pi(s,a;\pmb{\theta}_t) - y_{s,a,r,s'} \right) \pmb{\xi}_{s,a},  \label{eq:q-learning}
\end{equation}
where $\alpha$ is a learning rate and the subscript $t$ tracks the update iteration.
A SF-learning algorithm can be derived by defining the mean squared SF error~\citep{lehnert2017sf}
\begin{equation}
\text{SFE}( \pmb{\psi}^\pi_\text{SA} ) =  \sum_{s,a,r,s'} \mu(s,a,r,s') || \pmb{\psi}^\pi_\text{SA}(s,a) - \vec{\pmb{y}}_{s,a,r,s'} ||^2.
\end{equation}
Because the SF $\pmb{\psi}^\pi_\text{SA}(s,a)$ is a vector of dimension $m$, the target 
\begin{equation}
\vec{\pmb{y}}_{s,a,r,s'} = \pmb{\xi}_{s,a} + \gamma \sum_{a'} b(s',a') \pmb{\psi}^\pi_\text{SA}(s',a')
\end{equation}
 is also a vector but can be constructed similarly to the usual value-prediction target $y_{s,a,r,s'}$.
Assuming that SFs are approximated linearly using the basis function $\xi$,
\begin{equation}
\pmb{\psi}^\pi_\text{SA}(s,a; \pmb{G}) = \pmb{G} \pmb{\xi}_{s,a}, \label{eq:sf-assumption}
\end{equation}
where $\pmb{F}$ is a square matrix.
Computing the gradient of $\text{SFE}( \pmb{\psi}^\pi )$ with respect to $\pmb{F}$ results in an update rule similar to linear TD-learning:
\begin{equation}
\pmb{G}_{t+1} = \pmb{G}_t + \alpha \left( \pmb{\psi}^\pi_\text{SA}(s,a; \pmb{G}_t) - \vec{\pmb{y}}_{s,a,r,s'} \right) \pmb{\xi}_{s,a}^\top . \label{eq:sf-update-rule}
\end{equation}
Assuming the reward condition in Equation~\eqref{eq:r-def} holds, both linear TD-learning and SF-learning produce the same value-function sequence.

\begin{prop}[Linear TD-learning and SF-learning Equivalence]\label{prop:sf-q-equiv}
Consider an MDP and a basis function $\xi$ such that $r(s,a) =  \pmb{\xi}_{s,a}^\top \pmb{w}$ for all states $s$ and actions $a$.
Suppose both iterates in Equation~\eqref{eq:q-learning} and in Equation~\eqref{eq:sf-update-rule} use the same function $b$ to construct prediction targets and are applied for the same trajectory $(s_1,a_1,r_1,s_2,a_2,...)$.
If $\pmb{\theta}_0 = \pmb{G}_0 \pmb{w}$, then
\begin{equation}
\forall t > 0, ~ \pmb{\theta}_t = \pmb{G}_t \pmb{w}.
\end{equation}
\end{prop}

Proposition~\ref{prop:sf-q-equiv} proves that both linear TD-learning and linear SF-learning generate identical value-function estimates on the same trajectory.
Appendix~\ref{app:q-sf-learning-pf} proves Proposition~\ref{prop:sf-q-equiv}.
Because linear TD-learning need not converge to an optimal solution, the SF iterate in Equation~\eqref{eq:sf-update-rule} also need not converge to an optimal solution.
The tabular case, in which convergence can be guaranteed, is a sub-case of the presented analysis: 
For finite state and action spaces, a basis function $\xi$ can be constructed that outputs a one-hot bit vector of dimension $n$, where $n$ is the number of all state and action pairs.
In this case, each weight in the parameter vector $\pmb{\theta}$ corresponds to the Q-value for a particular state and action pair.
Similarly, each row in the matrix $\pmb{F}$ corresponds to the SF vector for a particular state and action pair.
In this light, learning SFs is akin to learning a value function in model-free RL.

\subsection{Generalization Across Transition and Reward Functions}\label{sec:transfer}

One key distinction between value- and reward-predictive state representations is their ability to generalize across different transition and reward functions.
While prior work on SFs~\citep{barreto2016successor} and adversarial IRL~\citep{fu2018adverserial} separately model the reward function from the transition function and observed policy, reward-predictive state representations only model equivalence relations between states separately from the transition and reward model.
Consequently, reward-predictive state representations extract equivalences between different state's transitions and one-step rewards, reward-predictive state representations can be reused across tasks that vary in their transition and reward functions~\citep{lehnert2019reward}.
Figure~\ref{fig:transfer-experiment} presents a transfer experiment highligting that reusing a previously learned reward-predictive state representation allows an intelligent agent to learn an optimal policy using less data.

This experiment uses two grid-world tasks (Figure~\ref{fig:transfer-maps}):
For Task A, a transition data set $\mathcal{D}_A$ is collected.
A reward-predictive state representation is learned using an LSFM and a value-predictive state representation is learned using a form of Fitted Q-iteration~\citep{riedmiller2005neural}.
These state representations are then reused without modification to learn an optimal policy for Task B given a data set $\mathcal{D}_B$ collected from Task B.
Both data sets are generated by performing a random walk from a uniformly sampled start state to one of the rewarding goal states. 
In both tasks, the agent can transition between adjacent grid cells by moving up, down, left, or right, but cannot transition across a barrier.
Transitions are probabilistic, because, with a 5\% chance, the agent does not move after selecting any action.

Figure~\ref{fig:transfer-dynamic-prog} presents two heuristics for clustering all 100 states into 50 latent states.
The first heuristic constructs a reward-predictive state representation by joining states into the same latent state partition if they are directly connected to another.
Because both tasks are navigation tasks, partitioning the state space in this way leads to approximately perserving the agent's location in the grid.
The second heuristic constructs a value-predictive state representation by joining states that have approximately the same optimal Q-values.
Because Q-values are discounted sums of rewards, Q-values decay as one moves further away from a goal cell. 
This situation leads to different corners being merged into the same state parition (for example grid cell $(0,0)$ is merged with $(0,9)$) and an overall more fragmented partitioning that does not perserve the agent's location.
Because both state representations are computed for Task A, both state representations can be used to compute an optimal policy for Task A.
For Task B, an optimal policy cannot be obtained using the value-predictive state representation.
For example, both grid cells at $(0,0)$ and $(0,9)$ have different optimal actions in Task B but are mapped to the same latent state.
Consequently, an optimal action cannot be computed using the previously learned value-predictive state representation.
Because each grid cell has a different optimal action, an optimal abstract policy mapping each latent state to an optimal action cannot be found and the navigation Task B cannot be completed within 5000 time steps if the value-predictive state representation is used (Figure~\ref{fig:transfer-dynamic-prog}, right panel).
In contrast, the reward-predictive state representation can be used, because it approximately models the grid-world topology.
For Task B, each latent state has to be associated with different one-step rewards and latent transitions, but it is still possible to obtain an optimal policy using this state representation and complete the navigation task quickly.

\begin{figure}
\centering

\subfigure[Maps of Transfer Grid Worlds]{\label{fig:transfer-maps}
\scalebox{0.92}{
\begin{tikzpicture}
\definecolor{c0}{RGB}{31,119,180}
\definecolor{c1}{RGB}{255,127,14}
\definecolor{c2}{RGB}{44,160,44}

\node[anchor=north east] at (.25,-0.1) {\small Task A:};
\node[anchor=north west] at (0,0) {\includegraphics[scale=.5]{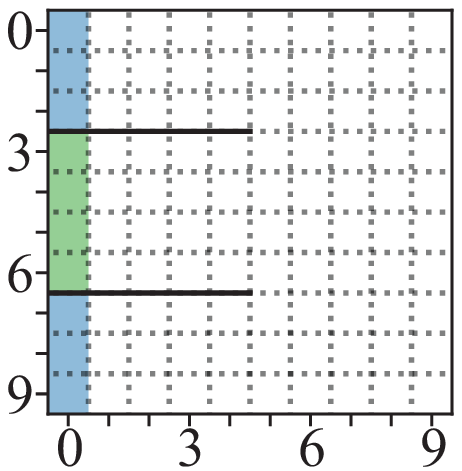}}; 
\node[anchor=north east] at (4.25,-0.1) {\small Task B:};
\node[anchor=north west] at (4,0) {\includegraphics[scale=.5]{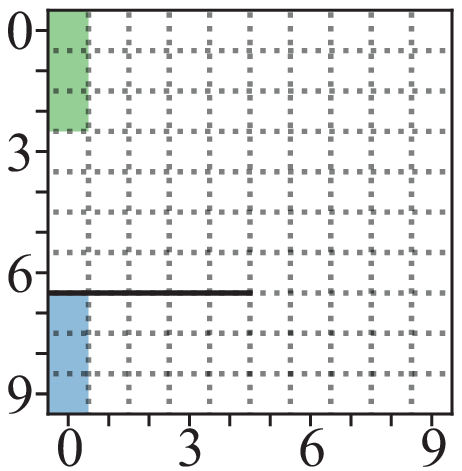}}; 

\node[fill=c0,opacity=0.5] at (7,-0.5) {};
\node[fill=c2,opacity=0.5] at (7,-1.0) {};
\draw (6.88,-1.5) -- (7.12,-1.5);
\node[anchor=west] at (7.1,-0.5) {\small : Start state};
\node[anchor=west] at (7.1,-1.0) {\small : Goal state (+1 reward)};
\node[anchor=west] at (7.1,-1.5) {\small : Barrier};
\end{tikzpicture}}}

\subfigure[State representations obtained using a clustering heuristic or optimal Q-values of Task A]{\label{fig:transfer-dynamic-prog}
\scalebox{0.92}{
\begin{tikzpicture}
\definecolor{c0}{RGB}{31,119,180}
\definecolor{c1}{RGB}{255,127,14}
\definecolor{c2}{RGB}{44,160,44}

\node[] at (0, 0) {};
\node[] at (16, 0) {};

\node[anchor=west] at (0.0,0) {\includegraphics[scale=.8]{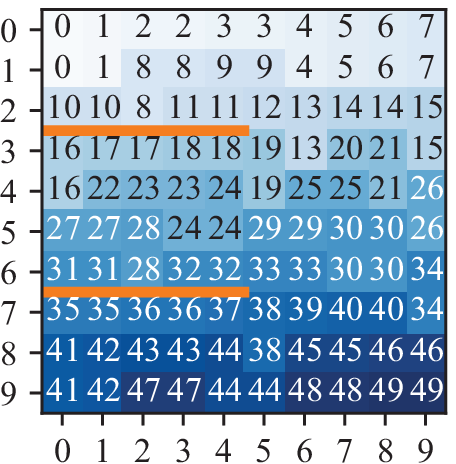}};
\node[anchor=west] at (4.5,0) {\includegraphics[scale=.8]{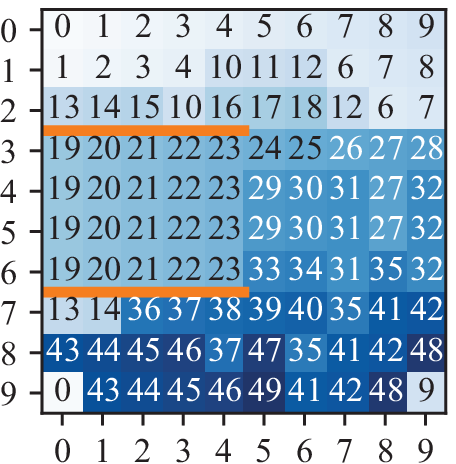}};
\node[anchor=west] at (8.4,.235) {\includegraphics[scale=0.9]{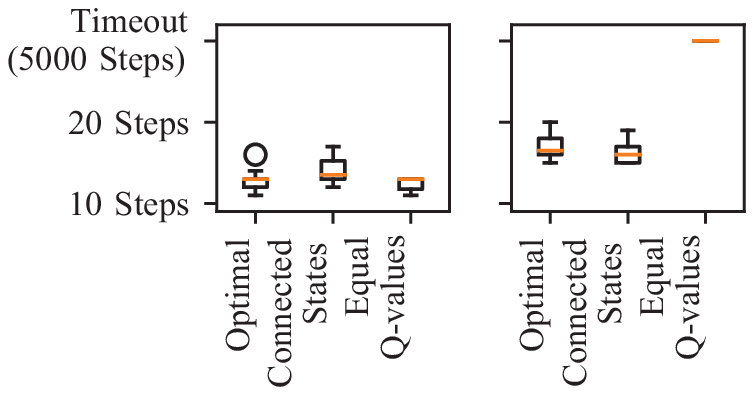}};

\node[anchor=west,text width=3.5cm] at (.0,2.3) {\small Reward Predictive (Connected States)};
\node[anchor=west,text width=3.5cm] at (4.5,2.3) {\small Value Predictive (App. Equal $Q^{\pi^*}$)};
\node[anchor=west,text width=2cm] at (10.6,2.3) {\small Performance in Task A};
\node[anchor=west,text width=2cm] at (13.4,2.3) {\small Performance in Task B};
\node[anchor=west,text width=8cm] at (9.25,-1.7) {\small Simulations are repeated 20 times.};
\end{tikzpicture}}}

\subfigure[State partitions obtained through learning on a transition data set $\mathcal{D}$.]{\label{fig:transfer-learning}
\scalebox{0.92}{
\begin{tikzpicture}
\definecolor{c0}{RGB}{31,119,180}
\definecolor{c1}{RGB}{255,127,14}
\definecolor{c2}{RGB}{44,160,44}

\node[] at (0, 0) {};
\node[] at (16, 0) {};

\node[anchor=north west] at (0.0,0) {\includegraphics[scale=.8]{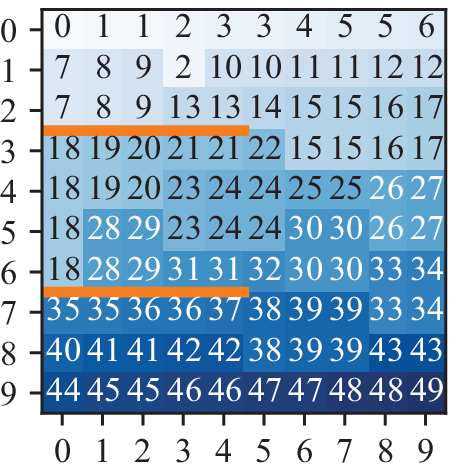}};
\node[anchor=north west] at (4.5,0) {\includegraphics[scale=.8]{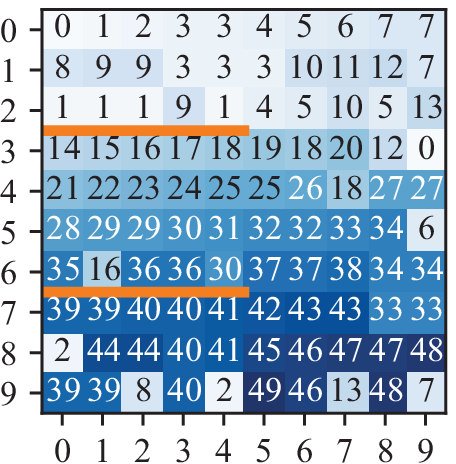}};
\node[anchor=north west] at (8.4,.42) {\includegraphics[scale=0.9]{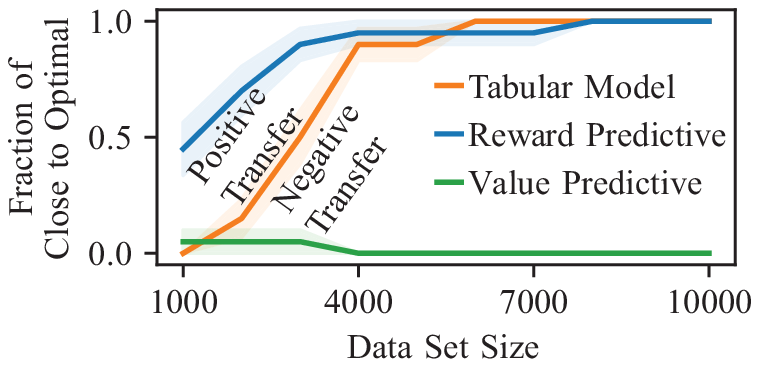}};

\node[anchor=south west,text width=4.5cm] at (0,-.25) {\small Reward Predictive (LSFM)};
\node[anchor=south west,text width=4.5cm] at (4.5,-.25) {\small Value Predictive (Fitted Q)};
\node[anchor=south west,text width=6cm] at (9.55,-.21) {\small Data Set Size Needed to Solve Task B};

\end{tikzpicture}}}

\caption{Reward-predictive representations generalize across variations in transitions and rewards.
\ref{fig:transfer-dynamic-prog}: The left panels plot state partitions obtained by clustering connected states or states with equal optimal Q-values in Task A (\ref{fig:transfer-maps}).
The right panels plot the number of times steps a policy, which uses each representation, needs to complete Task B (\ref{fig:transfer-maps}).
\ref{fig:transfer-learning}: The left panels plot partitions obtained by clustering latent states of a reward-predictive and value-predictive representation.
The right panel plots how often one out of 20 transition data sets can be used to find an optimal policy as a function of the data set size.
By reusing the learned reward-predictive representation, an agent can generalize across states and compute an optimal policy using less data than a tabular model-based RL agent. 
Reusing a value-predictive representation leads to poor performance, because this representation is only predictive of Task A's optimal policy.}
\label{fig:transfer-experiment}
\end{figure}

Figure~\ref{fig:transfer-learning} repeats a similar experiment, but learns a state representation using either an LSFM to find a reward-predictive state representation or a modification of Fitted Q-iteration to find a value-predictive state representation.
The two left panels in Figure~\ref{fig:transfer-learning} plot a partitioning of the state space that was obtained by clustering all latent state feature vectors using agglomerative clustering.
In this experiment, the latent feature space was set to have 50 dimensions.
One can observe that the state representation learned using an LSFM qualitatively corresponds to clustering connected grid cells.
The learned value-predictive state representation qualitatively resembles a clustering of states by their optimal Q-values, because Fitted Q-iteration optimizes this state representation to predict the optimal value function as accurately as possible.
Both state representations are tested on Task B using the following proceedure: First, a data set $\mathcal{D}_B$ was collected of a fixed size. 
Then, Fitted Q-iteration was used to compute the optimal policy for Task B using the previously learned state representation as a basis function such that $Q^{\pi^*}(s,a) \approx \pmb{\phi}_s^\top \pmb{q}_a$ where $\pmb{q}_a$ is a weight vector and $\phi(s) = \pmb{\phi}_s$.
The state representation $\phi$ trained on Task A is not modified to obtain an optimal policy in Task B.
If the training data set $\mathcal{D}_B$ obtained from Task B is too small, then the data set may not provide enough information to find an optimal policy.
In this case, Fitted Q-iteration converged to a sub-optimal policy.
Because the data sets $\mathcal{D}_B$ are generated at random, one may find that sampling a data set of 2000 transitions may lead to an optimal solution often, but not all the time.

The right panel in Figure~\ref{fig:transfer-learning} plots the dependency of being able to find an optimal policy as a function of the transition data set.
For each data set size, twenty different data sets were sampled and the y-axis plots the fraction (with standard error of measure) of how often using this data set leads to a close-to-optimal policy.
A close-to-optimal policy solves the navigation task in at most 22 time steps.
The orange curve is computed using a tabular model-based agent, which constructs a transition and reward table using the sampled data set and solves for an optimal policy using value iteration.
Reusing a reward-predictive state representation in Task B often leads to finding an optimal policy for small data set sizes (the blue curve in Figure~\ref{fig:transfer-learning}, right panel). 
Because the training data set $\mathcal{D}_B$ is only used to inform different latent transitions and rewards, this agent can generalize across different states and reuse what it has learned without having to observe every possible transition.
This behavior leads to better performance than the tabular model-based baseline algorithm, which does not generalize across different states and constructs a transition table for Task B and computes an optimal policy using value iteration~\citep[Chapter 4.4]{sutton2018rlbook}.
Reusing the learned value-predictive state representation leads to finding a sub-optimal policy in almost all cases (green curve in Figure~\ref{fig:transfer-learning}, right panel).
The value-predictive state representation is optimized to predict the Q-value function of the optimal policy in Task A.
Because Task B has a different optimal policy, reusing this representation does not lead to an optimal policy in Task B, because the previously learned representation is explicitly tied to Task A's optimal policy.
Note that any trial that did not find a close-to-optimal policy that completes the task in 22 time steps did also not finish the task and hit the timeout threshold of 5000 time steps.
Appendix~\ref{app:transfer} presents additional details on the experiments conducted in Figure~\ref{fig:transfer-experiment}.

\subsection*{5.3 Reward-Predictive Representations Encode Task Relevant State Information}\label{sec:combination-lock}

In this section, we present the last simulation result illustrating which aspect of an MDP reward-predictive state representations encode.
Figure~\ref{fig:combination-lock-tasks} illustrates the combination lock task, where an agent rotates three different numerical dials to obtain a rewarding number combination.
In this task, the state is defined as a number triple and each dial has five different positions labelled with the digits zero through four.
For example, if Action 1 is chosen at state $(0,1,4)$, then the left dial rotates by one step and the state changes to $(1,1,4)$.
A dial that is set to four will rotate to zero.
For example, selecting Action 2 at state $(0,4,4)$ will result in a transition to state $(0,0,4)$.

In the Training Lock task (Figure~\ref{fig:combination-lock-tasks}, left schematic), the right dial is ``broken'' and spins randomly after each time step. 
Consequently, Action 3 (red arrow) only causes a random change in the right dial and the agent can only use Action 1 (blue arrow) or Action 2 (green arrow) to manipulate the state of the lock.
When the agent enters a combination where the left and middle dials are set to four, a reward of $+1$ is given, otherwise the agent receives no reward.\footnote{Specifically, in the Training Lock task, the rewarding states are $(4,4,0),(4,4,1), \dots (4,4,4)$.}
While the combination lock task has $5 \cdot 5 \cdot 5 = 125$ different states, the state space of the Training Lock can be compressed to $5 \cdot 5 = 25$ latent states by ignoring the position of the randomly changing right dial, because the right dial is neither relevant for maximizing reward nor predicting reward sequences.

The Test Lock 1 and Test Lock 2 tasks (Figure~\ref{fig:combination-lock-tasks}, center and right schematics) differ from the training task in that the rewarding combination is changed and the rotation direction of the left dial is reversed (Action 1, blue arrow), resembling a change in both transition and reward functions.
While in Test Lock 1, the right dial still spins at random, in Test Lock 2 the middle dial rotates at random instead and the right dial becomes relevant for maximizing reward.
Both test tasks can also be compressed into 25 latent states by ignoring the position of the randomly rotating dial, but in Test Lock 2 this compression would be constructed differently than in Test Task 1 or in the Training Lock. 

\begin{figure}
\centering
\subfigure[Combination Lock Tasks]{\label{fig:combination-lock-tasks}\includegraphics[scale=1.25]{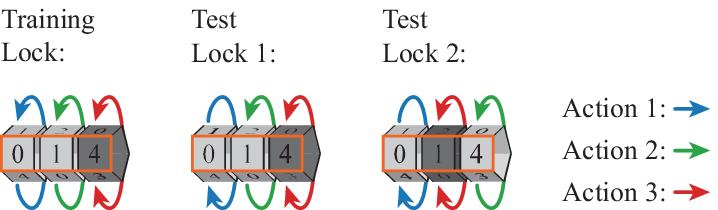}}

\subfigure[Performance of Each Model at Transfer]{\label{fig:combination-lock-ep-len}\includegraphics[scale=.9]{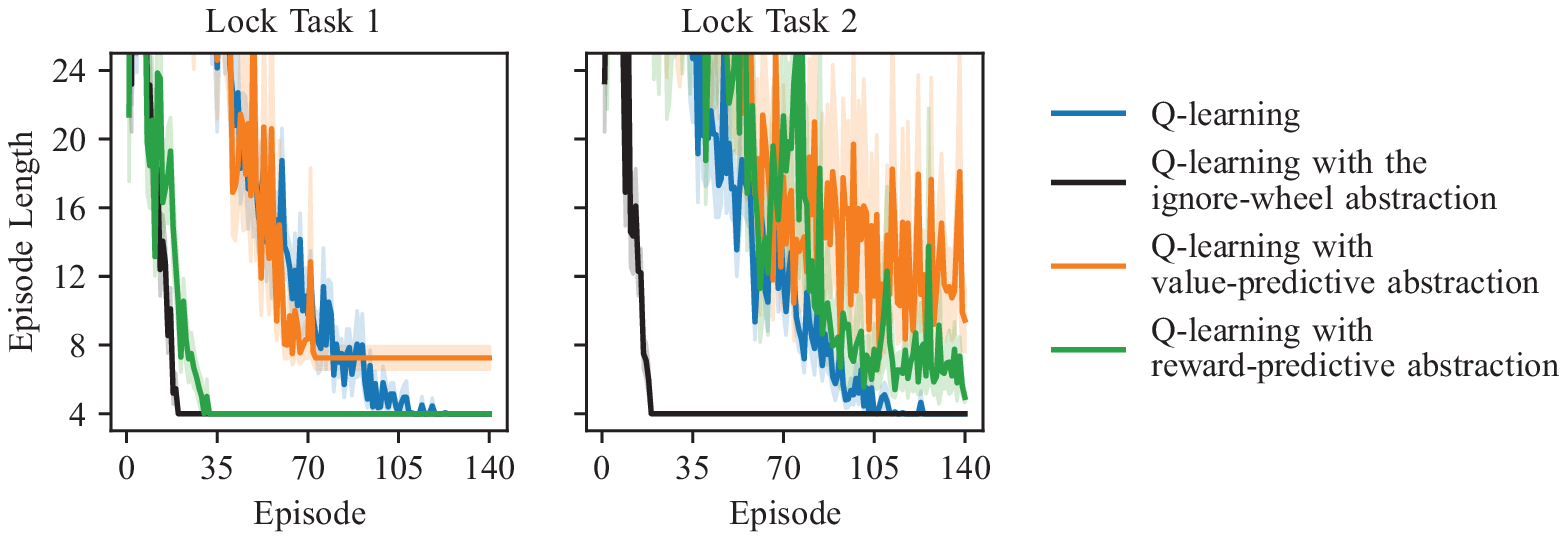}}

\caption{Combination Lock Transfer Experiment.
\ref{fig:combination-lock-tasks}: In the combination lock tasks, the agent decides between three different actions to rotate each dial by one digit. 
Each dial has five sides labelled with the digits zero through four.
The dark gray dial is ``broken'' and spins at random at every time step.
In the training task, any combination setting the left and middle dial to four are rewarding. 
In Test Lock 1, setting the left dial to two and the middle dial to three is rewarding and simulations were started by setting the left dial to two and the middle dial to four.
In Test Lock 2, setting the left dial to two and the right dial to three is rewarding and simulations were started by setting the left dial to two and the right dial to four.
\ref{fig:combination-lock-ep-len}:
Each panel plots the episode length of the Q-learning algorithm on Lock Task 1 and Lock Task 2 averaged over 20 repeats.
Note that Q-learning with the ignore-wheel abstraction uses a different abstraction in Test Lock 1 and Test Lock 2. 
In Test Lock 1, the ignore-wheel abstraction ignores the right dial.
In Test Lock 2, the ignore-wheel abstraction ignores the middle dial.
Please refer to Appendix~\ref{app:combination-lock} for a detailed description of the experiment implementation.
}
\label{fig:combintation-lock}
\end{figure}

To determine if a reward- or value-predictive state representation can be re-used to accelerate learning in a previously unseen task, we compute a two-state representations of each type for the Training Lock MDP.
To assess if each state abstraction can be re-used, they are both tested by compressing the state space of a Q-learning agent~\citep{Watkins1992aa} to learn an optimal policy in Test Lock 1 and Test Lock 2.
Any resulting performance changes are then indicative of the state abstraction's ability to generalize from the Training Lock task to Test Lock 1 and Test Lock 2.
If a state representation can generalize information from the training to the test task, then Q-learning with a state abstraction should converge to an optimal policy faster than a Q-learning agent that does not use any state abstraction.
To combine the tabular Q-learning algorithm with a state abstraction, we assume in this section that a state abstraction function maps each state to a set of discrete latent states.
Furthermore, before updating the Q-learning agent with a transition $(s,a,r,s')$, this transition is mapped to a latent space transition $(\phi(s),a,r',\phi(s))$ using the respective state abstraction function $\phi$.
Because the latent state space is smaller than the task's original state space, one would expect that using a state abstraction function results in faster convergence, because any Q-value update is generalized across all states that are mapped to the same latent state.

The reward-predictive state representation is computed using the training task's transition and reward table with the same procedure used for the column-world task presented in Figure~\ref{fig:column-world}.
(Please also refer to Appendix~\ref{app:combination-lock} for implementation details.)
To obtain a state-representation function mapping states to discrete latent states, the real-valued state representation function $\phi_\text{real-valued}: \mathcal{S} \to \mathbb{R}^n$ is then further refined by clustering the set of feature vectors $\{ \phi_\text{real-valued}(s) | s \in \mathcal{S} \}$ into discrete clusters using agglomerative clustering.
Each cluster then becomes a separate latent state in the resulting latent state space.

The value-predictive state representation is computed by associating two states with the same latent state if the optimal policy has equal Q-values at both states and for each action.
This type of state abstraction has been previously introduced by~\cite{li2006abstraction} as a $Q^*$-irrelevance state abstraction and is predictive of Q-values because each latent state is associated with a different set of Q-values.

Figure~\ref{fig:combination-lock-ep-len} plots the episode length of four different Q-learning agent configurations.
The first configuration (blue curves in Figure~\ref{fig:combination-lock-ep-len}) forms a baseline and simulates the Q-learning algorithm on both test tasks, without using any state abstraction.
The second configuration, called ``Q-learning with ignore dial abstraction'' (black curves in Figure~\ref{fig:combination-lock-ep-len}) simulates the Q-learning algorithm with a manually coded state abstraction that ignores the right dial in Test Lock 1 and the middle dial in Test Lock 2. 
Because this state abstraction compresses the 125 task states into 25 latent states by removing the digit from the state triplet not relevant for maximizing reward, this variation converges significantly faster than the baseline algorithm.
The third variation, called ``Q-learning with value-predictive abstraction'' (orange curves in Figure~\ref{fig:combination-lock-ep-len}) simulates Q-learning with the value-predictive state abstraction.
In both Test Lock 1 and Test Lock 2, this variation converges more slowly than using Q-learning without any state abstraction.
As discussed in Section~\ref{sec:transfer}, a value-predictive state abstraction is constructed using the Q-values of the policy that is optimal in the Training Lock. 
Because each combination lock MDP has a different optimal policy, a value-predictive state abstraction cannot be transferred across any of the two tasks.
Consequently, the ``Q-learning with a value-predictive abstraction'' agent does not converge more quickly than the baseline agent.
In these simulations, the Q-learning algorithm is not capable of finding an optimal policy, as outlined previously in Figure~\ref{fig:value-pred-counter-example}.
The green curves in Figure~\ref{fig:combination-lock-ep-len} plot the average episode length when Q-learning is combined with a reward-predictive state abstraction.
In Lock Task 1, this agent converges almost as fast as the agent using the manually coded state abstraction (black curve, left panel).
Only the position of the left and middle dials are relevant for predicting reward sequences $r_1,...,r_t$ that are generated by following an arbitrary action sequence $a_1,...,a_t$ from an arbitrary start state $s$ in the Training Lock MDP and in the Test Lock 1 MDP.
Consequently, a reward-predictive state abstraction can compress the 125 task states into 25 different latent states by ignoring the right dial, resulting in a significant performance improvement in Test Lock 1.
Note that LSFMs only approximate reward-predictive state representations resulting in slightly slower convergence in comparison to the black curve in the left panel of Figure~\ref{fig:combination-lock-ep-len}.

This behaviour changes in Test Lock 2, because in Test Lock 2 a different dial moves at random, changing how the state space should be compressed.
Because the right dial is relevant for predicting expected reward sequences in Test Lock 2, the reward-predictive state representation learned in the Training Lock task is no longer reward-predictive in Test Lock 2 and cannot be re-used without modifications.
Consequently, the ``Q-learning with reward-predictive abstraction'' agent exhibits worse performance (Figure~\ref{fig:combination-lock-ep-len}, right panel, green curve) than using Q-learning without any state abstraction (Figure~\ref{fig:combination-lock-ep-len}, right panel, blue curve).

The results presented in Figure~\ref{fig:combintation-lock} demonstrate that reward-predictive state representations encode which state information is relevant for predicting reward sequences and maximizing reward in a task.
Because reward-predictive state representations only model this aspect of an MDP, this type of state representation generalizes across tasks that preserve these state equivalences but differ in their transitions and rewards otherwise.
In contrast, value-predictive state representations ``overfit'' to a specific MDP and the MDP's optimal policy, resulting in negative transfer and possibly prohibiting an agent from finding an optimal policy at transfer.

\section{Discussion}\label{sec:discussion}

\begin{figure}
\centering

\scalebox{0.9}{
\begin{tikzpicture}

\definecolor{c0}{RGB}{31,119,180}
\definecolor{c1}{RGB}{255,127,14}
\definecolor{c2}{RGB}{44,160,44}

\node[anchor=north,align=left] at (-6, 3.6) {\small LAM\\ \small \citep{sutton2008lineardyna}};
\node at (-6, 0) {

\begin{tikzpicture}
\definecolor{c0}{RGB}{31,119,180}
\definecolor{c1}{RGB}{255,127,14}
\definecolor{c2}{RGB}{44,160,44}

\node(sa) at (0, 2) {\small$s,a$};
\node(phi) at (0, .5) {\small$\pmb{\phi}_s$};
\node(phin) at (-1, -1) {\small$\mathbb{E}_p [\pmb{\phi}' | s,a ]$};
\node(r) at (1, -1) {\small$r(s,a)$};
\node(rollout) at (0,-2.5) {\small$\mathbb{E} \left[ r_1,r_2,... \middle| s,a_1,a_2,... \right]$};

\draw[thick,-latex]                (sa)  -- node[pos=.5,left] {\small $\phi$} (phi);
\draw[thick,-latex,color=c2] (phi) -- node[pos=.5,left] {\small \color{c0}$\pmb{M}_a$} (phin);
\draw[thick,-latex,color=c2] (phi) -- node[pos=.5,right] {\small \color{c0}$\pmb{w}_a$} (r);
\draw[thick,-latex,color=c2] (phin) -- node[pos=.5,right] {} (rollout);
\draw[thick,-latex,color=c2] (r)    -- node[pos=.5,right] {} (rollout);

\end{tikzpicture}

};

\node[anchor=north,align=left] at (-2, 3.6) {\small LSFM };
\node at (-2, 0) {

\begin{tikzpicture}
\definecolor{c0}{RGB}{31,119,180}
\definecolor{c1}{RGB}{255,127,14}
\definecolor{c2}{RGB}{44,160,44}

\node(sa) at (0, 2) {\small$s,a$};
\node(phi) at (0, .5) {\small$\pmb{\phi}_s$};
\node(psi) at (-1, -1) {\small$\pmb{\psi}^\pi(s,a)$};
\node(r) at (1, -1) {\small$r(s,a)$};
\node(rollout) at (0,-2.5) {\small$\mathbb{E} \left[ r_1,r_2,... \middle| s,a_1,a_2,... \right]$};

\draw[thick,-latex]                (sa)  -- node[pos=.5,left] {\small $\phi$} (phi);
\draw[thick,-latex,color=c2] (phi) -- node[pos=.5,left] {\small \color{c0}$\pmb{F}_a$} (psi);
\draw[thick,-latex,color=c2] (phi) -- node[pos=.5,right] {\small \color{c0}$\pmb{w}_a$} (r);
\draw[thick,-latex,color=c2] (psi) -- node[pos=.5,right] {} (rollout);
\draw[thick,-latex,color=c2] (r)    -- node[pos=.5,right] {} (rollout);

\end{tikzpicture}

};

\node[anchor=north,align=left] at (2, 3.6) {\small SF\\ \small \citep{baretto2017sf}};
\node at (2, 0) {

\begin{tikzpicture}
\definecolor{c0}{RGB}{31,119,180}
\definecolor{c1}{RGB}{255,127,14}
\definecolor{c2}{RGB}{44,160,44}

\node(sa) at (0, 2) {\small$s,a$};
\node(phi) at (1.5, .5) {\small$\pmb{\xi}_{s,a}$};
\node(psi) at (0, -1) {\small$\pmb{\psi}^\pi(s,a)$};
\node(r) at (1.5, -1) {\small$r(s,a)$};
\node(rollout) at (0,-2.5) {\small$Q^\pi$};

\draw[thick,-latex]                (sa)  -- node[pos=.3,right] {~\small $\xi$} (phi);
\draw[thick,-latex] (sa) -- node[pos=.5,left] {\small $\psi$} (psi);
\draw[thick,-latex,color=c2] (phi) -- node[pos=.5,right] {\small \color{c0}$\pmb{w}$} (r);
\draw[thick,-latex,color=c2] (psi) -- node[pos=.5,right] {\small \color{c0}$\pmb{w}$} (rollout);

\end{tikzpicture}

};

\node[anchor=north,align=left] at (6, 3.6) {\small Fitted Q-Iteration\\ \small \citep{riedmiller2005neural}};
\node at (6, 0) {

\begin{tikzpicture}
\definecolor{c0}{RGB}{31,119,180}
\definecolor{c1}{RGB}{255,127,14}
\definecolor{c2}{RGB}{44,160,44}

\node(sa) at (0, 2) {\small$s,a$};
\node(phi) at (0, .5) {\small$\pmb{\phi}_s$};
\node(q) at (0,-2.5) {\small$Q^\pi$};

\draw[thick,-latex]                (sa)  -- node[pos=.5,right] {\small $\phi$} (phi);
\draw[thick,-latex,color=c2] (phi) -- node[pos=.5,right] {\small \color{c0}$\pmb{q}_a$} (q);
\end{tikzpicture}

};

\draw [decorate,decoration={brace,amplitude=10pt},xshift=0pt,yshift=0pt]
(-.2,-2.5) -- (-8,-2.5) node [black,midway,yshift=-20] {\small Reward-Predictive Models (model-based RL)};

\draw [decorate,decoration={brace,amplitude=10pt},xshift=0pt,yshift=0pt]
(8,-2.5) -- (.2,-2.5) node [black,midway,yshift=-20] {\small Value-Predictive Models (model-free RL)};

\draw [decorate,decoration={brace,amplitude=10pt},xshift=0pt,yshift=0pt]
(8,-3.5) -- (-4,-3.5) node [black,midway,yshift=-20] {\small Temporal Difference Learning};

\node                        (a) at (-8,-4.6) {};
\node[anchor=west] (b) at (-7, -4.6) {\small : Representation Map (arbitrary function)};
\draw[thick,-latex] (a) -- (b);
\node                        (a) at (-8,-5.1) {};
\node[anchor=west] (b) at (-7, -5.1) {\small : Prediction (linear if annotated with a matrix or vector)};
\draw[thick,-latex,color=c2] (a) -- (b);

\end{tikzpicture}}

\caption{Comparison of Presented State-Representation Models.}
\label{fig:model-comparison}
\end{figure}
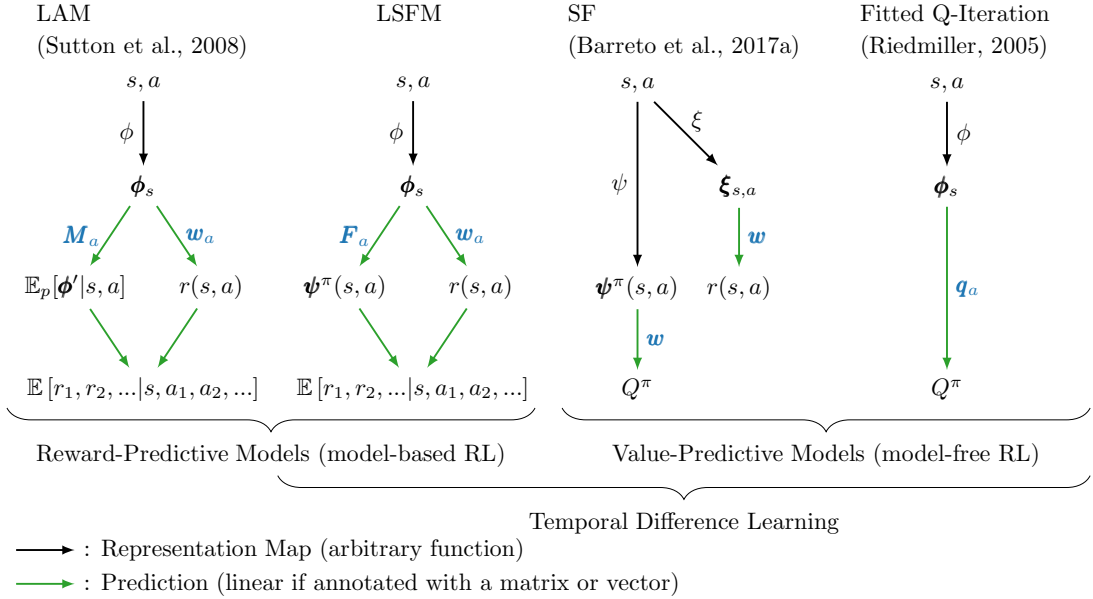

This article presents a study of how successor features combine aspects of model-free and model-based RL.
Connections are drawn by analyzing which properties different latent state spaces are predictive of.
The schematic in Figure~\ref{fig:model-comparison} illustrates the differences between the presented models.
By introducing LSFMs, SFs are tied to learning state representations that are predictive of future expected reward outcomes.
This model ties successor features to model-based reinforcement learning, because an agent that has learned an LSFM can predict expected future reward outcomes for any arbitrary action sequence.
While this connection to model-based RL has been previously hypothesized~\citep{russek2017predictive,momennejad2017successor}, LSFMs formalize this connection.
Because SFs obey a fixed-point equation similar to the Bellman fixed-point equation, SFs can also be linked to temporal-difference learning.
Similar to LAMs, LSFMs are a ``strict'' model-based architecture and are distinct from model-based and model-free hybrid architectures that iteratively search for an optimal policy and adjust their internal representation~\citep{oh2017value,weber2017imagination,franccois2019combined,gelada2019deepmdp}.
LSFMs only evaluate SFs for a fixed target policy that selects actions uniformly at random.
The learned model can then be used to predict the value function of any arbitrary policy, including the optimal policy.
In contrast to model-based and model-free hybrid architectures, the learned state representation does not have to be adopted to predict an optimal policy and generalizes across all latent policies. 
How to learn neural networks mapping inputs to latent feature spaces that are predictive of future reward outcomes is beyond the scope of this article and is left for future work.

Similar to reward-predictive state representations, the Predictron architecture~\citep{silver2017predictron} and the MuZero algorithm~\citep{schrittwieser2019muzero} use or construct a state representation to predict reward sequences.
In contrast to reward-predictive state representations, these other architectures predict reward sequences for $k$ time steps and then use the value function for one (or multiple) policies to predict the return obtained after $k$ time steps.
This distinction is key in learning reward-predictive state representations with LSFMs, which do not include a value prediction module, because if a state representation is designed to predict the value function of a policy, then the resulting state representation would be (to some extent) value predictive.
As outlined in Section~\ref{sec:val-pred}, this change would compromise the resulting state abstraction's ability to generalize across different transition and reward functions.

In contrast to the SF framework introduced by~\cite{baretto2017sf}, the connection between LSFMs and model-based RL is possible because the same state representation $\phi$ is used to predict its own SF (Figure~\ref{fig:model-comparison} center column).
While the deep learning models presented by~\cite{kulkarni2016deep} and~\cite{zhang2017deepsucc} also use one state representation to predict SFs and one-step rewards, these models are also constrained to predict image frames. 
LSFMs do not use the state representation to reconstruct actual states.
Instead, the state space is explicitly compressed, allowing the agent to generalize across distinct states.

Table~\ref{tab:summary} summarizes different properties of the presented state representations.
Bisimulation relations~\citep{givan2003bisimulation} preserve most structure of the original MDP and latent transition probabilities match with transition probabilities in the original MDP (Section~\ref{sec:bisim}).
Reward-predictive state representations do not preserve the transition probabilities of the original task (Figure~\ref{fig:example-reward-pred}) but construct a latent state space that is predictive of expected future reward outcomes.
These two representations generalize across all abstract policies, because they can predict reward outcomes for arbitrary action sequences.
Successor features are equivalent to value-predictive state representations, which construct a latent state space that is predictive of a policy's value function (Section~\ref{sec:val-pred}).
Because the value function can be factored into SFs and a reward model (Equation~\eqref{eq:sf-factorization}), SFs are robust to variations in reward functions.
Reward-predictive state representations remove previous limitations of SFs and generalize across variations in transition functions.
This property stands in contrast to previous work on SFs, which demostrate robustness against changes in the reward function only~\citep{baretto2017sf,baretto2018deepsf,kulkarni2016deep,zhang2017deepsucc,stachenfeld2017sr,momennejad2017successor,russek2017predictive}.
In all cases, including reward-predictive state representations, the learned models can only generalize to changes that approximately preserve the latent state space structure.
For example, in Figure~\ref{fig:transfer-experiment} positive transfer is possible because both tasks are grid worlds and only the locations of rewards and barriers is changed.
If the same representation is used on a completely different randomly generated finite MDP, then positive transfer may not be possible because both tasks do not have a latent state structure in common.

\begin{figure}
\small
\centering
\newcolumntype{C}[1]{>{\raggedright\arraybackslash}m{#1}}
\scalebox{0.9}{
\begin{tabular}{C{3.4cm} C{1.3cm} C{2.0cm} C{2.3cm} C{1.6cm} C{2cm}}
\hline
Model & Predicts Trained Policy & Generalizes to Variations in Rewards & Generalizes to Variations in Transitions & Generalizes Across All Policies & Predicts Transition Probabilities \\
\hline \hline
Bisimulation\hspace{1cm}\citep{givan2003bisimulation} & yes & yes & yes & yes & yes \\ \hline
Reward-Predictive (LSFM or LAM) & yes & yes & yes & yes & no \\ \hline
Successor Features \hspace{2cm} \citep{baretto2017sf} & yes & yes & no & no & no \\ \hline
Value-Predictive (Fitted Q-iteration) & yes & no & no & no & no \\ \hline
\end{tabular}}

\captionof{table}[]{Summary of Generalization Properties of Presented State Representations}
\label{tab:summary}
\end{figure}

In comparison to previous work on state abstractions~\citep{even2003nphardsa,li2006abstraction,abel2016icml,abel2018sa,abel2019rlit}, this article does not consider state representations that compress the state space as much as possible.
Instead, the degree of compression is set through the hyper-parameter that controls the dimension of the constructed latent state space.
The presented experiments demonstrate that these state representations compress the state space and implicitly convey information useful for transfer.
This formulation of state representations connects ideas from state abstractions to models that analyze linear basis functions~\citep{parr2008analysis,sutton1996generalization,konidaris2008fourier} or learn linear representations of the transition and reward functions~\citep{parrencoder2016}.
Recently, \cite{ruan2015representation} presented algorithms to cluster approximately bisimilar states.
Their method relies on bisimulation metrics~\citep{ferns2004bisimmetrics},
which use the Wasserstein metric to assess if two state transitions have the same distribution over next state clusters. 
In contrast to their approach, we phrase learning reward-predictive state representations as an energy-minimization problem and remove the requirement of computing the Wasserstein metric.

The presented experiments learn state representations by generating a transition data set covering all states of an MDP.
Complete coverage is obtained on the grid world tasks by generating a large enough transition data set.
Because this article focuses on drawing connections between different models of generalization across states, combining the presented algorithms with efficient exploration algorithms~\citep{jin2018isqefficient} or obtaining sample complexity or regret bounds similar to prior work~\citep{jaksch2010near,azar2017minimaxregret,osband2013psrl} is left to future studies.

\section{Conclusion}

This article presents an analysis of which latent representations an intelligent agent can construct to support different predictions, leading to new connections between model-based and model-free RL.
By introducing LSFMs, the presented analysis links learning successor features to model-based RL and demonstrates that the learned reward-predictive state representations are suitable for transfer across variations in transitions and rewards.
The presented results outline how different models of generalization are related to another and proposes a model for phrasing model-based learning as a representation-learning problem.
These results motivate the design and further investigation of new approximate model-based RL algorithms that learn state representations instead of one-step reward and transition models.

\acks{We would like to thank Prof. Michael J. Frank for many insightful discussions that benefited the development of the presented work.
This project was supported in part by the ONR MURI PERISCOPE project and in part by the NIH T32 Training Grant on Interactionist Cognitive Neuroscience.}

\newpage

\appendix

\section{Proofs of Theoretical Results}

This section lists formal proofs for all presented theorems and propositions.

\subsection{Bisimulation Theorems}\label{ap:bisim-pfs}

For an equivalence relation $\sim$ defined on a set $\mathcal{S}$, the set of all partition is denoted with $\mathcal{S} / \sim$.
Each partition $[s] \in \mathcal{S} / \sim$ is a subset of $\mathcal{S}$ and $s \in [s]$.

\begin{defn}[Bisimilarity~\citep{ferns2011bisimulation}]\label{def:bisimulation-full}
For an MDP $M = \langle \mathcal{S}, \mathcal{A}, p, r, \gamma \rangle$ where $\langle \mathcal{S}, \Sigma, p \rangle$ is a measurable space with $\sigma$-algebra $\Sigma$ and $p$ is a Markov kernel labelled for each action $a \in \mathcal{A}$.
Consider an equivalence relation $\sim_b$ on the state space $\mathcal{S}$ such that each state partition $[s']$ also lies in the $\sigma$-algebra and $\forall [s'] \in \mathcal{S} / \sim_b$, $[s'] \in \Sigma$.
The equivalence relation $\sim_b$ is a bisimulation if
\begin{align}
s \sim_b \tilde{s} \iff & \forall a \in \mathcal{A},~ \mathbb{E}_p \left[ r(s,a,s') \middle| s, a \right] = \mathbb{E}_p \left[ r(\tilde{s},a,s') \middle| \tilde{s}, a \right] \label{eq:bisim-1} \\
&\text{and}~\forall [s'] \in \mathcal{S} / \sim_b, ~ p([s'] | s, a) = p([s'] | \tilde{s},a). \label{eq:bisim-2}
\end{align}
\end{defn}

Using this definition, Theorem~\ref{thm:bisim-lam} can be proven.

\begin{proof}[Proof of Theorem~\ref{thm:bisim-lam}]
Consider any two states $s$ and $\tilde{s}$ such that $\pmb{\phi}_s = \pmb{\phi}_{\tilde{s}}$.
For both $s$ and $\tilde{s}$ we have
\begin{equation}
\mathbb{E}_p \left[ r(s,a,s') \middle| s, a \right] = \pmb{\phi}_s^\top \pmb{w}_a = \pmb{\phi}_{\tilde{s}}^\top \pmb{w}_a = \mathbb{E}_p \left[ r(\tilde{s},a,s') \middle| \tilde{s}, a \right], \label{eq:thm-1-pf-1}
\end{equation}
and the bisimulation reward condition in Equation~\eqref{eq:bisim-1} holds.
To show that also the bisimulation transition condition in Equation~\eqref{eq:bisim-2} holds, observe that
\begin{align}
\pmb{\phi}_s^\top &= \pmb{\phi}_{\tilde{s}}^\top &\iff \\
\pmb{\phi}_s^\top \pmb{M}_a &= \pmb{\phi}_{\tilde{s}}^\top \pmb{M}_a &\iff \\
\mathbb{E}_p \left[ \pmb{\phi}_{s'} \middle| s,a \right] &= \mathbb{E}_p \left[ \pmb{\phi}_{s'} \middle| \tilde{s},a \right] &\iff \\
\sum_{i=1}^n p(s, a, [s_i]) \pmb{e}_i &= \sum_{i=1}^n p(\tilde{s}, a, [s_i]) \pmb{e}_i, \label{eq:thm-1-pf-2}
\end{align}
where $[s_i] \subset \mathcal{S}$ are all states that are mapped to the one-hot feature vector $\pmb{e}_i$.
Each side of the identity~\eqref{eq:thm-1-pf-2} computes an expectation over one-hot bit vectors and thus the $i$th entry of $\sum_{i=1}^n p(s, a, [s_i]) \pmb{e}_i$ contains the probability value $p(s, a, [s_i])$.
Hence both $s$ and $\tilde{s}$ have equal probabilities of transitioning into each state partition that is associated with $\pmb{e}_i$.
Define an equivalence relation $\sim_\phi$ such that
\begin{equation}
\forall s, \tilde{s} \in \mathcal{S}, ~ \pmb{\phi}_s = \pmb{\phi}_{\tilde{s}} \iff s \sim_\phi \tilde{s}.
\end{equation}
Because all feature vectors $\pmb{\phi}_s$ are one-hot bit vectors, there are at most $n$ partitions and the set of all state partitions has size $| \mathcal{S} / \sim_\phi | = n$.
Combining these observations, Equation~\eqref{eq:thm-1-pf-2} can be rewritten as
\begin{equation}
\forall [s'] \in \mathcal{S} / \sim_\phi,~ p([s'] | s,a) = p([s'] | \tilde{s},a). \label{eq:thm-1-pf-3}
\end{equation}
By lines~\eqref{eq:thm-1-pf-1} and~\eqref{eq:thm-1-pf-3}, the equivalence relation $\sim_\phi$ is a bisimulation relation and if $\pmb{\phi}_s = \pmb{\phi}_{\tilde{s}}$ then both $s$ and $\tilde{s}$ are bisimilar.
\end{proof}

\begin{lem}\label{lem:pi-cond-mat}
Assume an MDP, state representation $\phi : \mathcal{S} \to \{ \pmb{e}_1,...,\pmb{e}_n \}$, LSFM $\{ \pmb{F}_a, \pmb{w}_a \}_{a \in \mathcal{A}}$, and arbitrary policy $\pi \in \Pi_\phi$. 
Let $\pmb{F}^\pi$ be an $n \times n$ real valued matrix with each row $\pmb{F}^\pi(i) = \mathbb{E}_\pi \left[ \pmb{e}_i^\top \pmb{F}_a \middle| s   \right]$, then
\begin{equation}
\mathbb{E}_\pi \left[ \pmb{\phi}_s^\top \pmb{F}_a \middle| s \right] = \pmb{e}_i^\top \pmb{F}^\pi, \label{eq:lem-pi-cond-mat-1}
\end{equation}
where $\pmb{\phi}_s = \pmb{e}_i$ for some $i$.
For a LAM $\{ \pmb{M}_a, \pmb{w}_a \}_{a \in \mathcal{A}}$, let $\pmb{M}^\pi$ be a $n \times n$ real-valued matrix with each row $\pmb{M}^\pi(i) = \mathbb{E}_\pi \left[ \pmb{e}_i^\top \pmb{M}_a \middle| s   \right]$, then
\begin{equation}
\mathbb{E}_\pi \left[ \pmb{\phi}_s^\top \pmb{M}_a \middle| s \right] = \pmb{e}_i^\top \pmb{M}^\pi. \label{eq:lem-pi-cond-mat-2}
\end{equation}
\end{lem}

\begin{proof}[Proof of Lemma~\ref{lem:pi-cond-mat}]
The first identities in~\eqref{eq:lem-pi-cond-mat-1} and~\eqref{eq:lem-pi-cond-mat-2} hold because $\pmb{\phi}_s = \pmb{e}_i$ for some $i$.
Then,
\begin{equation*}
\mathbb{E}_\pi \left[ \pmb{\phi}_s^\top \pmb{F}_a \middle| s \right] = \sum_a \pi(s,a) \pmb{\phi}_s \pmb{F}_a =  \sum_a \pi(s,a) \pmb{e}_i^\top \pmb{F}_a = \mathbb{E}_\pi \left[ \pmb{e}_i^\top \pmb{F}_a \middle| s   \right] = \pmb{F}^\pi(i) = \pmb{e}_i^\top \pmb{F}^\pi
\end{equation*}
and
\begin{equation*}
\mathbb{E}_\pi \left[ \pmb{\phi}_s^\top \pmb{M}_a \middle| s \right] = \sum_a \pi(s,a) \pmb{\phi}_s \pmb{M}_a =  \sum_a \pi(s,a) \pmb{e}_i^\top \pmb{M}_a = \mathbb{E}_\pi \left[ \pmb{e}_i^\top \pmb{M}_a \middle| s   \right] = \pmb{M}^\pi(i) = \pmb{e}_i^\top \pmb{M}^\pi.
\end{equation*}
\end{proof}

\begin{defn}[Weighting Function]\label{def:omega}
For an MDP $M=\langle \mathcal{S}, \mathcal{A}, p, r, \gamma \rangle$, let $\sim$ be an equivalence relation on the state space $\mathcal{S}$, creating a set of state partitions $\mathcal{S} / \sim$.
Assume that each state partition $[s]$ is a measurable space $\langle [s], \Sigma_{[s]}, \omega_{[s]} \rangle$, where $\omega_{[s]}$ is a probability measure indexed by each partition $[s]$ with $\sigma$-algebra $ \Sigma_{[s]}$.
The function $\omega$ is called the weighting function.
\end{defn}

\begin{lem}\label{lem:f-mat-id}
Assume an MDP, state representation $\phi : \mathcal{S} \to \{ \pmb{e}_1,...,\pmb{e}_n \}$, LSFM $\{ \pmb{F}_a, \pmb{w}_a \}_{a \in \mathcal{A}}$, and arbitrary abstract policy $\pi \in \Pi_\phi$.
Then, 
\begin{equation}
\forall s, \forall a ~ \pmb{\phi}_s^\top \pmb{F}_a = \pmb{\phi}_s^\top + \gamma \mathbb{E}_{p, \pi} \left[ \pmb{\phi}_{s'} \pmb{F}_{a'} \middle| s, a \right] \implies \forall s , \forall a , \exists \pmb{M}_a ~\text{such that}~ \pmb{F}_a = \pmb{I} + \gamma \pmb{M}_a \pmb{F}^\pi, \label{eq:lem-f-mat-id}
\end{equation}
where the matrix $\pmb{F}^\pi$ is constructed as described in Lemma~\ref{lem:pi-cond-mat}.
\end{lem}

\begin{proof}[Proof of Lemma~\ref{lem:f-mat-id}]
Consider an equivalence relation $\sim_\phi$ that is constructed using the state representation $\phi$ and 
\begin{equation}
\forall s, \tilde{s} \in \mathcal{S}, ~ \pmb{\phi}_s = \pmb{\phi}_{\tilde{s}} \iff s \sim_\phi \tilde{s}.
\end{equation}
The weighting function $\omega$ models a probability distribution of density function of visiting a state $s$ that belongs to a state partition $[s] \in \mathcal{S} / \sim_\phi$.
Because all states $s \in [s]$ are mapped to the same feature vector $\pmb{\phi}_s$, we have that
\begin{equation}
\mathbb{E}_{\omega_{[s]}} \left[ \pmb{\phi}_s \right] = \pmb{\phi}_s.  \label{eq:lem-f-mat-id-1}
\end{equation}
The stochastic matrix $\pmb{M}_a$ is defined for every action $a$ as
\begin{equation}
\pmb{M}_a(i, j) =\mathbb{E}_{\omega_{[s]}} \left[ \text{Pr} \left\{ s \overset{a}{\to} \pmb{e}_j \right\} \right] ~\text{with}~ \pmb{\phi}_s = \pmb{e}_i, \label{eq:lem-f-mat-id-2}
\end{equation}
where $\text{Pr} \left\{ s \overset{a}{\to} \pmb{e}_j \right\}$ is the probability of transitioning into the state partition associated with the latent state $\pmb{e}_j$ and 
\begin{equation}
\text{Pr} \left\{ s \overset{a}{\to} \pmb{e}_j \right\} = p(s,a,[s_i]) ~\text{such that}~ \forall s \in [s_i], \phi(s) = \pmb{e}_j.
\end{equation}
The identity in Equation~\eqref{eq:lem-f-mat-id} can be re-written as follows:
\begin{align}
\pmb{\phi}_s^\top \pmb{F}_a &= \pmb{\phi}_s^\top + \gamma \mathbb{E}_{p, \pi} \left[ \pmb{\phi}_{s'} \pmb{F}_{a'} \middle| s, a \right] & \iff \nonumber \\
\pmb{\phi}_s^\top \pmb{F}_a &= \pmb{\phi}_s^\top + \gamma \mathbb{E}_{p} \left[ \pmb{\phi}_{s'} \middle| s, a \right] \pmb{F}^\pi & \iff & (\text{by~Lemma~\ref{lem:pi-cond-mat}}) \nonumber \\
\mathbb{E}_{\omega_{[s]}} \left[ \pmb{\phi}_s^\top \pmb{F}_a \right] &= \mathbb{E}_{\omega_{[s]}} \left[ \left( \pmb{\phi}_s^\top + \gamma \mathbb{E}_{p} \left[ \pmb{\phi}_{s'} \middle| s, a \right] \right) \pmb{F}^\pi \right] & \iff  \nonumber \\
\pmb{\phi}_s^\top \pmb{F}_a &= \pmb{\phi}_s^\top + \gamma \mathbb{E}_{\omega_{[s]}} \left[ \mathbb{E}_{p} \left[ \pmb{\phi}_{s'} \middle| s, a \right] \right] \pmb{F}^\pi & \iff & (\text{by~\eqref{eq:lem-f-mat-id-1}}) \nonumber  \\
\pmb{\phi}_s^\top \pmb{F}_a &= \pmb{\phi}_s^\top + \gamma \mathbb{E}_{\omega_{[s]}} \left[ \sum_{j=1}^n \text{Pr} \left\{ s \overset{a}{\to} \pmb{e}_j \right\} \pmb{e}_j^\top  \right] \pmb{F}^\pi & \iff \nonumber \\
\pmb{\phi}_s^\top \pmb{F}_a &= \pmb{\phi}_s^\top + \gamma \underbrace{\left[ \mathbb{E}_{\omega_{[s]}} \left[ \text{Pr} \left\{ s \overset{a}{\to} \pmb{e}_1 \right\} \right], ..., \mathbb{E}_{\omega_{[s]}} \left[ \text{Pr} \left\{ s \overset{a}{\to} \pmb{e}_n \right\}  \right] \right]}_{\text{$n$-dimensional row vector, because $\pmb{e}_j$ is one-hot}} \pmb{F}^\pi & \iff \nonumber \\
\pmb{\phi}_s^\top \pmb{F}_a &= \pmb{\phi}_s^\top + \gamma \left[ \pmb{M}_a(i,1),...,\pmb{M}_a(i,n) \right] \pmb{F}^\pi & \iff & (\text{by~\eqref{eq:lem-f-mat-id-2}}) \nonumber \\
\pmb{\phi}_s^\top \pmb{F}_a &= \pmb{\phi}_s^\top + \gamma \pmb{\phi}_s^\top \pmb{M}_a \pmb{F}^\pi & \iff & (\text{by $\pmb{\phi}_s = \pmb{e}_i$}) \nonumber \\
\pmb{\phi}_s^\top \pmb{F}_a &= \pmb{\phi}_s^\top \left( \pmb{I} + \gamma \pmb{M}_a \pmb{F}^\pi \right)&   \label{eq:lem-f-mat-id-5}
\end{align}
Equation~\eqref{eq:lem-f-mat-id-5} holds for any arbitrary state $s$ and because each state is mapped to a one of the one-hot vectors $\{ \pmb{e}_1,...,\pmb{e}_n \}$,
\begin{align*}
\forall i \in \{ 1,...,n \}, ~ \pmb{e}_i^\top \pmb{F}_a &= \pmb{e}_i^\top \left( \pmb{I} + \gamma \pmb{M}_a \pmb{F}^\pi \right) & \iff \\
\pmb{F}_a &= \pmb{I} + \gamma \pmb{M}_a \pmb{F}^\pi .
\end{align*}
\end{proof}

\begin{lem}\label{lem:f-mat-inv}
Consider a state representation $\phi : \mathcal{S} \to \{ \pmb{e}_1,...,\pmb{e}_n \}$, an arbitrary abstract policy $\pi \in \Pi_\phi$, an LSFM $\{ \pmb{F}_a, \pmb{w}_a \}_{a \in \mathcal{A}}$ and LAM $\{ \pmb{M}_a, \pmb{w}_a \}_{a \in \mathcal{A}}$ where each transition matrix $\pmb{M}_a$ is stochastic. 
Then,
\begin{equation}
\pmb{F}_a = \pmb{I} + \gamma \pmb{M}_a \pmb{F}^\pi \implies \exists \left( \pmb{F}^\pi \right)^{-1} ~\text{and}~ \left( \pmb{F}^\pi \right)^{-1} = \pmb{I} - \gamma \pmb{M}^\pi,
\end{equation}
where the matrix $\pmb{M}^\pi$ is constructed as described in Lemma~\ref{lem:pi-cond-mat}.
\end{lem}

\begin{proof}[Proof of Lemma~\eqref{lem:f-mat-inv}]
By Lemma~\ref{lem:pi-cond-mat}, one can write for any arbitrary $i$
\begin{align}
\pmb{e}_i^\top \pmb{F}^\pi &= \mathbb{E}_\pi \left[ \pmb{e}_i^\top \pmb{F}_a \right] & \iff \nonumber \\
\pmb{e}_i^\top \pmb{F}^\pi &= \mathbb{E}_\pi \left[ \pmb{e}_i^\top (\pmb{I} + \gamma \pmb{M}_a \pmb{F}^\pi) \right] & \iff &(\text{by Lemma~\ref{lem:f-mat-id}}) \nonumber \\
\pmb{e}_i^\top \pmb{F}^\pi &= \pmb{e}_i \pmb{I} + \gamma \mathbb{E}_\pi \left[ \pmb{e}_i^\top \pmb{M}_a \right] \pmb{F}^\pi & \iff \nonumber \\
\pmb{e}_i^\top \pmb{F}^\pi &= \pmb{e}_i \pmb{I} + \gamma \pmb{e}_i^\top \pmb{M}^\pi \pmb{F}^\pi & \iff &(\text{by Lemma~\ref{lem:pi-cond-mat}}) \nonumber \\
\pmb{e}_i^\top \pmb{F}^\pi &= \pmb{e}_i \left( \pmb{I} + \gamma \pmb{M}^\pi \pmb{F}^\pi \right) \label{eq:lem-f-mat-inv-1}
\end{align}
Because Equation~\eqref{eq:lem-f-mat-inv-1} holds for every $i$, 
\begin{align*}
\pmb{F}^\pi &= \pmb{I} + \gamma \pmb{M}^\pi \pmb{F}^\pi &\iff \\
\pmb{F}^\pi - \gamma \pmb{M}^\pi \pmb{F}^\pi  &= \pmb{I} &\iff \\
( \pmb{I} - \gamma \pmb{M}^\pi ) \pmb{F}^\pi  &= \pmb{I}.
\end{align*}
Because $\pmb{M}^\pi$ is stochastic, it has a spectral radius of at most one and all its eigenvalues $\lambda_j \le 1$.
Thus the matrix $\pmb{I} - \gamma \pmb{M}^\pi$ is invertible because
\begin{align}
\det (\pmb{I} - \gamma \pmb{M}^\pi) \ge \det(\pmb{I}) + (-\gamma)^n \det(\pmb{M}^\pi) \ge 1 - \gamma^n \det(\pmb{M}^\pi) = 1 - \gamma^n \underbrace{\prod_j \lambda_j}_{\le 1} > 0.
\end{align}
Hence $\pmb{F}^\pi = \left( \pmb{I} - \gamma \pmb{M}^\pi \right)^{-1} \iff \left( \pmb{F}^\pi \right)^{-1} = \pmb{I} - \gamma \pmb{M}^\pi$.
\end{proof}

Using these lemmas, Theorem~\ref{thm:bisim-lsfm} can be proven.

\begin{proof}[Proof of Theorem~\ref{thm:bisim-lsfm}]
The proof is by reducing Equation~\eqref{eq:bisim-lsfm} to Equation~\eqref{eq:bisim-lam} using the previously established lemmas and then applying Theorem~\ref{thm:bisim-lam}.
Equation~\eqref{eq:bisim-lsfm} can be re-written as follows:
\begin{align}
\pmb{\phi}_s^\top \pmb{F}_a &= \pmb{\phi}_s^\top + \gamma \mathbb{E}_{p, \pi} \left[ \pmb{\phi}_{s'} \pmb{F}_{a'} \middle| s, a \right] &\Leftrightarrow \nonumber \\
\pmb{\phi}_s^\top \pmb{F}_a &= \pmb{\phi}_s^\top + \gamma \mathbb{E}_{p} \left[ \pmb{\phi}_{s'} \middle| s, a \right] \pmb{F}^\pi &\Leftrightarrow & (\text{Lem.~\ref{lem:pi-cond-mat}}) \nonumber \\
\pmb{\phi}_s^\top \pmb{F}_a (\pmb{I} - \gamma \pmb{M}^\pi) &= \pmb{\phi}_s^\top (\pmb{I} - \gamma \pmb{M}^\pi) + \gamma \mathbb{E}_{p} \left[ \pmb{\phi}_{s'} \middle| s, a \right] \pmb{F}^\pi (\pmb{I} - \gamma \pmb{M}^\pi) &\Leftrightarrow & (\text{Lem.~\ref{lem:f-mat-inv}})  \nonumber \\
\pmb{\phi}_s^\top \pmb{F}_a (\pmb{I} - \gamma \pmb{M}^\pi) &= \pmb{\phi}_s^\top (\pmb{I} - \gamma \pmb{M}^\pi) + \gamma \mathbb{E}_{p} \left[ \pmb{\phi}_{s'} \middle| s, a \right] \pmb{F}^\pi \left(\pmb{F}^\pi\right)^{-1} &\Leftrightarrow & (\text{Lem.~\ref{lem:f-mat-inv}}) \nonumber \\
\pmb{\phi}_s^\top \pmb{F}_a (\pmb{I} - \gamma \pmb{M}^\pi) &= \pmb{\phi}_s^\top (\pmb{I} - \gamma \pmb{M}^\pi) + \gamma \mathbb{E}_{p} \left[ \pmb{\phi}_{s'} \middle| s, a \right] &\Leftrightarrow \nonumber \\
\pmb{\phi}_s^\top (\pmb{I} + \gamma \pmb{M}_a \pmb{F}^\pi ) (\pmb{I} - \gamma \pmb{M}^\pi) &= \pmb{\phi}_s^\top (\pmb{I} - \gamma \pmb{M}^\pi) + \gamma \mathbb{E}_{p} \left[ \pmb{\phi}_{s'} \middle| s, a \right] &\Leftrightarrow &(\text{Lem.~\ref{lem:f-mat-id}}) \nonumber \\
\pmb{\phi}_s^\top \gamma \pmb{M}_a \pmb{F}^\pi (\pmb{I} - \gamma \pmb{M}^\pi) &=  \gamma \mathbb{E}_{p} \left[ \pmb{\phi}_{s'} \middle| s, a \right] &\Leftrightarrow \nonumber \\
\pmb{\phi}_s^\top \gamma \pmb{M}_a &=  \gamma \mathbb{E}_{p} \left[ \pmb{\phi}_{s'} \middle| s, a \right] &\Leftrightarrow & (\text{Lem.~\ref{lem:f-mat-inv}}) \nonumber \\
\pmb{\phi}_s^\top \pmb{M}_a &=  \mathbb{E}_{p} \left[ \pmb{\phi}_{s'} \middle| s, a \right] \label{eq:thm-bisim-lsfm-1}
\end{align}
Using the reward condition stated in Equation~\eqref{eq:bisim-lsfm} and Equation~\eqref{eq:thm-bisim-lsfm-1} Theorem~\ref{thm:bisim-lam} can be applied to conclude the proof.

To prove the last claim of Theorem~\ref{thm:bisim-lsfm}, assume that 
\begin{equation}
\pmb{\phi}_s^\top \pmb{F}_a = \pmb{\phi}_s^\top + \gamma \mathbb{E}_{p, \pi} \left[ \pmb{\phi}_{s'} \pmb{F}_{a'} \middle| s, a \right] \label{eq:thm-bisim-lsfm-2}
\end{equation}
for some policy $\pi \in \Pi_\phi$.
Consider an arbitrary distinct policy $\tilde{\pi} \in \Pi_\phi$, then the fix-point Eq.~\eqref{eq:thm-bisim-lsfm-2} can be re-stated in terms of then policy $\tilde{\pi}$:
\begin{align}
\pmb{\phi}_s^\top \pmb{F}_a &= \pmb{\phi}_s^\top + \gamma \mathbb{E}_{p, \pi} \left[ \pmb{\phi}_{s'} \pmb{F}_{a'} \middle| s, a \right] &\Leftrightarrow \label{eq:thm-bisim-lsfm-3-start}  \\
\pmb{\phi}_s^\top \pmb{M}_a &=  \mathbb{E}_{p} \left[ \pmb{\phi}_{s'} \middle| s, a \right] &\Leftrightarrow & (\text{Eq.~\eqref{eq:thm-bisim-lsfm-1}}) \\
\gamma \pmb{\phi}_s^\top \pmb{M}_a \pmb{F}^{\tilde{\pi}} &= \gamma \mathbb{E}_{p} \left[ \pmb{\phi}_{s'} \middle| s, a \right] \pmb{F}^{\tilde{\pi}} &\Leftrightarrow &(\text{multiply with $\pmb{F}^{\tilde{\pi}}$ and $\gamma$}) \\
\pmb{\phi}_s^\top + \gamma \pmb{\phi}_s^\top \pmb{M}_a \pmb{F}^{\tilde{\pi}} &=  \pmb{\phi}_s^\top + \gamma \mathbb{E}_{p} \left[ \pmb{\phi}_{s'} \middle| s, a \right] \pmb{F}^{\tilde{\pi}} &\Leftrightarrow &(\text{add $\pmb{\phi}_s$}) \\
\pmb{\phi}_s^\top ( \pmb{I} + \gamma \pmb{M}_a \pmb{F}^{\tilde{\pi}} ) &=  \pmb{\phi}_s^\top + \gamma \mathbb{E}_{p} \left[ \pmb{\phi}_{s'} \middle| s, a \right] \pmb{F}^{\tilde{\pi}} &\Leftrightarrow \\
\pmb{\phi}_s^\top \pmb{F}_a &=  \pmb{\phi}_s^\top + \gamma \mathbb{E}_{p} \left[ \pmb{\phi}_{s'} \middle| s, a \right] \pmb{F}^{\tilde{\pi}} &\Leftrightarrow  &(\text{Lem.~\ref{lem:f-mat-id}}) \\
\pmb{\phi}_s^\top \pmb{F}_a &=  \pmb{\phi}_s^\top + \gamma \mathbb{E}_{p, \tilde{\pi}} \left[ \pmb{\phi}_{s'} \pmb{F}_{a} \middle| s, a \right] &\Leftrightarrow & (\text{Lem.~\ref{lem:pi-cond-mat}}) \label{eq:thm-bisim-lsfm-3}
\end{align}
This argument shows that if Eq.~\eqref{eq:thm-bisim-lsfm-2} holds for a policy $\pi$, then Eq.~\eqref{eq:thm-bisim-lsfm-3} also holds for other arbitrary policy $\tilde{\pi}$.
\end{proof}

\subsection{Approximate Reward-Predictive State Representations}\label{app:approx-pfs}

This section presents formal proofs for Theorem~\ref{thm:rollout-bound-lam} and~\ref{thm:approx-val-fn}.
While the following proofs assume that the matrix $\overline{\pmb{F}}$ is defined as stated in Equation~\eqref{eq:f-bar-def}, these proofs could be generalized to different definitions of $\overline{\pmb{F}}$, assuming that the matrix $\overline{\pmb{F}}$ is not a function of the state $s$ and only depends on the matrices $\{ \pmb{F}_a \}_{a \in \mathcal{A}}$.

\begin{lem}\label{lem:t-model-match}
For an MDP, a state representation $\phi$, a LSFM $\{ \pmb{F}_a, \pmb{w}_a \}_{a \in \mathcal{A}}$ and a LAM $\{ \pmb{M}_a, \pmb{w}_a \}_{a \in \mathcal{A}}$ where $\Delta = 0$, 
\begin{equation}
\varepsilon_p \le \varepsilon_\psi \frac{1 + \gamma M}{\gamma}. \label{eq:lem-t-model-match}
\end{equation}
\end{lem}
\begin{proof}[Proof of Lemma~\ref{lem:t-model-match}]
The proof is by manipulating the definition of $\varepsilon_\psi$ and using the fact that $\Delta = 0$ and $\pmb{F}_a = \pmb{I} + \gamma \pmb{M}_a \overline{\pmb{F}}$.
Let $\overline{M} = \frac{1}{| \mathcal{A} |} \sum_{a \in \mathcal{A}} \pmb{M}_a$, then
\begin{equation}
\overline{\pmb{F}} = \pmb{I} + \gamma \overline{\pmb{M}} \overline{\pmb{F}} \iff \pmb{I} = (\pmb{I} - \gamma \overline{\pmb{M}} ) \overline{\pmb{F}} \label{eq:lem-t-model-match-1}
\end{equation}
Hence the square matrix $\pmb{I} - \gamma \overline{\pmb{M}}$ is a left inverse of the square matrix $\overline{\pmb{F}}$. 
By associativity of matrix multiplication, $\pmb{I} + \gamma \overline{\pmb{M}}$ is also a right inverse of $\overline{\pmb{F}}$ and $\overline{\pmb{F}} (\pmb{I} - \gamma \overline{\pmb{M}})$.\footnote{If $\overline{\pmb{F}}^{-1} \overline{\pmb{F}} = \pmb{I}$ then $\overline{\pmb{F}} = \overline{\pmb{F}} \pmb{I} = \overline{\pmb{F}} ( \overline{\pmb{F}}^{-1} \overline{\pmb{F}} ) = ( \overline{\pmb{F}} \overline{\pmb{F}}^{-1} ) \overline{\pmb{F}}$. If $\overline{\pmb{F}}^{-1} \overline{\pmb{F}} \ne \pmb{I}$ were true, then it would contradict $\overline{\pmb{F}} = ( \overline{\pmb{F}} \overline{\pmb{F}}^{-1} ) \overline{\pmb{F}}$. Hence $\overline{\pmb{F}} \overline{\pmb{F}}^{-1} = \pmb{I}$ and the right inverse exists.}
Consequently, the norm of $\overline{\pmb{F}}^{-1}$ can be bounded with
\begin{equation}
\left| \left| \overline{\pmb{F}}^{-1} \right| \right| = \left| \left| (\pmb{I} - \gamma \overline{\pmb{F}}) \right| \right| \le 1 + \gamma M. \label{eq:lem-t-model-match-2}
\end{equation}

\noindent For an arbitrary state and action pair $s,a$,
\begin{align}
\pmb{\delta}^\top_{s,a} &= \pmb{\phi}_s^\top + \gamma \mathbb{E}_p \left[ \pmb{\phi}_{s'}^\top \overline{\pmb{F}} \middle| s,a \right] - \pmb{\phi}_s^\top \pmb{F}_a \\
&= \pmb{\phi}_s^\top + \gamma \mathbb{E}_p \left[ \pmb{\phi}_{s'}^\top \overline{\pmb{F}} \middle| s,a \right] - \gamma \pmb{\phi}_s^\top \pmb{M}_a \overline{\pmb{F}} + \gamma \pmb{\phi}_s^\top \pmb{M}_a \overline{\pmb{F}} - \pmb{\phi}_s^\top \pmb{F}_a \\
&= \pmb{\phi}_s^\top + \gamma \left( \mathbb{E}_p \left[ \pmb{\phi}_{s'}^\top \middle| s,a \right] - \pmb{\phi}_s^\top \pmb{M}_a \right) \overline{\pmb{F}} + \gamma \pmb{\phi}_s^\top \pmb{M}_a \overline{\pmb{F}} - \pmb{\phi}_s^\top \pmb{F}_a \label{eq:lem-t-model-match-3a}
\end{align}
Let $\pmb{\varepsilon}^\top_{s,a} = \mathbb{E}_p \left[ \pmb{\phi}_{s'}^\top \middle| s,a \right] - \pmb{\phi}_s^\top \pmb{M}_a$.
Re-arranging the identity in~\eqref{eq:lem-t-model-match-3a} results in
\begin{align}
\gamma \pmb{\varepsilon}^\top_{s,a} \overline{\pmb{F}} &= \pmb{\delta}_{s,a}^\top - \pmb{\phi}_s^\top - \gamma \pmb{\phi}_s^\top \pmb{M}_a \overline{\pmb{F}} + \pmb{\phi}_s^\top \pmb{F}_a &\iff \nonumber \\
\gamma \pmb{\varepsilon}^\top_{s,a} &= \pmb{\delta}_{s,a}^\top \overline{\pmb{F}}^{-1} - \pmb{\phi}_s^\top  \overline{\pmb{F}}^{-1} - \gamma \pmb{\phi}_s^\top \pmb{M}_a + \pmb{\phi}_s^\top \pmb{F}_a \overline{\pmb{F}}^{-1} &\iff &(\text{by~\eqref{eq:lem-t-model-match-1}}) \nonumber \\
\gamma \pmb{\varepsilon}^\top_{s,a} &= \pmb{\delta}_{s,a}^\top \overline{\pmb{F}}^{-1} - \pmb{\phi}_s^\top  \overline{\pmb{F}}^{-1} - \gamma \pmb{\phi}_s^\top \pmb{M}_a + \pmb{\phi}_s^\top ( \pmb{I} + \gamma \pmb{M}_a \overline{\pmb{F}} ) \overline{\pmb{F}}^{-1} &\iff &(\text{by~$\Delta = 0$}) \nonumber \\
\gamma \pmb{\varepsilon}^\top_{s,a} &= \pmb{\delta}_{s,a}^\top \overline{\pmb{F}}^{-1} - \pmb{\phi}_s^\top  \overline{\pmb{F}}^{-1} - \gamma \pmb{\phi}_s^\top \pmb{M}_a + \pmb{\phi}_s^\top \overline{\pmb{F}}^{-1} + \gamma \pmb{\phi}_s^\top \pmb{M}_a &\iff \nonumber \\
\gamma \pmb{\varepsilon}^\top_{s,a} &= \pmb{\delta}_{s,a}^\top \overline{\pmb{F}}^{-1}  &\iff \nonumber \\
\gamma \left| \left|  \pmb{\varepsilon}^\top_{s,a}\right| \right| &\le \left| \left| \pmb{\delta}_{s,a}^\top \right| \right| \left| \left| \overline{\pmb{F}}^{-1} \right| \right|  &\iff \nonumber \\
\gamma \left| \left|  \pmb{\varepsilon}^\top_{s,a}\right| \right| &\le \varepsilon_\psi \left| \left| \overline{\pmb{F}}^{-1} \right| \right|  &\iff &(\text{by~\eqref{eq:eps-def-sf}}) \nonumber \\
\left| \left|  \pmb{\varepsilon}^\top_{s,a}\right| \right| &\le \varepsilon_\psi (1 + \gamma M) / \gamma  &\iff &(\text{by~\eqref{eq:lem-t-model-match-2}}) \label{eq:lem-t-model-match-4}
\end{align}
Note that the bound in Equation~\eqref{eq:lem-t-model-match-4} does not depend on the state and action pair $s,a$ and thus
\begin{equation}
\forall s,a,~ \left| \left|  \mathbb{E}_p \left[ \pmb{\phi}_{s'}^\top \middle| s,a \right] - \pmb{\phi}_s^\top \pmb{M}_a \right| \right| \le \varepsilon_\psi (1 + \gamma M) / \gamma \implies \varepsilon_p \le \varepsilon_\psi (1 + \gamma M) / \gamma.
\end{equation}
\end{proof}

The following lemma is a restatement of Lemma~\ref{lem:t-model} in the main paper.

\begin{lem}\label{lem:t-model-app}
For an MDP, a state representation $\phi$, a LSFM $\{ \pmb{F}_a, \pmb{w}_a \}_{a \in \mathcal{A}}$ and a LAM $\{ \pmb{M}_a, \pmb{w}_a \}_{a \in \mathcal{A}}$ where $\Delta \ge 0$, 
\begin{equation}
\varepsilon_p \le \varepsilon_\psi \frac{1 + \gamma M}{\gamma} +  C_{\gamma,M,N} \Delta, \label{eq:lem-t-model-app}
\end{equation}
where $C_{\gamma,M,N} = \frac{(1 + \gamma) (1 + \gamma M) N}{\gamma (1 - \gamma M)}$
\end{lem}
\begin{proof}
The proof reuses and extends the bound shown in Lemma~\ref{lem:t-model-match}.
Using the LAM $\{ \pmb{M}_a, \pmb{w}_a \}_{a \in \mathcal{A}}$, construct an LSFM $\{ \pmb{F}^*_a, \pmb{w}^*_a \}_{a \in \mathcal{A}}$ such that
\begin{equation}
\pmb{F}^*_a = \pmb{I} + \gamma \pmb{M}_a \underbrace{\frac{1}{| \mathcal{A} |} \sum_{a \in \mathcal{A}} \pmb{F}^*_a}_{=\overline{\pmb{F}}^*}. \label{eq:lem-t-model-0}
\end{equation}
If 
\begin{equation}
 \varepsilon_\psi^* = \sup_{s,a} \left| \pmb{\phi}_s^\top + \gamma \mathbb{E}_p \left[ \pmb{\phi}_{s'}^\top \overline{\pmb{F}}^* \middle| s,a \right] - \pmb{\phi}_s^\top \pmb{F}^*_a \right|,
\end{equation}
then
\begin{equation}
\varepsilon_p \le \varepsilon_\psi^* \frac{1 + \gamma M}{\gamma}, \label{eq:lem-t-model-0}
\end{equation}
by Lemma~\ref{lem:t-model-match}.
By linearity of the expectation operator, the SF-error for the LSFM  $\{ \pmb{F}^*_a, \pmb{w}^*_a \}_{a \in \mathcal{A}}$ and LSFM  $\{ \pmb{F}_a, \pmb{w}_a \}_{a \in \mathcal{A}}$ can be founded for any arbitrary state and action pair $s,a$ with
\begin{align}
& \Big| \Big| \underbrace{\left( \pmb{\phi}_s^\top + \gamma \mathbb{E}_p \left[ \pmb{\phi}_{s'}^\top \overline{\pmb{F}}^* \middle| s,a \right] - \pmb{\phi}_s^\top \pmb{F}^*_a \right)}_{=\pmb{\delta}_{s,a}^*} - \underbrace{\left( \pmb{\phi}_s^\top + \gamma \mathbb{E}_p \left[ \pmb{\phi}_{s'}^\top \overline{\pmb{F}} \middle| s,a \right] - \pmb{\phi}_s^\top \pmb{F}_a \right)}_{\pmb{\delta}_{s,a}} \Big| \Big| \label{eq:lem-t-model-1} \\
&\le \gamma N \left( \overline{\pmb{F}}^* - \overline{\pmb{F}} \right) + N \left( \pmb{F}^*_a - \pmb{F}_a \right). \label{eq:lem-t-model-2}
\end{align}
As stated in Equation~\eqref{eq:lem-t-model-1}, the SF errors for the LSFM $\{ \pmb{F}_a, \pmb{w}_a \}_{a \in \mathcal{A}}$ are defined as $\pmb{\delta}_{s,a}$ and for the LSFM $\{ \pmb{F}^*_a, \pmb{w}^*_a \}_{a \in \mathcal{A}}$ as $\pmb{\delta}^*_{s,a}$.
Because the LSFM $\{ \pmb{F}_a, \pmb{w}_a \}_{a \in \mathcal{A}}$ has a $\Delta > 0$, we define
\begin{align}
\overline{\pmb{\Delta}} &= \frac{1}{| \mathcal{A} |} \sum_a \pmb{I} + \gamma \pmb{M}_a \overline{\pmb{F}} - \pmb{F}_a \label{eq:lem-t-model-match-2b} \\
&= \pmb{I} + \gamma \overline{\pmb{M}} \overline{\pmb{F}} - \overline{\pmb{F}}. \label{eq:lem-t-model-match-3b}
\end{align}
By Equation~\eqref{eq:lem-t-model-match-2b}, $|| \overline{\pmb{\Delta}} || \le \frac{1}{| \mathcal{A} |} \sum_a || \pmb{I} + \gamma \pmb{M}_a \overline{\pmb{F}} - \pmb{F}_a || \le \Delta$ (by triangle inequality).
Reusing Equation~\eqref{eq:lem-t-model-match-3b} one can write, 
\begin{align}
\overline{\pmb{F}}^* - \overline{\pmb{F}} &= \pmb{I} + \gamma \overline{\pmb{M}} \overline{\pmb{F}}^* - \pmb{I} - \gamma \overline{\pmb{M}} \overline{\pmb{F}} + \overline{\pmb{\Delta}} \\
&= \gamma \overline{\pmb{M}} ( \overline{\pmb{F}}^* - \overline{\pmb{F}} ) + \overline{\pmb{\Delta}} \\
\iff || \overline{\pmb{F}}^* - \overline{\pmb{F}} || &\le \gamma M || \overline{\pmb{F}}^* - \overline{\pmb{F}} || + \Delta \\
\iff || \overline{\pmb{F}}^* - \overline{\pmb{F}} || &\le \frac{\Delta}{1 - \gamma M} \label{eq:lem-t-model-4}
\end{align}
Similarly, define
\begin{equation}
\pmb{\Delta}_a = \pmb{I} + \gamma \pmb{M}_a \overline{\pmb{F}} - \pmb{F}_a,
\end{equation}
then,
\begin{align}
\pmb{F}^*_a - \pmb{F}_a &= \pmb{I} + \gamma \pmb{M}_a \overline{\pmb{F}}^* - \pmb{I} - \gamma \pmb{M}_a \overline{\pmb{F}} + \pmb{\Delta}_a \\
&= \gamma \pmb{M}_a (\overline{\pmb{F}}^* - \overline{\pmb{F}}) + \pmb{\Delta}_a \\
\iff || \pmb{F}^*_a - \pmb{F}_a || &\le \gamma M  || \overline{\pmb{F}}^* - \overline{\pmb{F}} || + \Delta \\
&\le \gamma M  \frac{\Delta}{1 - \gamma M} + \Delta \\
&= \frac{\Delta}{1 - \gamma M} \label{eq:lem-t-model-6a}
\end{align}
Substituting lines~\eqref{eq:lem-t-model-4} and~\eqref{eq:lem-t-model-6a} into~\eqref{eq:lem-t-model-2}, 
\begin{equation}
\left| \left| \pmb{\delta}^*_{s,a} - \pmb{\delta}_{s,a} \right| \right| \le \frac{(1 + \gamma) N \Delta}{1 - \gamma M}. \label{eq:lem-t-model-5}
\end{equation}
For both LSFM $\{ \pmb{F}_a, \pmb{w}_a \}_{a \in \mathcal{A}}$ and $\{ \pmb{F}^*_a, \pmb{w}^*_a \}_{a \in \mathcal{A}}$, the worst case SF prediction errors $\varepsilon_\psi$ and $\varepsilon^*_\psi$ are defined as
\begin{equation}
\varepsilon_\psi = \sup_{s,a} || \pmb{\delta}_{s,a} || ~\text{and}~ \varepsilon^*_\psi = \sup_{s,a} || \pmb{\delta}^*_{s,a} ||.
\end{equation}
To find a bound on $| \varepsilon_\psi - \varepsilon^*_\phi |$, the maximizing state and action pairs are defined as
\begin{equation}
s_\text{sup}, a_\text{sup} = \arg \sup_{s,a} || \pmb{\delta}_{s,a} || ~\text{and}~ s_\text{sup}^*, a_\text{sup}^* = \arg \sup_{s,a} || \pmb{\delta}^*_{s,a} ||. \label{eq:lem-t-model-6b}
\end{equation}
If $(s_\text{sup}, a_\text{sup}) = (s_\text{sup}^*, a_\text{sup}^*)$ then
\begin{align}
\left| \varepsilon_\psi - \varepsilon_\psi^* \right| &\le \frac{(1 + \gamma) N \Delta}{1 - \gamma M}. &(\text{by~\eqref{eq:lem-t-model-5}}) \label{eq:lem-t-model-7}
\end{align} 
If $(s_\text{sup}, a_\text{sup}) \ne (s_\text{sup}^*, a_\text{sup}^*)$ and $\varepsilon_\psi \ge \varepsilon_\psi^*$, then
\begin{align}
\varepsilon_\psi - \varepsilon_\psi^* &= \big|\big| \pmb{\delta}_{s_\text{sup}, a_\text{sup}} \big|\big| - \big|\big| \pmb{\delta}^*_{s_\text{sup}^*, a_\text{sup}^*} \big|\big| \\
&\le \big|\big| \pmb{\delta}_{s_\text{sup}, a_\text{sup}} \big|\big| - \big|\big| \pmb{\delta}^*_{s_\text{sup}, a_\text{sup}} \big|\big| &(\text{by~\eqref{eq:lem-t-model-6b}}) \\
&\le \big|\big| \pmb{\delta}_{s_\text{sup}, a_\text{sup}} - \pmb{\delta}^*_{s_\text{sup}, a_\text{sup}} \big|\big| &(\text{by inv. triangle ineq.}) \\
&\le \frac{(1 + \gamma) N \Delta}{1 - \gamma M}. &(\text{by~\eqref{eq:lem-t-model-5}}) \label{eq:lem-t-model-8}
\end{align}
If $(s_\text{sup}, a_\text{sup}) \ne (s_\text{sup}^*, a_\text{sup}^*)$ and $\varepsilon_\psi^* \ge \varepsilon_\psi$, then
\begin{align}
\varepsilon_\psi^* - \varepsilon_\psi &= \big|\big| \pmb{\delta}^*_{s_\text{sup}^*, a_\text{sup}^*} \big|\big| - \big|\big| \pmb{\delta}_{s_\text{sup}, a_\text{sup}} \big|\big| \\
&\le \big|\big| \pmb{\delta}^*_{s_\text{sup}^*, a_\text{sup}^*} \big|\big| - \big|\big| \pmb{\delta}_{s_\text{sup}^*, a_\text{sup}^*} \big|\big| &(\text{by~\eqref{eq:lem-t-model-6b}}) \\
&\le \big|\big| \pmb{\delta}^*_{s_\text{sup}^*, a_\text{sup}^*} - \pmb{\delta}_{s_\text{sup}^*, a_\text{sup}^*} \big|\big| &(\text{by inv. triangle ineq.}) \\
&\le \frac{(1 + \gamma) N \Delta}{1 - \gamma M}. &(\text{by~\eqref{eq:lem-t-model-5}}) \label{eq:lem-t-model-9}
\end{align}
By lines~\eqref{eq:lem-t-model-7}, ~\eqref{eq:lem-t-model-8}, and~\eqref{eq:lem-t-model-9}, 
\begin{equation}
\left| \varepsilon_\psi - \varepsilon_\psi^* \right| \le \frac{(1 + \gamma) N \Delta}{1 - \gamma M} \implies \varepsilon_\psi^* \le \varepsilon_\psi + \frac{(1 + \gamma) N \Delta}{1 - \gamma M}. \label{eq:lem-t-model-10}
\end{equation}
Substituting~\eqref{eq:lem-t-model-10} into~\eqref{eq:lem-t-model-0} results in the desired bound:
\begin{equation}
\varepsilon_p \le \varepsilon_\psi^* \frac{1 + \gamma M}{\gamma} \le \left( \varepsilon_\psi + \frac{(1 + \gamma) N \Delta}{1 - \gamma M} \right)  \frac{1 + \gamma M}{\gamma} =  \varepsilon_\psi \frac{1 + \gamma M}{\gamma} +  \frac{(1 + \gamma) (1 + \gamma M) N}{\gamma (1 - \gamma M)} \Delta.
\end{equation}
\end{proof}

Using these lemmas, Theorem~\ref{thm:rollout-bound-lam} can be proven.

\begin{proof}[Proof of Theorem~\ref{thm:rollout-bound-lam}]
The proof is by induction on the sequence length $T$.

\noindent \textbf{Base Case:} For $T=1$, 
\begin{equation}
\left| \pmb{\phi}_s^\top \pmb{w}_{a_1} - \mathbb{E}_p \left[ r_1 \middle| s,a_1 \right] \right| = \left| \pmb{\phi}_s^\top \pmb{w}_{a_1} - r(s,a_1) \right| \le \varepsilon_r.
\end{equation}

\noindent \textbf{Induction Step:} Assume that the bound~\eqref{eq:rollout-bnd-lam} holds for $T$, then for $T+1$, 
\begin{align}
& \left| \pmb{\phi}_s^\top \pmb{M}_{a_1} \cdots \pmb{M}_{a_T} \pmb{w}_{a_{T+1}} - \mathbb{E}_p \left[ r_{T+1} \middle| s,a_1,...,a_{T+1} \right] \right| \\
&= \Big| \pmb{\phi}_s^\top \pmb{M}_{a_1} \cdots \pmb{M}_{a_T} \pmb{w}_{a_{T+1}} - \mathbb{E}_p \left[ \pmb{\phi}_{s_2}^\top \pmb{M}_{a_2} \cdots \pmb{M}_{a_T} \pmb{w}_{a_{T+1}}  \middle| s,a_1 \right] \nonumber \\
&~~~~+ \mathbb{E}_p \left[ \pmb{\phi}_{s_2}^\top \pmb{M}_{a_2} \cdots \pmb{M}_{a_T} \pmb{w}_{a_{T+1}}  \middle| s,a_1 \right] - \mathbb{E}_p \left[ r_{T+1} \middle| s,a_1,...,a_{T+1} \right] \Big| \\
&\le \Big| \left( \pmb{\phi}_s^\top \pmb{M}_{a_1} - \mathbb{E}_p \left[ \pmb{\phi}_{s_2}^\top \middle| s,a_1 \right] \right) \pmb{M}_{a_2} \cdots \pmb{M}_{a_T} \pmb{w}_{a_{T+1}} \Big| \nonumber \\
&~~~~+ \Big| \mathbb{E}_p \left[ \pmb{\phi}_{s_2}^\top \pmb{M}_{a_2} \cdots \pmb{M}_{a_T} \pmb{w}_{a_{T+1}}  \middle| s,a_1 \right] - \mathbb{E}_p \left[ r_{T+1} \middle| s,a_1,...,a_{T+1} \right] \Big| \\
&\le \Big| \Big| \pmb{\phi}_s^\top \pmb{M}_{a_1} - \mathbb{E}_p \left[ \pmb{\phi}_{s_2}^\top \middle| s,a_1 \right] \Big| \Big| \cdot \Big| \Big| \pmb{M}_{a_2} \cdots \pmb{M}_{a_T} \pmb{w}_{a_{T+1}} \Big| \Big| \nonumber \\
&~~~~+ \Big| \mathbb{E}_p \left[ \pmb{\phi}_{s_2}^\top \pmb{M}_{a_2} \cdots \pmb{M}_{a_T} \pmb{w}_{a_{T+1}}  \middle| s,a_1 \right] - \mathbb{E}_p \left[ r_{T+1} \middle| s,a_1,...,a_{T+1} \right] \Big| \nonumber \\
&\le \Big| \Big| \pmb{\phi}_s^\top \pmb{M}_{a_1} - \mathbb{E}_p \left[ \pmb{\phi}_{s_2}^\top \middle| s,a_1 \right] \Big| \Big| \cdot \Big| \Big| \pmb{M}_{a_2} \cdots \pmb{M}_{a_T} \pmb{w}_{a_{T+1}} \Big| \Big| \nonumber \\
&~~~~+ \Big| \mathbb{E}_p \left[ \pmb{\phi}_{s_2}^\top \pmb{M}_{a_2} \cdots \pmb{M}_{a_T} \pmb{w}_{a_{T+1}} - \mathbb{E}_p \left[ r_{T+1} \middle| s_1,a_2...,a_{T+1} \right]  \middle| s,a_1 \right] \Big| \\
&\le \varepsilon_p M^{T-1} W \nonumber \\
&~~~~+ \varepsilon_p \sum_{t=1}^{T-1} M^t W + \varepsilon_r \\
&= \varepsilon_p \sum_{t=1}^{(T+1)-1} M^t W + \varepsilon_r.
\end{align}
\end{proof}

\begin{thm}\label{thm:rollout-bound-lsfm}
For an MDP, state representation $\phi : \mathcal{S} \to \mathbb{R}^n$, and for all $T \ge 1, s, a_1,...,a_T$,
\begin{equation}
\left| \pmb{\phi}_s^\top \pmb{M}_{a_1} \cdots \pmb{M}_{a_{T-1}} \pmb{w}_{a_T} - \mathbb{E}_p \left[ r_T \middle| s,a_1,...,a_T \right] \right| \le \left( \varepsilon_\psi \frac{1 + \gamma M}{\gamma} + C_{\gamma,M,N} \Delta \right) \sum_{t=1}^{T-1} M^t W + \varepsilon_r. \nonumber 
\end{equation}
\end{thm}
\begin{proof}[Proof of Theorem~\ref{thm:rollout-bound-lsfm}]
The proof is by reusing the bound in Theorem~\ref{thm:rollout-bound-lam} and substituting $\varepsilon_p$ with the bound presented in Lemma~\ref{lem:t-model}. 
\end{proof}

Theorem~\ref{thm:approx-val-fn}, which is stated in the main paper, can be proven as follows.

\begin{proof}[Proof of Theorem~\ref{thm:approx-val-fn}]
The value error term can be upper-bounded with 
\begin{align}
\left| V^\pi(s) - \pmb{\phi}_s^\top \pmb{v}^\pi \right| &\le \sum_{a \in \mathcal{A}} \pi(s,a) \left| r(s,a) + \gamma \mathbb{E}_p \left[ V^{\pi}(s') \middle| s,a \right] - \pmb{\phi}_s^\top \pmb{w}_a - \gamma \pmb{\phi}_s^\top \pmb{M}_a \pmb{v}^\pi \right| \\
&\le \sum_{a \in \mathcal{A}} \pi(s,a) \left| r(s,a) - \pmb{\phi}_s^\top \pmb{w}_a \right| + \gamma \left| \mathbb{E}_p \left[ V^{\pi}(s') \middle| s,a \right] - \pmb{\phi}_s^\top \pmb{M}_a \pmb{v}^\pi \right| \label{eq:thm-approx-val-fn-1}
\end{align}
The second term in Equation~\eqref{eq:thm-approx-val-fn-1} is bounded by
\begin{align}
&\left| \mathbb{E}_p \left[ V^{\pi}(s') \middle| s,a \right] - \pmb{\phi}_s^\top \pmb{M}_a \pmb{v}^\pi \right| \nonumber \\
&= \left| \mathbb{E}_p \left[ V^{\pi}(s') \middle| s,a \right] - \mathbb{E}_p \left[ \pmb{\phi}_{s'}^\top \pmb{v}^\pi \middle| s,a \right] + \mathbb{E}_p \left[ \pmb{\phi}_{s'}^\top \pmb{v}^\pi \middle| s,a \right] - \pmb{\phi}_s^\top \pmb{M}_a \pmb{v}^\pi \right| \nonumber \\
&= \sup_s \left| V^\pi(s) - \pmb{\phi}_s^\top \pmb{v}^\pi \right| + \left| \mathbb{E}_p \left[ \pmb{\phi}_{s'}^\top \pmb{v}^\pi \middle| s,a \right] - \pmb{\phi}_s^\top \pmb{M}_a \pmb{v}^\pi \right| \nonumber \\
&= \sup_s \left| V^\pi(s) - \pmb{\phi}_s^\top \pmb{v}^\pi \right| + \left| \left| \mathbb{E}_p \left[ \pmb{\phi}_{s'}^\top \middle| s,a \right] - \pmb{\phi}_s^\top \pmb{M}_a \right| \right| \left| \left| \pmb{v}^\pi \right| \right| \nonumber \\
&= \sup_s \left| V^\pi(s) - \pmb{\phi}_s^\top \pmb{v}^\pi \right| + \varepsilon_p \left| \left| \pmb{v}^\pi \right| \right| \label{eq:thm-approx-val-fn-2}
\end{align}
Substituting~\eqref{eq:thm-approx-val-fn-2} into~\eqref{eq:thm-approx-val-fn-1} results in
\begin{equation}
\left| V^\pi(s) - \pmb{\phi}_s^\top \pmb{v}^\pi \right| \le \sum_{a \in \mathcal{A}} \pi(s,a) \left( \left| r(s,a) - \pmb{\phi}_s^\top \pmb{w}_a \right| + \gamma \left( \sup_s \left| V^\pi(s) - \pmb{\phi}_s^\top \pmb{v}^\pi \right| + \varepsilon_p \left| \left| \pmb{v}^\pi \right| \right| \right) \right).
\end{equation}
Let $B = \sup_s \left| V^\pi(s) - \pmb{\phi}_s^\top \pmb{v}^\pi \right|$, then
\begin{align}
\left| V^\pi(s) - \pmb{\phi}_s^\top \pmb{v}^\pi \right| &\le \sum_{a \in \mathcal{A}} \pi(s,a) ( \varepsilon_r + \gamma \left( B+ \varepsilon_p \left| \left| \pmb{v}^\pi \right| \right| \right) ) \nonumber \\
&= \varepsilon_r + \gamma B + \gamma \varepsilon_p \left| \left| \pmb{v}^\pi \right| \right| \label{eq:thm-approx-val-fn-3}
\end{align}
The bound in Equation~\eqref{eq:thm-approx-val-fn-3} does not depend on any particular state and action pair $s,a$ and thus
\begin{align}
\forall s,a,~ \left| V^\pi(s) - \pmb{\phi}_s^\top \pmb{v}^\pi \right| \le \varepsilon_r + \gamma B + \gamma \varepsilon_p \left| \left| \pmb{v}^\pi \right| \right| & \implies B \le \varepsilon_r + \gamma B + \gamma \varepsilon_p \left| \left| \pmb{v}^\pi \right| \right| \nonumber \\
&\implies B \le \frac{\varepsilon_r + \gamma \varepsilon_\psi \left| \left| \pmb{v}^\pi \right| \right|}{1 - \gamma}.
\end{align}
To bound the Q-value function,
\begin{equation}
\left| Q^\pi(s,a) - \pmb{\phi}_s^\top \pmb{q}_a \right| \le \left| r(s,a) + \gamma \mathbb{E}_p \left[ V^{\pi}(s') \middle| s,a \right] - \pmb{\phi}_s^\top \pmb{w}_a - \gamma \pmb{\phi}_s^\top \pmb{M}_a \pmb{v}^\pi \right|,
\end{equation}
which is similar to Equation~\eqref{eq:thm-approx-val-fn-1} and the proof proceeds in the same way.
The LSFM bound 
\begin{equation}
\frac{\varepsilon_r + \gamma \varepsilon_p \left| \left| \pmb{v}^\pi \right| \right|}{1 - \gamma} \le \frac{\varepsilon_r + \varepsilon_\psi (1 + \gamma M) \left| \left| \pmb{v}^\pi \right| \right| + \gamma C_{\gamma,M,N} \Delta \left| \left| \pmb{v}^\pi \right| \right| }{1 - \gamma}
\end{equation}
follows by Lemma~\ref{lem:t-model}.
\end{proof}

\subsection{Bound on Error Term $\Delta$}\label{app:loss-fn}

The following proposition formally proofs the bound presented in Equation~\eqref{eq:delta-bnd}.

\begin{prop}\label{prop:delta-bnd}
For a data set $\mathcal{D} = \left\{ (s_i,a_i,r_i,s_i') \right\}_{i=1}^D$,
\begin{equation}
\Delta \le \max_a || \pmb{\Phi}_a^+ ||_2^2 \mathcal{L}_{\psi},
\end{equation}
where each row of $\pmb{\Phi}_a$ is set to a row-vector $\pmb{\phi}_{s}$ for a transition $(s,a,r,s') \in \mathcal{D}$ that uses action $a$, and $\pmb{\Phi}_a^+$ is the pseudo-inverse of $\pmb{\Phi}_a$.
\end{prop}
\begin{proof}[Proof of Proposition~\ref{prop:delta-bnd}]
For a data set $\mathcal{D} = \left\{ (s_i,a_i,r_i,s_i') \right\}_{i=1}^D$, construct the matrix $\pmb{\Phi}_a$ and similarly construct the matrix $\pmb{\Phi}_a'$ where each row of $\pmb{\Phi}_a'$ is set to a row-vector $\pmb{\phi}_{s'}$ for a transition $(s,a,r,s') \in \mathcal{D}$ that uses action $a$.
The transition matrix of a LAM can be obtained using a least squares regression and
\begin{equation}
\pmb{M}_a = \arg \min_{\pmb{M}} || \pmb{\Phi}_a \pmb{M} - \pmb{\Phi}_a' ||_2^2 \implies \pmb{M}_a =  \pmb{\Phi}_a^+  \pmb{\Phi}_a',
\end{equation}
where $\pmb{\Phi}_a^+$ is the pseudo-inverse of $ \pmb{\Phi}_a$.
Using this notation, one can write 
\begin{align}
\pmb{\Phi}_a + \gamma \pmb{\Phi}_a' \overline{\pmb{F}} - \pmb{\Phi}_a \pmb{F}_a &= \pmb{L}_a &\iff \\
\pmb{\Phi}_a^+ \pmb{\Phi}_a + \gamma \pmb{\Phi}_a^+ \pmb{\Phi}_a' \overline{\pmb{F}} - \pmb{\Phi}_a^+ \pmb{\Phi}_a \pmb{F}_a &= \pmb{\Phi}_a^+ \pmb{L}_a &\iff \\
\pmb{I} + \gamma \pmb{M}_a \overline{\pmb{F}} - \pmb{F}_a &= \pmb{\Phi}_a^+ \pmb{L}_a &\iff  \\
|| \pmb{I} + \gamma \pmb{M}_a \overline{\pmb{F}} - \pmb{F}_a ||_2^2 &\le ||\pmb{\Phi}_a^+ ||_2^2 || \pmb{L}_a ||_2^2.
\end{align}
Note that $\mathcal{L}_\psi = \sum_{a \in \mathcal{A}} \pmb{L}_a$, and thus
\begin{equation}
\Delta = \max_a || \pmb{I} + \gamma \pmb{M}_a \overline{\pmb{F}} - \pmb{F}_a ||_2^2 \le \max_a ||\pmb{\Phi}_a^+ ||_2^2 \mathcal{L}_\psi.
\end{equation}
\end{proof}

\section{Connection Between Q-learning and SF-learning}\label{app:q-sf-learning-pf}

\begin{proof}[Proof of Proposition~\ref{prop:sf-q-equiv}]
Before proving the main statement, we first make the following observation.
Assuming that for some $t$, $\pmb{\theta}_t = \pmb{F}_t \pmb{w}$, then
\begin{align}
\pmb{w}^\top \pmb{y}_{s,a,r,s'} &= \pmb{w}^\top \left( \pmb{\xi}_{s,a} + \gamma \sum_{a'} b(s',a') \pmb{\psi}^\pi_{s',a'} \right) \\
&= r(s,a) + \gamma \sum_{a'} b(s',a') \pmb{w}^\top \pmb{\psi}_{s',a'}^\pi \\
&= r(s,a) + \gamma \sum_{a'} b(s',a') \pmb{w}^\top \pmb{F}_t \pmb{\xi}_{s',a'} \label{eq:prop1-1} \\
&= r(s,a) + \gamma \sum_{a'} b(s',a') \pmb{\theta}_t^\top \pmb{\xi}_{s',a'} \\
&= y_{s,a,r,s'}. \label{eq:prop1-2}
\end{align}
Equation~\eqref{eq:prop1-1} follows by substitution with Equation~\eqref{eq:sf-assumption}.
The proof of the main statement is by induction on $t$.

\noindent \textit{Base Case:} For $t=1$, assume $\pmb{\theta}_0 = \pmb{F}_0 \pmb{w}$.
Then
\begin{align}
\pmb{w}^\top \pmb{F}_1 &= \pmb{w}^\top \left( \pmb{F}_0 + \alpha_\psi \left( \pmb{F}_0 \pmb{\xi}_{s,a} - \pmb{y}_{s,a,r,s'} \right)^\top \pmb{\xi}_{s,a} \right) \\
&= \pmb{w}^\top \pmb{F}_0 + \alpha_\psi \left( \pmb{w}^\top \pmb{F}_0 \pmb{\xi}_{s,a} - \pmb{w}^\top \pmb{y}_{s,a,r,s'} \right)^\top \pmb{\xi}_{s,a} \\
&= \pmb{\theta}_0 + \alpha_\psi \left( \pmb{\theta}_0^\top \pmb{\xi}_{s,a} - y_{s,a,r,s'} \right)^\top \pmb{\xi}_{s,a} \label{eq:prop1-3} \\
&= \pmb{\theta}_1. \label{eq:prop1-4}
\end{align}
Equation~\eqref{eq:prop1-3} us obtained by substituting the identity in Equation~\eqref{eq:prop1-2} for $t=0$.
Equation~\eqref{eq:prop1-4} is obtained by substituting the linear TD iterate from Equation~\eqref{eq:q-learning}.

\noindent \textit{Induction Step:} Assuming the hypothesis $\pmb{w}^\top \pmb{F}_t = \pmb{\theta}_t^\top$ holds for $t$ and proceeding as in the base case, then
\begin{align}
\pmb{w}^\top \pmb{F}_{t+1} &= \pmb{w}^\top \left( \pmb{F}_t + \alpha_\psi \left( \pmb{F}_t \pmb{\xi}_{s,a} - \pmb{y}_{s,a,r,s'} \right)^\top \pmb{\xi}_{s,a} \right) \\
&= \pmb{w}^\top \pmb{F}_t + \alpha_\psi \left( \pmb{w}^\top \pmb{F}_t \pmb{\xi}_{s,a} - \pmb{w}^\top \pmb{y}_{s,a,r,s'} \right)^\top \pmb{\xi}_{s,a} \\
&= \pmb{\theta}_t + \alpha_\psi \left( \pmb{\theta}_t^\top \pmb{\xi}_{s,a} - y_{s,a,r,s'} \right)^\top \pmb{\xi}_{s,a} \\
&= \pmb{\theta}_{t+1}.
\end{align}
Hence for all $t$, $\pmb{w}^\top \pmb{F}_t = \pmb{\theta}_t$, as desired.

Note that this proof assumes that both iterates are applied for exactly the same transitions.
This assumption is not restrictive assuming that control policies are constructed using the current parameters $\pmb{\theta}_t$ in the case for TD-learning or the parameters $\pmb{F}_t$ and $\pmb{w}$ in the case for SF-learning.
Even in the control case, where an $\varepsilon$-greedy exploration strategy is used, for example, both algorithms will produce an identical sequence of value functions and will chose actions with equal probability.
\end{proof}

\section{Experiment Design}

The presented experiments are conducted on finite MDPs and use a state representation function
\begin{equation}
\phi: s \mapsto \pmb{\Phi}(s,:), \label{eq:feat-mat}
\end{equation}
where $\pmb{\Phi}$ is a $\mathcal{S} \times n$ matrix and $\pmb{\Phi}(s,:)$ is a row with state index $s$.
The feature dimension $n$ is a fixed hyper parameter for each experiment.

\subsection{Matrix Optimization in Column World} \label{app:matrix-opt}

The column world experiment (Figure~\ref{fig:column-world}) learns a state representation using the full transition and reward tables.
Assume that the transition table of the column world task is stored as a set of stochastic transition matrices $\{ \pmb{P}_a \}_{a \in \mathcal{A}}$ and the reward table as a set of reward vectors $\{ \pmb{r}_a \}_{a \in \mathcal{A}}$.
The one-step reward prediction errors and linear SF prediction errors are minimized for the LSFM $\{ \pmb{F}_a, \pmb{w}_a \}_{a \in \mathcal{A}}$ using the loss objective
\begin{equation}
\mathcal{L}_\text{LSFM-mat} = \sum_{a \in \mathcal{A}} || \pmb{\Phi} \pmb{w}_a - \pmb{r}_a ||_2^2 + \alpha_\psi || \pmb{\Phi} + \gamma \pmb{P}_a \pmb{\Phi} \overline{\pmb{F}} - \pmb{\Phi} \pmb{F}_a ||_2^2. \label{eq:lsfm-mat-def}
\end{equation}
For $\alpha_\psi=1$, the loss objective $\mathcal{L}_\text{LSFM-mat}$ is optimized with respect to all free parameters $\{ \pmb{F}_a, \pmb{w}_a \}_{a \in \mathcal{A}}$ and $\pmb{\Phi}$.
Similarly, a LAM $\{ \pmb{M}_a, \pmb{w}_a \}_{a \in \mathcal{A}}$ is computed using the loss objective 
\begin{equation}
\mathcal{L}_\text{LAM-mat} = \sum_{a \in \mathcal{A}} || \pmb{\Phi} \pmb{w}_a - \pmb{r}_a ||_2^2 + || \pmb{\Phi} \pmb{M}_a - \pmb{P}_a \pmb{\Phi} ||_2^2.
\end{equation}
This loss objective is optimized with respect to all free parameters $\{ \pmb{M}_a, \pmb{w}_a \}_{a \in \mathcal{A}}$ and $\pmb{\Phi}$.
Both experiments used the Adam optimizer~\citep{kingma2014adam} with a learning rate of 0.1 and Tensorflow~\citep{tensorflow} default parameters.
Optimization was initialized by sampling entries for $\pmb{\Phi}$ uniformly from the interval $[0,1]$. 
The LAM $\{ \pmb{M}_a, \pmb{w}_a \}_{a \in \mathcal{A}}$ or LSFM $\{ \pmb{F}_a, \pmb{w}_a \}_{a \in \mathcal{A}}$ was initialized using a least squares solution for the initialization of $\pmb{\Phi}$.

\subsection{Puddle-World Experiment}\label{app:opt}

In the puddle world MDP transitions are probabilistic, because with a 5\% chance, the agent does not move after selecting any action.
The partition maps presented in Figures~\ref{fig:puddle-world-gen-lam} and~\ref{fig:puddle-world-gen-lsfm} were obtained by clustering latent state vectors using agglomerative clustering.
A finite data set of transitions $\mathcal{D} = \left\{ (s_i,a_i,r_i,s_i') \right\}_{i=1}^D$ was collected by selecting actions uniformly as random. 
Given such a data set $\mathcal{D}$, the loss objective
\begin{equation*}
\mathcal{L}_\text{LAM} = \underbrace{\sum_{i=1}^D \Big( \pmb{\phi}_{s_i}^\top \pmb{w}_{a_i} - r_i \Big)^2}_{=\mathcal{L}_r} + \alpha_p \underbrace{ \sum_{i=1}^D \Big| \Big| \pmb{\phi}_{s_i}^\top \pmb{M}_a - \pmb{\phi}_{s_i'}^\top \Big| \Big|_2^2 }_{= \mathcal{L}_p} + \alpha_N \underbrace{ \sum_{i=1}^D \Big( \Big| \Big| \pmb{\phi}_{s_i} \Big| \Big|_2^2 - 1 \Big)^2}_{= \mathcal{L}_N},
\end{equation*}
is used to approximate a reward-predictive state representation using a LAM.
Optimization was initialized by each entry of the matrix $\pmb{\Phi}$ uniformly at random and then finding a LAM $\{ \pmb{M}_a, \pmb{w}_a \}_{a \in \mathcal{A}}$ for this initialized representation using least squares regression.

For the LSFM experiment, the matrix $\pmb{\Phi}$ was also initialized using values sampled uniformly at random. 
The LSFM $\{ \pmb{F}_a, \pmb{w}_a \}_{a \in \mathcal{A}}$ was set to zero matrices and vectors at initialization.
Both loss objective functions were optimized using the Adam optimizer with Tensorflow's default parameters.
Table~\ref{tab:hyper-param-puddle-world} lists the hyper-parameter that were found to work best for each model.
Figures~\ref{fig:puddle-world-value-error-lam} and~\ref{fig:puddle-world-value-error-lsfm} are plotted by first evaluating an $\varepsilon$-greedy policy using the full transition and reward tables of the task. 
Then the state representation is used to find an approximation of the value functions for each $\varepsilon$ setting using least-squares linear regression.
Each curve then plots the maximum value prediction error.

\begin{figure}
\small
\centering
\newcolumntype{C}[1]{>{\raggedright\arraybackslash}m{#1}}
\begin{tabular}{l l l | l}
\hline
Hyper-Parameter & LAM & LSFM & Tested Values \\
\hline \hline
Learning Rate & 0.0005 & 0.0005 & 0.0001, 0.0005, 0.001, 0.005 \\ \hline
$\alpha_\psi$ & - & 0.01 & 0.0001, 0.001, 0.01, 0.1, 1.0, 10.0, 100.0 \\ \hline
$\alpha_p$ & 1.0 & - & 0.0001, 0.001, 0.01, 0.1, 1.0, 10.0, 100.0 \\ \hline
$\alpha_N$ & 0.1 & 0.0 & 0.0, 0.0001, 0.001, 0.01, 0.1 \\ \hline
Feature Dimension & 80 & 80 &  \\ \hline
Batch Size & 50 & 50 & \\ \hline
Number of Training Transitions & 10000 & 10000 & \\ \hline
Number of Gradient Steps & 50000 & 50000 & \\ \hline
\end{tabular}

\captionof{table}[]{Hyper-Parameter for Puddle-World Experiment}
\label{tab:hyper-param-puddle-world}
\end{figure}

\subsection{Transfer Experiments}\label{app:transfer}

For the transfer experiment presented in Section~\ref{sec:transfer}, a training data set of 10000 transitions was collected from Task B.
The LSFM was trained using the hyper-parameter listed in Table~\ref{tab:hyper-param-transfer-lsfm}.

\begin{figure}
\small
\centering
\newcolumntype{C}[1]{>{\raggedright\arraybackslash}m{#1}}
\begin{tabular}{l l | l}
\hline
Hyper-Parameter & LSFM & Tested Values\\
\hline \hline
Learning Rate & 0.001 & 0.0001, 0.0005, 0.001, 0.005 \\ \hline
$\alpha_\psi$ & 0.0001 & 0.001, 0.001, 0.01, 0.1, 1.0, 10.0, 100.0 \\ \hline
$\alpha_N$ & 0.0 & 0.0, 0.0001, 0.001, 0.01, 0.1 \\ \hline
Feature Dimension & 50 &  \\ \hline
Batch Size & 50 &  \\ \hline
Number of Training Transitions & 10000 &  \\ \hline
Number of Gradient Steps & 50000 &  \\ \hline
\end{tabular}

\captionof{table}[]{Hyper-Parameter for LSFM on Task A}
\label{tab:hyper-param-transfer-lsfm}
\end{figure}

Value-predictive state representations are learned using a modified version of Neural Fitted Q-iteration~\citep{riedmiller2005neural}.
The Q-value function is computed with
\begin{equation}
Q(s,a; \pmb{\Phi},\{\pmb{q}_a\}_{a \in \mathcal{A}} ) = \pmb{\phi}_s^\top \pmb{q}_a,
\end{equation}
where the state features $\pmb{\phi}_s$ are computed as shown in Equation~\eqref{eq:feat-mat}.
To find a value-predictive state representation, the loss function
\begin{equation}
\mathcal{L}_q (\pmb{\Phi}, \{\pmb{q}_a\}_{a \in \mathcal{A}}) = \sum_{i=1}^D (Q(s_i,a_i; \pmb{\Phi},\{\pmb{q}_a\}_{a \in \mathcal{A}} )  - y_{s_i,a_i,r_i,s'_i})^2
\end{equation}
is minimized using stochastic gradient descent on a transition data set $\mathcal{D} = \left\{ (s_i,a_i,r_i,s_i') \right\}_{i=1}^D$.
When differentiating the loss objective $\mathcal{L}_q$ with respect to its free parameters $\pmb{\Phi}, \{\pmb{q}_a\}_{a \in \mathcal{A}}$, the prediction target
\begin{equation}
y_{s,a,r,s'} = r + \gamma \max_{a'} Q(s',a'; \pmb{\Phi},\{\pmb{q}_a\}_{a \in \mathcal{A}} )
\end{equation}
is considered a constant and no gradient of $y_{s,a,r,s'}$ is computed.
A value-predictive state representation is learned for Task A by optimizing over all free parameters $\pmb{\Phi}, \{\pmb{q}_a\}_{a \in \mathcal{A}}$.
Table~\ref{tab:hyper-param-transfer-fittedq} lists the used hyper-parameter.
For Task A the best hyper-parameter setting was obtained by testing the learned state representation on Task A and using the model that produced the shortest episode length, averaged over 20 repeats.

On Task B, a previously learned state representation is evaluated using the same implementation of Fitted Q-iteration, but the previously learned state representation is re-used and considered a constant.
At transfer, gradients of $\mathcal{L}_q$ are only computed with respect to the vector set $\{\pmb{q}_a\}_{a \in \mathcal{A}}$ and the feature matrix $\pmb{\Phi}$ is held constant.

The tabular model first compute a partial transition and reward table of Task B by averaging different transitions and reward using the given data set.
If no transition is provided for a particular state and action pair, uniform transitions are assumed.
If no reward is provided for a particular state and action pair, a reward value is sampled uniformly at random from the interval $[0,1]$.
Augmenting a partial transition and reward table is equivalent to providing the agent with a uniform prior over rewards and transitions.
The tabular model's optimal policy is computed using value iteration.

To plot the right panel in Figure~\ref{fig:transfer-learning}, twenty different transition data sets of a certain fixed size were collected and the Fitted Q-iteration algorithm was used to approximate the optimal value function.
For both tested state representations and data sets, a small enough learning rate was found to guarantee that Fitted Q-iteration converges.
The found solutions were evaluated twenty times, and if all evaluation completed the navigation task within 22 time steps (which is close to optimal), then this data set is considered to be optimally solved.
Note that all tested evaluation runs either complete within 22 time steps or hit the timeout threshold of 5000 time steps.
Table~\ref{tab:hyper-param-transfer-fittedq} lists the hyper-parameters used for Fitted Q-iteration to obtain the right panel in Figure~\ref{fig:transfer-learning}.
For transfer evaluation, the hyper-parameter setting was used that approximated the Q-values optimal in Task B with the lowest error.

\begin{figure}
\small
\centering
\newcolumntype{C}[1]{>{\raggedright\arraybackslash}m{#1}}
\scalebox{0.95}{
\begin{tabular}{l p{3.3cm} p{3.3cm} | l}
\hline
Hyper-Parameter & Fitted Q-iteration, learning on Task A & Fitted Q-iteration, evaluation on Task B& Tested Values\\
\hline \hline
Learning Rate & 0.001 & 0.00001 & 0.00001, 0.0001, 0.001, 0.01 \\ \hline
Feature Dimension & 50 & 50 &  \\ \hline
Batch Size & 50 & 50 &  \\ \hline
Training Transitions & 10000 & Varies &   \\ \hline
Gradient Steps & 50000 & 20000 & \\ \hline
\end{tabular}}

\captionof{table}[]{Hyper-Parameter used for Fitted Q-iteration}
\label{tab:hyper-param-transfer-fittedq}
\end{figure}

\subsection*{C.4 Combination Lock Experiment}\label{app:combination-lock}

For the combination lock simulations presented in Figure~\ref{fig:combintation-lock}, each Q-learning configuration was tested independently on each task with learning rates 0.1, 0.5, and 0.9.
Because Q-values were initialized optimistically with a value of 1.0, each plot in Figure~\ref{fig:combintation-lock} uses a learning rate of 0.9.

To find a reward-predictive state representation, the LSFM loss objective $\mathcal{L}_\text{LSFM-mat}$ (Equation~\eqref{eq:lsfm-mat-def}) was optimized 100000 iterations using Tensorflow's Adam optimizer with a learning rate of 0.005 and $\alpha_\psi = 0.01$.

\bibliography{main}

\begin{thebibliography}{50}
\providecommand{\natexlab}[1]{#1}
\providecommand{\url}[1]{\texttt{#1}}
\expandafter\ifx\csname urlstyle\endcsname\relax
  \providecommand{\doi}[1]{doi: #1}\else
  \providecommand{\doi}{doi: \begingroup \urlstyle{rm}\Url}\fi

\bibitem[Abadi et~al.(2015)Abadi, Agarwal, Barham, Brevdo, Chen, Citro,
  Corrado, Davis, Dean, Devin, Ghemawat, Goodfellow, Harp, Irving, Isard, Jia,
  Jozefowicz, Kaiser, Kudlur, Levenberg, Man\'{e}, Monga, Moore, Murray, Olah,
  Schuster, Shlens, Steiner, Sutskever, Talwar, Tucker, Vanhoucke, Vasudevan,
  Vi\'{e}gas, Vinyals, Warden, Wattenberg, Wicke, Yu, and Zheng]{tensorflow}
Mart\'{\i}n Abadi, Ashish Agarwal, Paul Barham, Eugene Brevdo, Zhifeng Chen,
  Craig Citro, Greg~S. Corrado, Andy Davis, Jeffrey Dean, Matthieu Devin,
  Sanjay Ghemawat, Ian Goodfellow, Andrew Harp, Geoffrey Irving, Michael Isard,
  Yangqing Jia, Rafal Jozefowicz, Lukasz Kaiser, Manjunath Kudlur, Josh
  Levenberg, Dandelion Man\'{e}, Rajat Monga, Sherry Moore, Derek Murray, Chris
  Olah, Mike Schuster, Jonathon Shlens, Benoit Steiner, Ilya Sutskever, Kunal
  Talwar, Paul Tucker, Vincent Vanhoucke, Vijay Vasudevan, Fernanda Vi\'{e}gas,
  Oriol Vinyals, Pete Warden, Martin Wattenberg, Martin Wicke, Yuan Yu, and
  Xiaoqiang Zheng.
\newblock {TensorFlow}: Large-scale machine learning on heterogeneous systems,
  2015.
\newblock URL \url{https://www.tensorflow.org/}.
\newblock Software available from tensorflow.org.

\bibitem[Abel et~al.(2016)Abel, Hershkowitz, and Littman]{abel2016icml}
David Abel, David Hershkowitz, and Michael Littman.
\newblock Near optimal behavior via approximate state abstraction.
\newblock In \emph{Proceedings of The 33rd International Conference on Machine
  Learning}, pages 2915--2923, 2016.

\bibitem[Abel et~al.(2018)Abel, Arumugam, Lehnert, and Littman]{abel2018sa}
David Abel, Dilip Arumugam, Lucas Lehnert, and Michael Littman.
\newblock State abstractions for lifelong reinforcement learning.
\newblock In Jennifer Dy and Andreas Krause, editors, \emph{Proceedings of the
  35th International Conference on Machine Learning}, volume~80 of
  \emph{Proceedings of Machine Learning Research}, pages 10--19,
  Stockholmsm{\"a}ssan, Stockholm Sweden, 10--15 Jul 2018. PMLR.
\newblock URL \url{http://proceedings.mlr.press/v80/abel18a.html}.

\bibitem[Abel et~al.(2019)Abel, Arumugam, Asadi, Jinnai, Littman, and
  Wong]{abel2019rlit}
David Abel, Dilip Arumugam, Kavosh Asadi, Yuu Jinnai, Michael~L. Littman, and
  Lawson~L.S. Wong.
\newblock State abstraction as compression in apprenticeship learning.
\newblock In \emph{Proceedings of the AAAI Conference on Artificial
  Intelligence}, 2019.

\bibitem[Asadi et~al.(2018)Asadi, Misra, and Littman]{asadi2018lipschitzmb}
Kavosh Asadi, Dipendra Misra, and Michael Littman.
\newblock {L}ipschitz continuity in model-based reinforcement learning.
\newblock In Jennifer Dy and Andreas Krause, editors, \emph{Proceedings of the
  35th International Conference on Machine Learning}, volume~80 of
  \emph{Proceedings of Machine Learning Research}, pages 264--273,
  Stockholmsm{\"a}ssan, Stockholm Sweden, 10--15 Jul 2018. PMLR.
\newblock URL \url{http://proceedings.mlr.press/v80/asadi18a.html}.

\bibitem[Azar et~al.(2017)Azar, Osband, and Munos]{azar2017minimaxregret}
Mohammad~Gheshlaghi Azar, Ian Osband, and R{\'e}mi Munos.
\newblock Minimax regret bounds for reinforcement learning.
\newblock In \emph{Proceedings of the 34th International Conference on Machine
  Learning-Volume 70}, pages 263--272. JMLR.org, 2017.

\bibitem[Barreto et~al.(2016)Barreto, Munos, Schaul, and
  Silver]{barreto2016successor}
Andr{\'{e}} Barreto, R{\'{e}}mi Munos, Tom Schaul, and David Silver.
\newblock Successor features for transfer in reinforcement learning.
\newblock \emph{CoRR}, abs/1606.05312, 2016.
\newblock URL \url{http://arxiv.org/abs/1606.05312}.

\bibitem[Barreto et~al.(2017{\natexlab{a}})Barreto, Dabney, Munos, Hunt,
  Schaul, van Hasselt, and Silver]{baretto2017sf}
Andr{\'e} Barreto, Will Dabney, R{\'e}mi Munos, Jonathan~J Hunt, Tom Schaul,
  Hado~P van Hasselt, and David Silver.
\newblock Successor features for transfer in reinforcement learning.
\newblock In \emph{Advances in Neural Information Processing Systems}, pages
  4055--4065, 2017{\natexlab{a}}.

\bibitem[Barreto et~al.(2017{\natexlab{b}})Barreto, Dabney, Munos, Hunt,
  Schaul, van Hasselt, and Silver]{barreto2017sf}
Andr{\'e} Barreto, Will Dabney, R{\'e}mi Munos, Jonathan~J Hunt, Tom Schaul,
  Hado~P van Hasselt, and David Silver.
\newblock Successor features for transfer in reinforcement learning.
\newblock In \emph{Advances in Neural Information Processing Systems}, pages
  4055--4065, 2017{\natexlab{b}}.

\bibitem[Barreto et~al.(2018)Barreto, Borsa, Quan, Schaul, Silver, Hessel,
  Mankowitz, Zidek, and Munos]{baretto2018deepsf}
Andre Barreto, Diana Borsa, John Quan, Tom Schaul, David Silver, Matteo Hessel,
  Daniel Mankowitz, Augustin Zidek, and Remi Munos.
\newblock Transfer in deep reinforcement learning using successor features and
  generalised policy improvement.
\newblock In Jennifer Dy and Andreas Krause, editors, \emph{Proceedings of the
  35th International Conference on Machine Learning}, volume~80 of
  \emph{Proceedings of Machine Learning Research}, pages 501--510,
  Stockholmsm{\"a}ssan, Stockholm Sweden, 10--15 Jul 2018. PMLR.
\newblock URL \url{http://proceedings.mlr.press/v80/barreto18a.html}.

\bibitem[Bellman(1961)]{bellman1961adaptivecontrol}
Richard~E Bellman.
\newblock \emph{Adaptive Control Processes: A Guided Tour}, volume 2045.
\newblock Princeton University Press, 1961.

\bibitem[Bertsekas(2011)]{bertsekas2011dp}
Dimitri~P Bertsekas.
\newblock Dynamic programming and optimal control 3rd edition, volume ii.
\newblock \emph{Belmont, MA: Athena Scientific}, 2011.

\bibitem[Boyan and Moore(1995)]{boyan1995generalization}
Justin~A Boyan and Andrew~W Moore.
\newblock Generalization in reinforcement learning: Safely approximating the
  value function.
\newblock In \emph{Advances in Neural Information Processing Systems}, pages
  369--376, 1995.

\bibitem[Dayan(1993)]{dayan1993successor}
Peter Dayan.
\newblock Improving generalization for temporal difference learning: The
  successor representation.
\newblock \emph{Neural Computation}, 5\penalty0 (4):\penalty0 613--624, 1993.

\bibitem[Even-Dar and Mansour(2003)]{even2003nphardsa}
Eyal Even-Dar and Yishay Mansour.
\newblock Approximate equivalence of {M}arkov decision processes.
\newblock In \emph{Learning Theory and Kernel Machines}, pages 581--594.
  Springer, 2003.

\bibitem[Ferns et~al.(2004)Ferns, Panangaden, and
  Precup]{ferns2004bisimmetrics}
Norm Ferns, Prakash Panangaden, and Doina Precup.
\newblock Metrics for finite markov decision processes.
\newblock In \emph{Proceedings of the 20th Conference on Uncertainty in
  Artificial Intelligence}, pages 162--169. AUAI Press, 2004.

\bibitem[Ferns et~al.(2011)Ferns, Panangaden, and
  Precup]{ferns2011bisimulation}
Norm Ferns, Prakash Panangaden, and Doina Precup.
\newblock Bisimulation metrics for continuous {M}arkov decision processes.
\newblock \emph{SIAM Journal on Computing}, 40\penalty0 (6):\penalty0
  1662--1714, 2011.

\bibitem[Fran{\c{c}}ois-Lavet et~al.(2019)Fran{\c{c}}ois-Lavet, Bengio, Precup,
  and Pineau]{franccois2019combined}
Vincent Fran{\c{c}}ois-Lavet, Yoshua Bengio, Doina Precup, and Joelle Pineau.
\newblock Combined reinforcement learning via abstract representations.
\newblock In \emph{Proceedings of the AAAI Conference on Artificial
  Intelligence}, volume~33, pages 3582--3589, 2019.

\bibitem[Fu et~al.(2018)Fu, Luo, and Levine]{fu2018adverserial}
Justin Fu, Katie Luo, and Sergey Levine.
\newblock Learning robust rewards with adverserial inverse reinforcement
  learning.
\newblock In \emph{International Conference on Learning Representations}, 2018.
\newblock URL \url{https://openreview.net/forum?id=rkHywl-A-}.

\bibitem[Gelada et~al.(2019)Gelada, Kumar, Buckman, Nachum, and
  Bellemare]{gelada2019deepmdp}
Carles Gelada, Saurabh Kumar, Jacob Buckman, Ofir Nachum, and Marc~G Bellemare.
\newblock Deep{MDP}: Learning continuous latent space models for representation
  learning.
\newblock In \emph{International Conference on Machine Learning}, pages
  2170--2179, 2019.

\bibitem[Givan et~al.(2003)Givan, Dean, and Greig]{givan2003bisimulation}
Robert Givan, Thomas Dean, and Matthew Greig.
\newblock Equivalence notions and model minimization in {M}arkov decision
  processes.
\newblock \emph{Artificial Intelligence}, 147\penalty0 (1):\penalty0 163--223,
  2003.

\bibitem[Jaksch et~al.(2010)Jaksch, Ortner, and Auer]{jaksch2010near}
Thomas Jaksch, Ronald Ortner, and Peter Auer.
\newblock Near-optimal regret bounds for reinforcement learning.
\newblock \emph{Journal of Machine Learning Research}, 11\penalty0
  (Apr):\penalty0 1563--1600, 2010.

\bibitem[Jin et~al.(2018)Jin, Allen-Zhu, Bubeck, and
  Jordan]{jin2018isqefficient}
Chi Jin, Zeyuan Allen-Zhu, Sebastien Bubeck, and Michael~I Jordan.
\newblock Is {Q}-learning provably efficient?
\newblock In \emph{Advances in Neural Information Processing Systems}, pages
  4863--4873, 2018.

\bibitem[Kingma and Ba(2014)]{kingma2014adam}
Diederik~P. Kingma and Jimmy Ba.
\newblock Adam: {A} method for stochastic optimization.
\newblock \emph{CoRR}, abs/1412.6980, 2014.
\newblock URL \url{http://arxiv.org/abs/1412.6980}.

\bibitem[Konidaris et~al.(2011)Konidaris, Osentoski, and
  Thomas]{konidaris2008fourier}
George Konidaris, Sarah Osentoski, and Philip Thomas.
\newblock Value function approximation in reinforcement learning using the
  {F}ourier basis.
\newblock \emph{Proceedings of the Twenty-Fifth AAAI Conference on Artificial
  Intelligence}, pages pages 380--385, August 2011.

\bibitem[Kulkarni et~al.(2016)Kulkarni, Saeedi, Gautam, and
  Gershman]{kulkarni2016deep}
Tejas~D Kulkarni, Ardavan Saeedi, Simanta Gautam, and Samuel~J Gershman.
\newblock Deep successor reinforcement learning.
\newblock \emph{arXiv preprint arXiv:1606.02396}, 2016.

\bibitem[Lehnert et~al.(2017)Lehnert, Tellex, and Littman]{lehnert2017sf}
Lucas Lehnert, Stefanie Tellex, and Michael~L Littman.
\newblock Advantages and limitations of using successor features for transfer
  in reinforcement learning.
\newblock \emph{arXiv preprint arXiv:1708.00102}, 2017.

\bibitem[Lehnert et~al.(2019)Lehnert, Frank, and Littman]{lehnert2019reward}
Lucas Lehnert, Michael~J Frank, and Michael~L Littman.
\newblock Reward predictive representations generalize across tasks in
  reinforcement learning.
\newblock \emph{BioRxiv}, page 653493, 2019.

\bibitem[Li et~al.(2006)Li, Walsh, and Littman]{li2006abstraction}
Lihong Li, Thomas~J Walsh, and Michael~L Littman.
\newblock Towards a unified theory of state abstraction for {MDP}s.
\newblock In \emph{ISAIM}, 2006.

\bibitem[Mohri et~al.(2018)Mohri, Rostamizadeh, and
  Talwalkar]{mohri2018foundationsml}
Mehryar Mohri, Afshin Rostamizadeh, and Ameet Talwalkar.
\newblock \emph{Foundations of Machine Learning}.
\newblock MIT Press, 2018.

\bibitem[Momennejad et~al.(2017)Momennejad, Russek, Cheong, Botvinick, Daw, and
  Gershman]{momennejad2017successor}
Ida Momennejad, Evan~M Russek, Jin~H Cheong, Matthew~M Botvinick, ND~Daw, and
  Samuel~J Gershman.
\newblock The successor representation in human reinforcement learning.
\newblock \emph{Nature Human Behaviour}, 1\penalty0 (9):\penalty0 680, 2017.

\bibitem[Oh et~al.(2017)Oh, Singh, and Lee]{oh2017value}
Junhyuk Oh, Satinder Singh, and Honglak Lee.
\newblock Value prediction network.
\newblock \emph{arXiv preprint arXiv:1707.03497}, 2017.

\bibitem[Osband et~al.(2013)Osband, Russo, and Van~Roy]{osband2013psrl}
Ian Osband, Daniel Russo, and Benjamin Van~Roy.
\newblock ({M}ore) efficient reinforcement learning via posterior sampling.
\newblock In \emph{Advances in Neural Information Processing Systems}, pages
  3003--3011, 2013.

\bibitem[Parr et~al.(2008)Parr, Li, Taylor, Painter-Wakefield, and
  Littman]{parr2008analysis}
Ronald Parr, Lihong Li, Gavin Taylor, Christopher Painter-Wakefield, and
  Michael~L Littman.
\newblock An analysis of linear models, linear value-function approximation,
  and feature selection for reinforcement learning.
\newblock In \emph{Proceedings of the 25th International Conference on Machine
  Learning}, pages 752--759. ACM, 2008.

\bibitem[Riedmiller(2005)]{riedmiller2005neural}
Martin Riedmiller.
\newblock Neural fitted {Q} iteration--first experiences with a data efficient
  neural reinforcement learning method.
\newblock In \emph{European Conference on Machine Learning}, pages 317--328.
  Springer, 2005.

\bibitem[Ruan et~al.(2015)Ruan, Comanici, Panangaden, and
  Precup]{ruan2015representation}
Sherry~Shanshan Ruan, Gheorghe Comanici, Prakash Panangaden, and Doina Precup.
\newblock Representation discovery for mdps using bisimulation metrics.
\newblock In \emph{AAAI}, pages 3578--3584, 2015.

\bibitem[Russek et~al.(2017)Russek, Momennejad, Botvinick, Gershman, and
  Daw]{russek2017predictive}
Evan~M Russek, Ida Momennejad, Matthew~M Botvinick, Samuel~J Gershman, and
  Nathaniel~D Daw.
\newblock Predictive representations can link model-based reinforcement
  learning to model-free mechanisms.
\newblock \emph{PLoS computational biology}, 13\penalty0 (9):\penalty0
  e1005768, 2017.

\bibitem[Schrittwieser et~al.(2019)Schrittwieser, Antonoglou, Hubert, Simonyan,
  Sifre, Schmitt, Guez, Lockhart, Hassabis, Graepel,
  et~al.]{schrittwieser2019muzero}
Julian Schrittwieser, Ioannis Antonoglou, Thomas Hubert, Karen Simonyan,
  Laurent Sifre, Simon Schmitt, Arthur Guez, Edward Lockhart, Demis Hassabis,
  Thore Graepel, et~al.
\newblock Mastering atari, go, chess and shogi by planning with a learned
  model.
\newblock \emph{arXiv preprint arXiv:1911.08265}, 2019.

\bibitem[Silver et~al.(2017)Silver, Hasselt, Hessel, Schaul, Guez, Harley,
  Dulac-Arnold, Reichert, Rabinowitz, Barreto, et~al.]{silver2017predictron}
David Silver, Hado Hasselt, Matteo Hessel, Tom Schaul, Arthur Guez, Tim Harley,
  Gabriel Dulac-Arnold, David Reichert, Neil Rabinowitz, Andre Barreto, et~al.
\newblock The predictron: End-to-end learning and planning.
\newblock In \emph{International Conference on Machine Learning}, pages
  3191--3199. PMLR, 2017.

\bibitem[Song et~al.(2016)Song, Parr, Liao, and Carin]{parrencoder2016}
Zhao Song, Ronald~E Parr, Xuejun Liao, and Lawrence Carin.
\newblock Linear feature encoding for reinforcement learning.
\newblock In \emph{Advances in Neural Information Processing Systems}, pages
  4224--4232, 2016.

\bibitem[Stachenfeld et~al.(2017)Stachenfeld, Botvinick, and
  Gershman]{stachenfeld2017sr}
Kimberly~L Stachenfeld, Matthew~M Botvinick, and Samuel~J Gershman.
\newblock The hippocampus as a predictive map.
\newblock \emph{Nature Neuroscience}, 20\penalty0 (11):\penalty0 1643, 2017.

\bibitem[Sutton(1996)]{sutton1996generalization}
Richard~S Sutton.
\newblock Generalization in reinforcement learning: Successful examples using
  sparse coarse coding.
\newblock \emph{Advances in Neural Information Processing Systems}, pages
  1038--1044, 1996.

\bibitem[Sutton and Barto(1998)]{sutton98}
Richard~S. Sutton and Andrew~G. Barto.
\newblock \emph{Reinforcement Learning: An Introduction}.
\newblock A Bradford Book. MIT Press, Cambridge, MA, 1 edition, 1998.

\bibitem[Sutton and Barto(2018)]{sutton2018rlbook}
Richard~S Sutton and Andrew~G Barto.
\newblock \emph{Reinforcement Learning: An Introduction}.
\newblock MIT Press, 2018.

\bibitem[Sutton et~al.(2008)Sutton, Szepesv{\'a}ri, Geramifard, and
  Bowling]{sutton2008lineardyna}
Richard~S. Sutton, Csaba Szepesv{\'a}ri, Alborz Geramifard, and Michael
  Bowling.
\newblock Dyna-style planning with linear function approximation and
  prioritized sweeping.
\newblock In \emph{Proceedings of the 24th Conference on Uncertainty in
  Artificial Intelligence}, 2008.

\bibitem[Talvitie(2018)]{talvitie2018rewardsformisspecifiedmodel}
Erik Talvitie.
\newblock Learning the reward function for a misspecified model.
\newblock In Jennifer Dy and Andreas Krause, editors, \emph{Proceedings of the
  35th International Conference on Machine Learning}, volume~80 of
  \emph{Proceedings of Machine Learning Research}, pages 4838--4847,
  Stockholmsm{\"a}ssan, Stockholm Sweden, 10--15 Jul 2018. PMLR.
\newblock URL \url{http://proceedings.mlr.press/v80/talvitie18a.html}.

\bibitem[Watkins and Dayan(1992)]{Watkins1992aa}
Christopher~J.C.H. Watkins and Peter Dayan.
\newblock {$Q$-learning}.
\newblock \emph{Machine Learning}, 8\penalty0 (3):\penalty0 279--292, May 1992.

\bibitem[Weber et~al.(2017)Weber, Racani{\`e}re, Reichert, Buesing, Guez,
  Rezende, Badia, Vinyals, Heess, Li, et~al.]{weber2017imagination}
Th{\'e}ophane Weber, S{\'e}bastien Racani{\`e}re, David~P Reichert, Lars
  Buesing, Arthur Guez, Danilo~Jimenez Rezende, Adria~Puigdom{\`e}nech Badia,
  Oriol Vinyals, Nicolas Heess, Yujia Li, et~al.
\newblock Imagination-augmented agents for deep reinforcement learning.
\newblock \emph{arXiv preprint arXiv:1707.06203}, 2017.

\bibitem[Yao and Szepesv{\'a}ri(2012)]{yao2012lam}
Hengshuai Yao and Csaba Szepesv{\'a}ri.
\newblock Approximate policy iteration with linear action models.
\newblock In \emph{AAAI}, 2012.

\bibitem[Zhang et~al.(2017)Zhang, Springenberg, Boedecker, and
  Burgard]{zhang2017deepsucc}
Jingwei Zhang, Jost~Tobias Springenberg, Joschka Boedecker, and Wolfram
  Burgard.
\newblock Deep reinforcement learning with successor features for navigation
  across similar environments.
\newblock In \emph{2017 IEEE/RSJ International Conference on Intelligent Robots
  and Systems (IROS)}, pages 2371--2378. IEEE, 2017.

\end{thebibliography}

\end{document}